\documentclass[10pt,sort&compress,3p]{elsarticle}

\usepackage{amsmath,amssymb,amsthm,mathrsfs,amsopn}
\usepackage{graphicx}
\usepackage{epstopdf}
\usepackage{subcaption}
\usepackage{booktabs}
\usepackage{array}
\usepackage{bm}
\usepackage{caption}
\usepackage{newunicodechar}
\newunicodechar{ﬁ}{fi}
\newunicodechar{ﬀ}{ff}
\usepackage{tcolorbox} %per fare blocchi colorati
\tcbset{colback=blue!5!white,colframe=blue!50!black, left=0mm,right=0mm,top=0mm,bottom=0mm, box align=center, halign=center} %decidere le opzioni per tutte le color boxes
\usepackage{tikz} %per fare gli alberi e non solo

\usetikzlibrary{automata}
\usepackage{enumitem}
\usepackage{fancyhdr} %per intestazione pagina
\usepackage{pdfpages} %to include pdf
\usepackage{float}
\usepackage{multirow}
\usepackage{tabularx}
\usepackage{longtable}
\usepackage{booktabs}
\usepackage{diagbox}
\usepackage{rotating}
\journal{Neural Networks}

\begin{document}
\theoremstyle{definition}
\newtheorem{definition}{Definition}[section]
 \newtheorem{theorem}{Theorem}
\begin{frontmatter}

% --------------------------------------------------------------
%                         Start here
% --------------------------------------------------------------
 
%\renewcommand{\qedsymbol}{\filledbox}
 
\title{Stochastic analysis of heterogeneous porous material with modified neural architecture search (NAS) based physics-informed neural networks using transfer learning}%replace X with the appropriate number
\author[add3]{Hongwei Guo\fnref{email1}}
\fntext[email1]{email: guo{@}hot.uni-hannover.de}
\author[add3,add4,add5]{Xiaoying Zhuang\fnref{email2}}
\fntext[email2]{email: zhuang{@}hot.uni-hannover.de}
\author[add1,add2]{Timon Rabczuk\corref{mycorrespondingauthor}}
\cortext[mycorrespondingauthor]{Corresponding author}
\ead{timon.rabczuk@tdtu.edu.vn}

\address[add1]{Division of Computational Mechanics,\\ Ton Duc Thang University,\\ Ho Chi Minh City, Vietnam}
\address[add2]{Faculty of Civil Engineering,\\ Ton Duc Thang University,\\ Ho Chi Minh City, Vietnam}
\address[add3]{Computational Science and Simulation Technology,\\ Institute of Photonics,\\ Leibniz Universität Hannover,\\ Hannover, Germany}
\address[add4]{Department of Geotechnical Engineering,\\ Tongji University,\\ Shanghai, China.}
\address[add5]{Key Laboratory of Geotechnical and Underground Engineering of Ministry of Education,\\ Tongji University,\\ Shanghai, China.}

\date{\today}

\begin{abstract}
In this work, a modified neural architecture search method (NAS) based physics-informed deep learning model is presented for stochastic analysis in heterogeneous porous material. Monte Carlo method based on a randomized spectral representation is first employed to construct a stochastic model for simulation of flow through porous media. To solve the governing equations for stochastic groundwater flow problem, we build a modified NAS model based on physics-informed neural networks (PINNs) with transfer learning in this paper that will be able to fit different partial differential equations (PDEs) with less calculation. The performance estimation strategies adopted is constructed from an error estimation model using the method of manufactured solutions. A sensitivity analysis is performed to obtain the prior knowledge of the PINNs model and narrow down the range of parameters for search space and use hyper-parameter optimization algorithms to further determine the values of the parameters. Further the NAS based PINNs model also saves the weights and biases of the most favorable architectures, then used in the fine-tuning process. It is found that the log-conductivity field using Gaussian correlation function will perform much better than exponential correlation case, which is more fitted to the PINNs model and the modified neural architecture search based PINNs model shows a great potential in approximating solutions to PDEs. Moreover, a three dimensional stochastic flow model is built to provide a benchmark to the simulation of groundwater flow in highly heterogeneous aquifers. The NAS model based deep collocation method is verified to be effective and accurate through numerical examples in different dimensions using different manufactured solutions.
\end{abstract}


\end{frontmatter}

\section{Introduction}
\label{section 1:Introduction}
In recent years, groundwater pollution has become one of the most important environmental problems worldwide. In order to protect groundwater quality, it is necessary to study groundwater flow and solute transport. With this needs of society and the further development of science and technology, it is quite important to investigate the groundwater flow model, which includes the flow through the aquifer comprising saturated sediment and rock with rather heterogeneous hydrologic properties. In this study, the stochastic analysis in heterogeneous porous material is investigated with deep learning method, which in hope to give more insight into the mechanism of groundwater flow model.

The hydraulic conductivity describes the physical properties with which measures the ability of the medium to transmit fluid through pore spaces, and depends on the structure of the medium (particle size, arrangement, void filling, etc.) and the physical properties of the water (viscosity of the liquid). It is common to use random fields with a given statistical structure to describe the porous media because of its intrinsic complexity \cite{kolyukhin2005stochastic}. Freeze \cite{freeze} has proved, that the hydraulic conductivity field can be well characterized by random log-normal distribution. This approach is often used for the flow analysis in saturated zone, e.g., in \cite{matheron}. Both Gaussian \cite{sabelfeld2d} and exponential \cite{gelharstochastic} correlations are always chosen for the log-normal probability distribution. Different correlation coefficients can lead to different accuracy of the model. Therefore, the two correlation coefficients are compared in this paper.

Different approaches have been used to construct functional equations for permeability as a function of pore structure parameters \cite{carman1997fluid,rumpf,papevariation}. Analytical spectral representation methods were first used by Bakr \cite{bakrstochastic} to solve the stochastic flow and solute transport equations perturbed with a random hydraulic conductivity field. They pointed out, that if a random field is homogeneous and has zero mean, then it will always possible to represent the process by a kind of Fourier (or Fourier-Stieltjes) decomposition into essentially uncorrelated random components, and these random components in turn will give the spectral density function, which is the distribution of variance over different wave numbers $\textit{k}$. This theory is widely used to construct hydraulic conductivity fields, and a number of construction methods have been derived, such as turning bands method \cite{mantoglou}, HYDRO\_GEN method \cite{berlinhydro} and Kraichnan algorithm \cite{kraichnandiffusion}. Ababou and his coworker \cite{ababounumerical} used turning bands method to verify the feasibility of solving this problem using numerical methods, and also indicated the scale range of relevant parameters. Wörmann and Kronnnäs \cite{ababounumerical} tested a gradually increase of the heterogeneity of the flow resistance and compared the numerically simulated residence time PDF with the observed based on HYDRO\_GEN method. Unlike the other two methods, Kraichnan proposed an approximation algorithm for direct control of random field accuracy,  by increasing the modulus, and of its variability, through the variance of the log-hydraulic conductivity random field. Inspired by these results, Kolyukhin and Sabelfeld \cite{kolyukhin2005stochastic} constructed a randomized spectral model (RSM) for simulation of a steady flow in porous media in 3D case under small fluctuation assumptions. This approach is used in this paper to generate hydraulic conductivity fields.

With the development of information technology, machine learning has become the hottest topic nowadays. As a new branch of machine learning, deep learning was first brought up in the realm of artificial intelligence in 2006, which uses deep neural networks to learn features of data with high-level of abstractions \cite{lecun2015deep}. The deep neural networks adopt artificial neural network architectures with various hidden layers, which exponentially reduce the computational cost and amount of training data in some applications \cite{al2018solving}. The major two desirable traits of deep learning lie in the nonlinear processing in multiple hidden layers in supervised or unsupervised learning \cite{vargas2017deep}. Several types of deep neural networks such as convolutional neural networks (CNN) and recurrent/recursive neural networks (RNN) \cite{patterson2017deep} have been created, which further boost the application of deep learning in image processing \cite{yang2018visually}, object detection \cite{zhao2019object}, speech recognition \cite{nassif2019speech} and many other domains including genomics \cite{yue2018deep} and even finance \cite{fischer2018deep}. In some cases, certain physical laws of the model to be predicted is a priori, and these laws can be used to formulate learning algorithms as well as set up neural network configurations. This leads to a physics-informed neural network (PINN), which is widely used in practical research, for example, earthquake velocity model prediction \cite{xuphysic}.

Besides the application of deep learning in regression and classification, some researchers employed deep learning for the solution of PDEs, some researchers employed deep learning for the solution of PDEs. Mills et al. deployed a deep conventional neural network to solve Schrödinger equation, which directly learned the mapping between potential and energy \cite{Mills:2017aa}. E et al. applied deep learning-based numerical methods for high-dimensional parabolic PDEs and back-forward stochastic differential equations, which was proven to be efficient and accurate for 100-dimensional nonlinear PDEs \cite{E_2017,Han:2018aa}. Also, E and Yu proposed a Deep Ritz method for solving variational problems arising from partial differential equations \cite{E_2018}. Raissi et al. applied the probabilistic machine learning in solving linear and nonlinear differential equations using Gaussian Processes and later introduced a data-driven Numerical Gaussian Processes to solve time-dependent and nonlinear PDEs, which circumvented the need for spatial discretization \cite{raissi2019physics,Raissi:2018aa}. Later, they \cite{RAISSI2018125,Raissi:2018:DHP:3291125.3291150} introduced a physical informed neural networks for supervised learning of nonlinear partial differential equations from Burger’s equations to Navier-Stokes equations. Two distinct models were tailored for spatio-temporal datasets: continuous time and discrete time models. They also applied a deep neural networks in solving coupled forward-backward stochastic differential equations and their corre- sponding high-dimensional PDEs [35]. Beck et al. \cite{Beck:2018aa,Beck_2019} studied the deep learning in solving stochastic differential equations and Kolmogorov equations. Anitescu et al. \cite{anitescu2019artificial}, Guo et al. \cite{guo2019deep} applied deep neural networks based collocation method for finding the solutions for second order and fourth order boundary value problems. Rabczuk et al. \cite{samaniego2020energy, nguyen2020deep} proposed an energy approach to the solution of partial differential equations in computational mechanics via machine learning. Goswami \cite{goswami2020transfer} applied a transfer learning enhanced physics informed neural network (PINN) algorithm for solving brittle fracture problems based on phase-field modeling.

However, in a typical machine learning application, the practitioner must apply appropriate data pre-processing, feature engineering, feature extraction and feature selection methods to make the dataset suitable for machine learning. Following these pre-processing steps, practitioners must perform algorithm selection and hyper-parameter optimization to maximize their final machine learning model's predictive performance. The PINNs model is no exception though the randomly distributed collocation points are generated for further calculation without the need for data-preprocessing. Most of the time wasted in PINNs based model are exerted upon the tuning of neural architecture configurations, which strongly influences the accuracy and stability of the numerical performance. Since many of these steps are typically beyond the capabilities of non-experts, automated machine learning (AutoML) has become the trend as an AI-based solutions to meet the growing challenges in machine learning applications, which seeks to automate the many time-consuming and laborious steps of the machine learning pipeline to make it easier to apply machine learning to real-world problems. 

The oldest AutoML library is AutoWEKA \cite{thorntonauto}, first released in 2013, which automatically selects models and hyper-parameters. Other notable AutoML libraries include auto-sklearn \cite{feurerefficient}, H2O AutoML \cite{H2OAutoML}, and TPOT \cite{le2020scaling}. Neural architecture  search (NAS) \cite{zophneural} is a technique for automatic design of neural networks, which allows algorithms to automatically design high performance network structures based on sample sets. Neural architecture  search (NAS) aims to find a configuration comparable to human experts on certain tasks and even discover certain network structures that have not been proposed by humans before, which can effectively reduce the use and implementation cost of neural networks. In \cite{pham2018efficient} the authors add a controller in the efficient NAS, which can learn to discover neural network architectures by searching for an optimal subgraph within a large computational graph. They used parameters sharing between the subgraphs to make the computing process faster. The controller decides which parameter matrices are used, by choosing the previous indices. Therefore, in ENAS, all recurrent cells in a search space share the same set of parameters. Liu \cite{liu2018progressive} used a sequence model-based optimization (SMBO) strategy to learn a surrogate model with which to guide the search of the structure space while searching for neural network structures.

To build up an Neural architecture  search (NAS) model, it is necessary to conduct dimensionality reduction and identification of valid parameter
bounds to reduce the calculation involved in auto-tuning. A global sensitivity
analysis can be used to identify valid regions in the search space and subsequently decrease its dimensionality \cite{fraikin2019dimensionality}, which can serve as a starting point for an efficient calibration process.

The remainder of this paper is organized as follows. In Section~\ref{section 2:hydraulic conductivity}, we describe the physical model of the groundwater flow problem, the randomized spectral method to generate hydraulic conductivity fields, and also construct the manufactured solutions to verify the accuracy of our model, which is helpful in building up the performance estimation strategy. In Section~\ref{section 3:neural networks}, we introduce the neural architecture search model, including its structure and how each part works specifically. And we presented an efficient sensitivity analysis method and compared several hyper-parameter optimizers in order to find an accurate and efficient search method. In Section~\ref{section 4:numerical examples}, 
we introduce the finite difference method and the finite element method as a comparison to verify the feasibility of our method through numerical experiments. At last, some conclusions of the present study are drawn in Section~\ref{section 5:conclusion}. The specific generation process of random fields and the expressions for manufactured solutions  are presented in ~\ref{appendix a} and ~\ref{appendix b}.

\section{Generate stochastic flow fields in a heterogeneous porous medium with randomized spectral model}
\label{section 2:hydraulic conductivity} 
\subsection{Darcy equation for groundwater flow problem}
\label{subsection 2.1:pde3D} 
Consider the continuity equation for steady-state, aquifer flow in a porous media governed by the Darcy law:
\begin{equation}\label{eq:gvequation}
\bm{q}(\bm{x})=-K(\bm{x})\nabla(h(\bm{x})),
\end{equation}
where $\bm{q}$ is the Darcy velocity, $K$, is the hydraulic conductivity, $h$,  the hydraulic head $h=H+\delta h$,with the mean $H$ and the perturbation $\delta h$.

To describe the variation of hydraulic conductivity as a function of the position vector $\bm{x}$, it is convenient to introduce the variable
\begin{equation}\label{eq:log-normal}
Y(\bm{x})=\ln{K(\bm{x})},
\end{equation}
where $Y(\bm{x})$ is the  hydraulic log-conductivity with  the mean $\langle Y \rangle$ and perturbation $Y'(\bm{x})$:
\begin{equation}\label{eq:y}
Y(\bm{x})=\langle Y \rangle+Y'(\bm{x}),
\end{equation}
where $E[Y'(\bm{x})]=0$, and $Y(\bm{x})$ is taken to be a three-dimensional statistically homogeneous random field characterized by its correlation function,
\begin{equation}\label{eq:correlation function of Y}
C_Y(\bm{r}) = \langle Y'(\bm{x} + \bm{r})Y'(\bm{x}) \rangle,
\end{equation}
where $\bm{r}$ is the separation vector.

According to the conservation equation $\nabla \cdot \bm{q} = 0$, the Equation~\eqref{eq:gvequation} can be rewritten in the following form:
\begin{equation}\label{eq:darcy equation}
E(h)=\sum_{j=1}^{N}\frac{\partial }{\partial x_j}(K(\bm{x})\frac{\partial h}{\partial x_j})=0.
\end{equation}
with $N$ the dimension and subject to the Neumann and Dirichlet boundary conditions as follows:
\begin{equation}\label{eq:boundary conditions}
\begin{aligned}
h(\bm{x})=\bar{h}, \bm{x} \in \tau_D,\\
q_n(\bm{x})=\bar{q}_n, \bm{x} \in \tau_N.
\end{aligned}
\end{equation}
The groundwater flow problem can be boiled down to find a solution $h$ such that Equations \ref{eq:darcy equation} and \ref{eq:boundary conditions} holds. $E$ is an operator that maps elements of vector space H to vector space V and represents the investigated physical model:
\begin{equation}\label{eq:mapping}
E:H\rightarrow V, with\,h\in H.
\end{equation}

To be more specific, to solving Equation~\eqref{eq:darcy equation} with $N=3$ in domain $D=[0,L_x]\times[0,L_y]\times[0,L_z]$, the Dirichlet boundary and Neumann boundary conditions can be assumed as follows:
\begin{equation}\label{eq:boundary conditions}
\left\{  
\begin{array}{lr}  
h(0, y, z) = -J\cdot L_x, \quad h(L_x, y, z) = 0, &\forall y \in [0, L_y], z \in [0, L_z],\\
\frac{\partial h}{\partial y}(x,0,z)= \frac{\partial h}{\partial y}(x,L_y,z)=0, &\forall x \in [0, L_x],z \in [0, L_z],\\
\frac{\partial h}{\partial z}(x,y,0)= \frac{\partial h}{\partial z}(x,y,L_z)=0, & \forall x \in [0, L_x],y \in [0, L_y],
\end{array}  
\right.  
\end{equation}
where $J$ is  the mean slope of the hydraulic head in $x$ direction, a dimensionless constant \cite{bakrstochastic}.

As suggested by Ababou \cite{ababounumerical}, the scale of fluctuation should be much smaller than the scale of domain to get statistically meaningful results, considering the simplicity of the calculations, we set the $L_x, L_y, L_z$ ten times larger than $\lambda$. They also suggested a reasonable resolution requirement could be 
\begin{equation}\label{eq:dx}
\frac{\Delta x}{\lambda} \leq \frac{1}{5},
\end{equation}
we set the number of discretization points in each direction 100.

As $Y'$ is  homogeneous and isotropic, we consider two correlation functions: the exponential correlation function \cite{salandinsolute},
\begin{equation}\label{eq:exp correlation}
C_Y(\bm{r})=\sigma_Y^2exp(-\frac{\left|\bm{r}\right|}{\lambda}),
\end{equation}
and the Gaussian correlation function \cite{dentztemporal},
\begin{equation}\label{eq:gauss correlation}
C_Y(\bm{r})=\sigma_Y^2exp(-\frac{\left|\bm{r}\right|^2}{\lambda^2}),
\end{equation}
where $\lambda$ is the log conductivity correlation length scale.

\subsection{Generate the hydraulic conductivity fields}
\label{subsection 2.2:generate the hydraulic conductivity fields}
In the application of spectral methods, the Wiener–Khinchin theorem plays a very important role. It states the power spectral density of a smooth stochastic process is the Fourier transform of his auto-correlation function. Applying it to the heterogeneous groundwater flow problem, the Gaussian random field with given spectral density $S(\bm{k})$  is just the Fourier transform of the correlation function (in Equation \eqref{eq:correlation function of Y}):
\begin{align}
&C_Y(\bm{r})=\int_{\mathbb{R}^N} e^{i2\pi \bm{k}\cdot \bm{r}}S(\bm{k}),\label{eq:wiener}\\ 
&S(\bm{k})=\int_{\mathbb{R}^N} e^{-i2\pi \bm{k}\cdot \bm{r}}C_{Y}(\bm{r}),\label{eq:wiener2}
\end{align}
with $S(\bm{k})$ spectral function of the random field $Y'(\bm{x})$.

According to the Fourier transform table, 
\begin{align}
&\mathscr{F}(exp(-\frac{\left| \bm{r}\right|}{\lambda}))=\frac{2\lambda}{1+4\pi^2\bm{k}^2\bm{r}^2},\label{eq:FFT1}\\
&\mathscr{F}(exp(-\frac{\left| \bm{r}\right|^2}{\lambda^2}))=\lambda \sqrt{\pi}e^{-\pi^2\bm{k}^2\bm{r}^2},\label{eq:FFT2}
\end{align}

Substituting Equation~\eqref{eq:exp correlation},~\eqref{eq:gauss correlation},~\eqref{eq:FFT1} and~\eqref{eq:FFT2} into Equation~\eqref{eq:wiener2} respectively, the spectral function under the exponential and the Gaussian correlation coefficient can be derived:
\begin{align}
&S(\bm{k},\lambda)=\sigma_Y^2 \lambda^d (1+(2\pi \bm{k}\lambda)^2)^{-\frac{d+1}{2}},\label{eq:exp spectral}\\
&S(\bm{k},\lambda)=\sigma_Y^2 \pi^{d/2}\lambda^d e^{-(\pi\bm{k}\lambda)^2}.\label{eq:gauss spectral}
\end{align}

Kramer \cite{kramercomparative} has given an expression for a Gaussian homogenous random field in the general case:
\begin{equation}\label{eq:fourier}
Y'(\bm{x})=\int e^{-i2\pi \bm{k} \bm{x}}\sqrt{S(\bm{k})}W(\mathrm{d}\bm{k}),
\end{equation}
here $W(\mathrm{d}k)$ is a complex-valued white noise random measure, with $W(b)=\overline{W(-b)}$ and $\langle W(b)\rangle=0$.

According to Euler's formula, Equation~\eqref{eq:fourier} can be rewritten as:
\begin{equation}\label{eq:euler}
Y'(\bm{x})=\int cos(2\pi \bm{k}\bm{x})\sqrt{S(\bm{k})}W(\mathrm{d}\bm{k})+i\int sin(2\pi k \bm{x})\sqrt{S(\bm{k})}W(\mathrm{d}\bm{k}).
\end{equation}

It is obvious from Equation~\eqref{eq:exp spectral} and Equation~\eqref{eq:gauss spectral} that the two spectral functions are even. So,
\begin{equation}\label{eq:euler1}
Y'(\bm{x})=2\big (\int_{0}^{+\infty} cos(2\pi \bm{k} \bm{x})\sqrt{S(\bm{k})}W_1(\mathrm{d}\bm{k})+\int_{0}^{+\infty} sin(2\pi \bm{k} \bm{x})\sqrt{S(\bm{k})}W_2(\mathrm{d}\bm{k})\big ),
\end{equation}
with $W_1(\mathrm{d}\bm{k})$ and $W_2(\mathrm{d}\bm{k})$ being two independent real valued Wiener processes.

Choose the probability density function($\textit{pdf}$) of random variables $k$, $p(\bm{k}) = 2S(\bm{k})/ \sigma^2$, substitute it into Equation~\eqref{eq:euler1},
\begin{equation}\label{eq:euler2}
Y'(\bm{x})=\sqrt{2\sigma^2}\big (\int_{0}^{+\infty} cos(2\pi \bm{k} \bm{x})\sqrt{p(\bm{k})}W_1(\mathrm{d}\bm{k})+\int_{0}^{+\infty} sin(2\pi \bm{k} \bm{x})\sqrt{p(\bm{k})}W_2(\mathrm{d}\bm{k})\big ),
\end{equation}
which variance can be calculated as:
\begin{equation}\label{eq:eulervar}
Var\big (Y'(\bm{x})\big )=2\sigma^2\big (\int_{0}^{+\infty} cos^2(2\pi \bm{k} \bm{x})p(\bm{k})\mathrm{d}\bm{k}+\int_{0}^{+\infty} sin^2(2\pi \bm{k} \bm{x})p(\bm{k})\mathrm{d}\bm{k}\big ).
\end{equation}

The following approximations can be obtained using the Monte-Carlo integration,
\begin{equation}\label{eq:monte}
Var(Y'(\bm{x})) \approx \frac{2\sigma^2}{N}(\sum_{i=1}^{N} \big (cos^2(2\pi \bm{k}_i \bm{x})+ sin^2(2\pi \bm{k}_i \bm{x})\big ).
\end{equation}

From this we can derive the final approximate form,
\begin{equation}\label{eq:u}
Y'(\bm{x}) = \sqrt{\frac{2\sigma^2}{N}}\sum_{i=1}^{N} \big (\xi_1 cos(2\pi \bm{k}_i \bm{x})+\xi_2 sin(2\pi \bm{k}_i \bm{x})\big ),
\end{equation}
where the $\xi_i$ are again mutually independent Gaussian random variables.

For the random variable $\bm{k}_i$, we can get its probability density distribution function to be $p(\bm{k})$ and calculate its cumulative distribution function($\textit{cdf}$) according to $F(k)=\int_{-\infty}^{\bm{k}}p(x)\mathrm{d}x$, and as long as there is an uniformly distributed other random variable $\theta$, the inverse function $\bm{k} = F^{-1}(\theta)$ can be obtained, and $k$ must obey the $p(\bm{k})$ distribution. The choices of $\bm{k}$ and the proof process are detailed in~\ref{appendix a}.

Using these expressions for $\bm{k}$, we can write MATLAB programs to generate the random fields we need. Shown in Figures \ref{fig:k2d} and \ref{fig:k3d}, the schematic of the random field space is generated with fixed $\langle k\rangle$ to 15, $\sigma^2$ to 0.1, and $N$ to 1000,

\begin{figure}[H]
	\captionsetup{width=0.85\columnwidth}
	\centering
	\begin{subfigure}[b]{6.0cm}
		\centering\includegraphics[height=6cm,width=6.0cm]{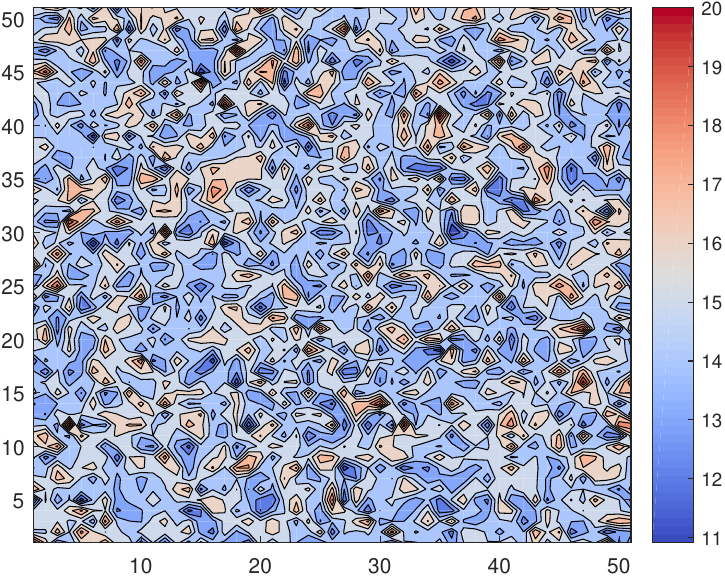}
		\caption{}
	\end{subfigure}%
	\hspace{0.5cm}
	\begin{subfigure}[b]{6.0cm}
		\centering\includegraphics[height=6cm,width=6.0cm]{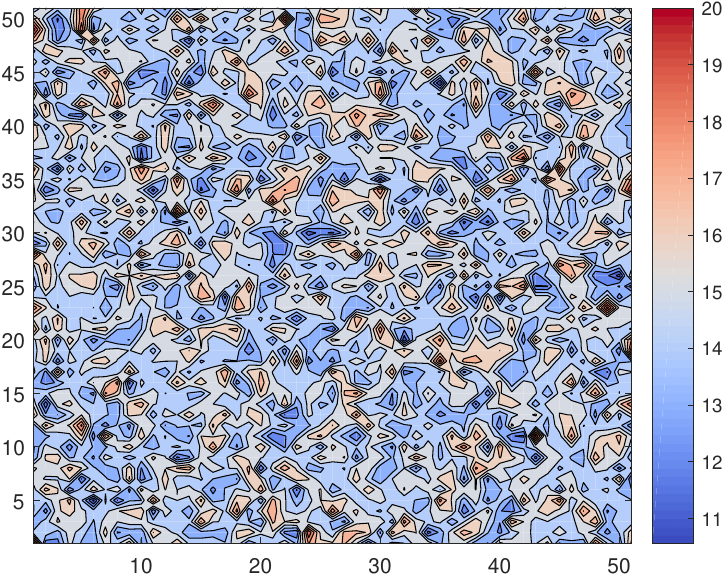}
		\caption{}
	\end{subfigure}% 
	\caption{Two dimensional hydraulic conductivity field with $\left(a\right)$ exponential correlation and $\left(b\right)$ Gaussian correlation}
	\label{fig:k2d}
\end{figure}

\begin{figure}[H]
	\captionsetup{width=1.0\columnwidth}
	\centering
	\begin{subfigure}[b]{6.0cm}
		\centering\includegraphics[height=5cm,width=6.0cm]{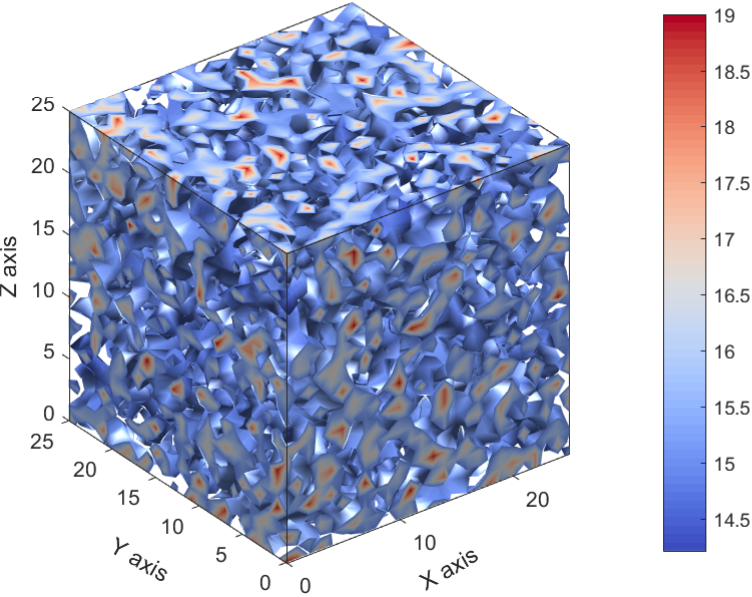}
		\caption{}
	\end{subfigure}%
	\hspace{0.5cm}
	\begin{subfigure}[b]{6.0cm}
		\centering\includegraphics[height=5cm,width=6.0cm]{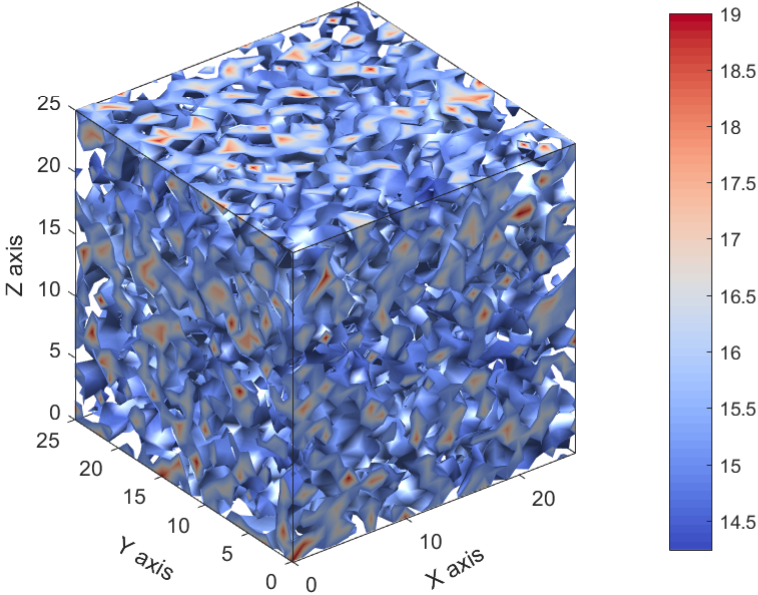}
		\caption{}
	\end{subfigure}% 
	\caption{Three dimensional hydraulic conductivity field with $\left(a\right)$ exponential correlation and $\left(b\right)$ Gaussian correlation}
	\label{fig:k3d}
\end{figure}

\subsection{Defining the experimental model}
\label{subsection 2.3:Parameter setting}
After determining the expression for the groundwater flow problem, we need to set the geometric and physical parameters involved in simulation. For the hydraulic conductivity, the number of modes $N$ and the variance $\sigma^2$ got a big impact. With the Lipschitz continuity in mind, experiments have found that an excessively large N leads to, when the exponential correlation coefficient is used, the realization of the K field apparently approaches a non-differentiable function \cite{alecsa2019benchmark}. So in this paper, we set the $N$ values to 500, 1000 and 2000, to verify the fit of the model to $k$ for different modes. $\sigma^2$ determines the heterogeneity of the hydraulic conductivity, with a larger $\sigma^2$ indicating larger heterogeneity. 
In reality geological formations, $\sigma^2$ has a wide range of variation. As summarized in Sudicky' study \cite{sudicky2010heterogeneity}, in low heterogeneous Canadian Forces Base Borden aquifers is $\sigma^2 =0.29$, for Cape Cod is 0.14, but in  highly heterogeneous Columbus aquifers, $\sigma^2 =4.5$. First-order analysis \cite{bakrstochastic} has been provided a solid basis for predictions. Numerical simulations \cite{berlinhydro}
indicate that the first-order results are robust and applicable
for $\sigma^2$ close to and even above 1. With this approximation, we can get the $e^{\langle Y\rangle}=\langle K\rangle exp(-\sigma^2/2)$ for one and two dimensional cases \cite{attinger2003generalized}, and $e^{\langle Y\rangle}=\langle K\rangle exp(-\sigma^2/6)$ for three dimensional cases \cite{gelhar1983three}. So in this paper we set the value of $\sigma^2$ to 0.1, 1 and 3, covering the three cases from small to moderate and large. The mean hydraulic conductivity is fixed to $\langle K\rangle = 15 m/day$, a value representative for gravel or coarse sand aquifers \cite{dagan2012flow}. And we set all the correlation lengths in one and two dimensional cases equal $1m$, in three dimensional cases, we set them different, $\lambda_1=0.5m$, $\lambda_2=0.2m$ and $\lambda_3=0.1m$. Based on the above settings, we have finalized our test domain: 
\begin{itemize}
	\item One dimensional groundwater flow $\rightarrow [0,25]$
	\item Two dimensional groundwater flow $\rightarrow[0,20]\times[0,20]$
	\item Three dimensional groundwater flow $\rightarrow [0,5]\times[0,2]\times[0,1]$
\end{itemize}

\subsection{Manufactured solutions}
\label{subsection 2.4:mms}
To verify the accuracy of our model and build up an error estimation model, we use the  method of manufactured Solution (MMS), which provides a general procedure for generating analytical solutions \cite{tremblay2006code}. Malaya et al. \cite{malaya2013masa} discussed the method manufactured solutions in constructing an error estimator for solution verification where one simulates the phenomenon of interest and has no priori knowledge of the solution. MMS, instead of relying upon the availability of an exact solution to the governing equations, specifies a manufactured solution. This artificial solution is then substituted into the partial differential equations. Naturally, there will be a residual term since the chosen function is unlikely to be an exact solution to the original partial differential equations. This residual can then be added to the code as a source term; subsequently, the MMS test uses the code to solve the modiﬁed equations and checks that the chosen function is recovered.

With MMS, the original problem is to find the solution 
Equation~\eqref{eq:darcy equation} is thus changed to the following form,

\begin{equation}\label{eq:darcy with f}
E(\hat{h})=\sum_{j=1}^{N}(\frac{\partial }{\partial x_j}(K(\bm{x})\frac{\partial \hat{h}}{\partial x_j}))=\sum_{j=1}^{N}f_j=f.
\end{equation}
For operator $E(\hat{h})$, we now get a source term $f$. By adding the source term to the original governing equation $E$, a slightly modified governing equation will be obtained:
\begin{equation}\label{eq:E darcy with f}
E'(\hat{h})=E(\hat{h})-f=0,
\end{equation}
which is solved by the manufactured solution $\hat{h}$.

The Neumann and Dirichlet boundary conditions are thus modified as follows:
\begin{equation}\label{eq:modified boundary conditions}
\begin{aligned}
\hat{h}(\bm{x})=\hat{h}_{MMS}(\bm{x}), \bm{x} \in \tau_D,\\
\hat{q}_n(\bm{x})=-K(\bm{x})\hat{h}_{MMS,n}(\bm{x}), \bm{x} \in \tau_N.
\end{aligned}
\end{equation}

We adopt the form of the manufactured solution mentioned in Tremblay's study \cite{tremblay2006code},

\begin{equation}\label{eq:mms1}
\hat{h}_{MMS}(\bm{x})=a_0+sin( \sum_{j=1}^{N}a_j x_j),
\end{equation}
where $a_i$ are arbitrary non-zero real numbers.

When the manufactured solutions~\eqref{eq:mms1} are applied on the lift side of the Equation~\eqref{eq:darcy equation}, we will get a source term $f$,

\begin{equation}\label{eq:f1}
f(x_j)=a_j \frac{\partial K(\bm{x})}{\partial x_j} cos( \sum_{i=1}^{N}a_i x_i)-a_j^2 K(\bm{x}) sin(\sum_{i=1}^{N}a_i x_i).
\end{equation}

To verify the adaptability of our model to different solutions, we also used another form of manufactured solution \cite{malaya2013masa},

\begin{equation}\label{eq:mms2}
\hat{h}_{MMS}(\bm{x})=a_0+\sum_{j=1}^{N}sin( a_j x_j),
\end{equation}
the parameter values are the same as Equation~\eqref{eq:mms1}. We can get the source term as follows,

\begin{equation}\label{eq:f2}
f(x_j)=a_j \frac{\partial K(\bm{x})}{\partial x_j} cos(a_j x_j)-a_j^2 K(\bm{x}) sin(a_j x_j).
\end{equation}

This leads to the change of the boundary conditions from Equation~\eqref{eq:boundary conditions} to 
\begin{equation}\label{eq:boundary conditions f}
\left\{  
\begin{array}{lr}  
\hat{h}(0, y, z) =\hat{h}_{MMS}(0,y,z), &\forall y \in [0, L_y], z \in [0, L_z]\\
\hat{h}(L_x, y, z) = \hat{h}_{MMS}(L_x,y,z), &\forall y \in [0, L_y], z \in [0, L_z]\\
\frac{\partial \hat{h}}{\partial y}(x,0,z)= \frac{\partial \hat{h}_{MMS}}{\partial y}(x,0,z), &\forall x \in [0, L_x],z \in [0, L_z] \\
\frac{\partial \hat{h}}{\partial y}(x,L_y,z)= \frac{\partial \hat{h}_{MMS}}{\partial y}(x,L_y,z), &\forall x \in [0, L_x],z \in [0, L_z] \\
\frac{\partial \hat{h}}{\partial z}(x,y,0)= \frac{\partial \hat{h}_{MMS}}{\partial z}(x,y,0), & \forall x \in [0, L_x],y \in [0, L_y]\\
\frac{\partial \hat{h}}{\partial z}(x,y,L_z)= \frac{\partial \hat{h}_{MMS}}{\partial z}(x,y,L_z), & \forall x \in [0, L_x],y \in [0, L_y]
\end{array}  
\right.  
\end{equation}

These source terms can be used as a physical law  to describe the system, and also as a basis for evaluating neural networks. The specific form of the construct solutions and source term $f$ used in this paper are given in~\ref{appendix b}.

\section{Analysis of groundwater flow problems with deep learning based neural architecture search method}
\label{section 3:neural networks}
\subsection{Modified Neural architecture search (NAS) model}
\label{subsection 3.1:NAS}

The convolutional NAS approach has three main dimensions \cite{elsken2018neural}, the first one is a collection of candidate neural network structures called the search space, a narrowing of which can be achieved by combining it with the conditions of a known practical model. Second one is the search strategy, which represents in detail how NN works in the search space. And the last, the performance of the neural network is measured by certain metrics such as precision and speed, which is called performance evaluation. Inspired by Park \cite{parkecolutionay}, we construct the system configuration of the NAS fitted to the PINNs model in Figure~\ref{fig:NAS}. It consists of sensitivity analyses (SA), search methods, physics-informed neural networks (NN) generator, which eventually output the optimum neural architecture configuration and the corresponding weights and biases. A transfer learning model is eventually built up based on the weights and biases and the selected neural network configurations.

\begin{figure}[H]
		\captionsetup{width=0.9\columnwidth}
		\includegraphics[height=9cm]{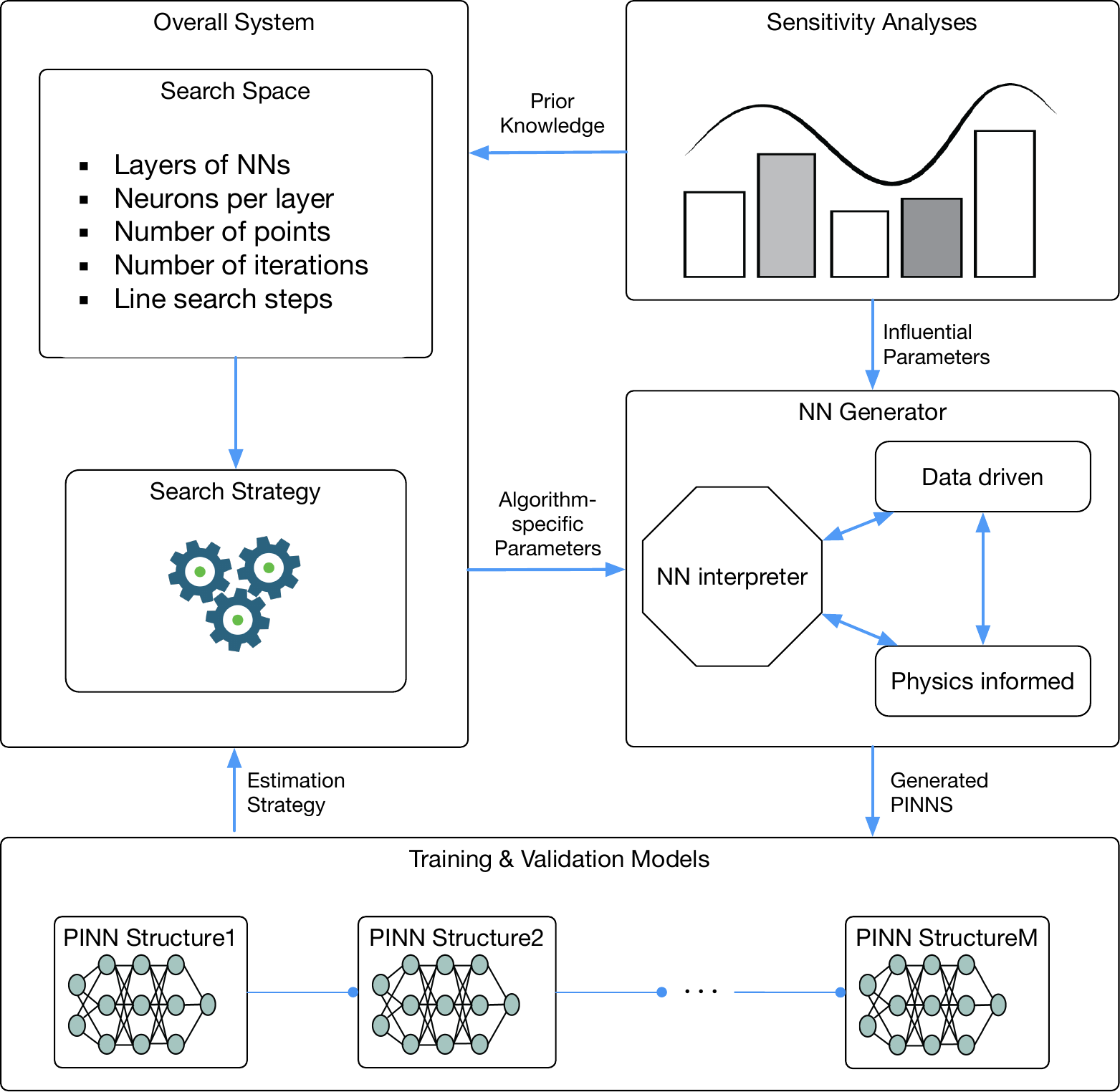}
	\centering
	\caption{Overall methodology}
	\label{fig:NAS}
\end{figure}

\subsubsection{Components of convolutional NAS}
\label{subsubsection 3.1.1:components}
An abstract illustration of convolutional NAS methods is shown below:

\begin{figure}[H]
	\captionsetup{width=0.85\columnwidth}
	\centering
	\begin{subfigure}[b]{6.0cm}
		\centering\includegraphics[height=9cm]{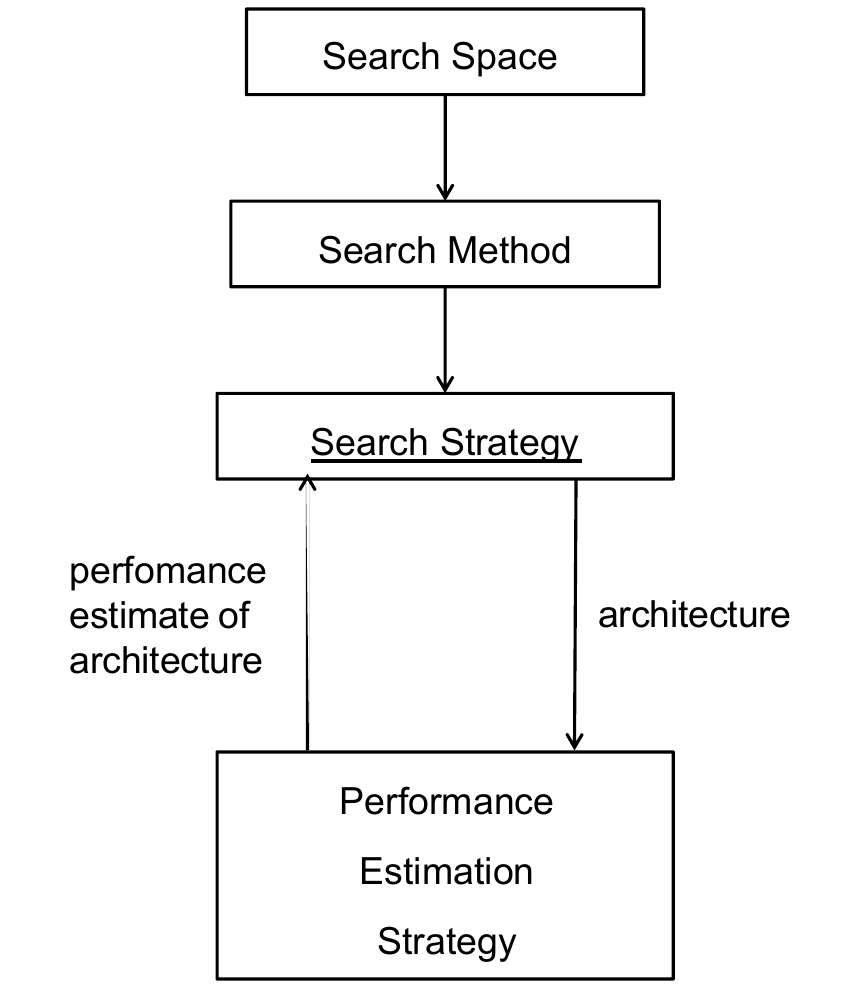}   
		\caption{}
	\end{subfigure}%
	\hspace{0.5cm}
	\begin{subfigure}[b]{6.0cm}
		\centering\includegraphics[height=9cm]{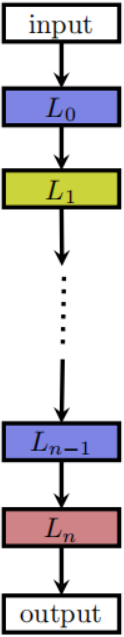}   
		\caption{}
	\end{subfigure}% 
	\caption{$\left(a\right)$ Abstract illustration of NAS methods and $\left(b\right)$ Search space}
	\label{fig:abstract NAS}
\end{figure}

\begin{itemize}
\item \textbf{Search Space}. The search space defines the architecture that in principle can be represented. Combined with a priori knowledge of the typical properties of architectures well suited to the task, this can reduce the size of the search space and simplify the search. For the model in this study, the priori knowledge of search space is gained from the global sensitive analysis. Figure~\ref{fig:abstract NAS}$\left(b\right)$ shows a common global search space with a chain structure in NAS work.

The chain-structured neural network architecture can be written as a sequence of $n$ layers, where i'th layer $L_i$ receives input from layer $i-1$ and its output is used as input for layer $i+1$:
\begin{equation}
	output = L_n\odot L_{n-1}\odot ... L_1\odot L_0
\end{equation}
where $\odot$ are operations.
\item \textbf{Search Method}. The search method is an initial filtering step to help us narrow down the search space. In this paper, hyperparameter optimizers will be used to accomplish this goal. It should be noted here the choice of search space largely determines the difficulty of the optimization problem, which may result in the optimization problem remaining (i) noncontinuous and (ii) relatively high-dimensional. Thus, the prior knowledge of the model features needs to be incorporated into consideration. 
\item \textbf{Search Strategy}. The search strategy details how to explore the search space. It's important to note that, on the one hand, it is desirable to quickly find architectures that perform well, but on the other hand, it should avoid converging too early to areas of suboptimal architecture.
 \item \textbf{Performance Estimation Strategy}. Performance estimation is the process of estimating this performance: the simplest option is standard training and validation of the data for the architecture. Based on discussion in Section \ref{subsection 2.4:mms}, we can define the relative error for performance estimation strategy:
  \begin{equation}\label{eq:Estimation Strategy}
\delta h=\frac{\| \hat{h}-\hat{h}_{MMS}\|_2}{\|\hat{h}_{MMS}\|_2}
\end{equation} 
\end{itemize}

\subsubsection{Modified NAS}
\label{subsubsection 3.1.2:modified NAS}
To integrated with the physics-informed machine learning model, some changes have been drawn. For the modified model shown in Figure~\ref{fig:NAS}, the NAS is divided into four main phases. Namely, a sensitivity analysis to dive into prior knowledge behind the physics-informed machine learning model, which eventually helped to construct the search space in hope to be less dependent on human experts. The second phase is the search strategies, there are a wide choice of optimization methods. In this paper, we have tested several commonly used optimization strategies, including randomization search method, Bayesian optimization method, Hyperband optimization method, and Jaya optimization method. The third phase is the neural network generators, including the generation of physics-informed deep neural networks tailored for a mechanical model based on the information from optimization. The final phase are the training and validation models, with the input neural architectures, it will output the estimation strategies. A suitable estimation is recommended in Equation \ref{eq:Estimation Strategy}.

\subsection{Neural networks generator}
\label{subsection 3.2:NN}

How to approximating a function, regarding solving a partial differential equation, has long been a problem in mathematics. Mathematicians have developed many tools to approximate functions, such as interpolation theory, frameworks, spectral methods, finite elements, etc. From the perspective of approximation theory, neural Networks can be viewed as a nonlinear smooth function approximator. Using the neural network, we can obtain an output value that reflects the quality, validity, etc. of the input data, adjust the configuration of the neural network based on this result, recalculate the result, and repeat these steps until the target is reached. Physics-informed neural networks, on the other hand, add  physical conservation law and  prior physical knowledge to the existing neural network, which require substantially less training data and can result in simpler neural network structures, while achieving high accuracy \cite{misyris2019physics}. The diagram of its structure is as follows:

\begin{figure}[H]
	\captionsetup{width=0.9\columnwidth}
	\includegraphics[height=9cm]{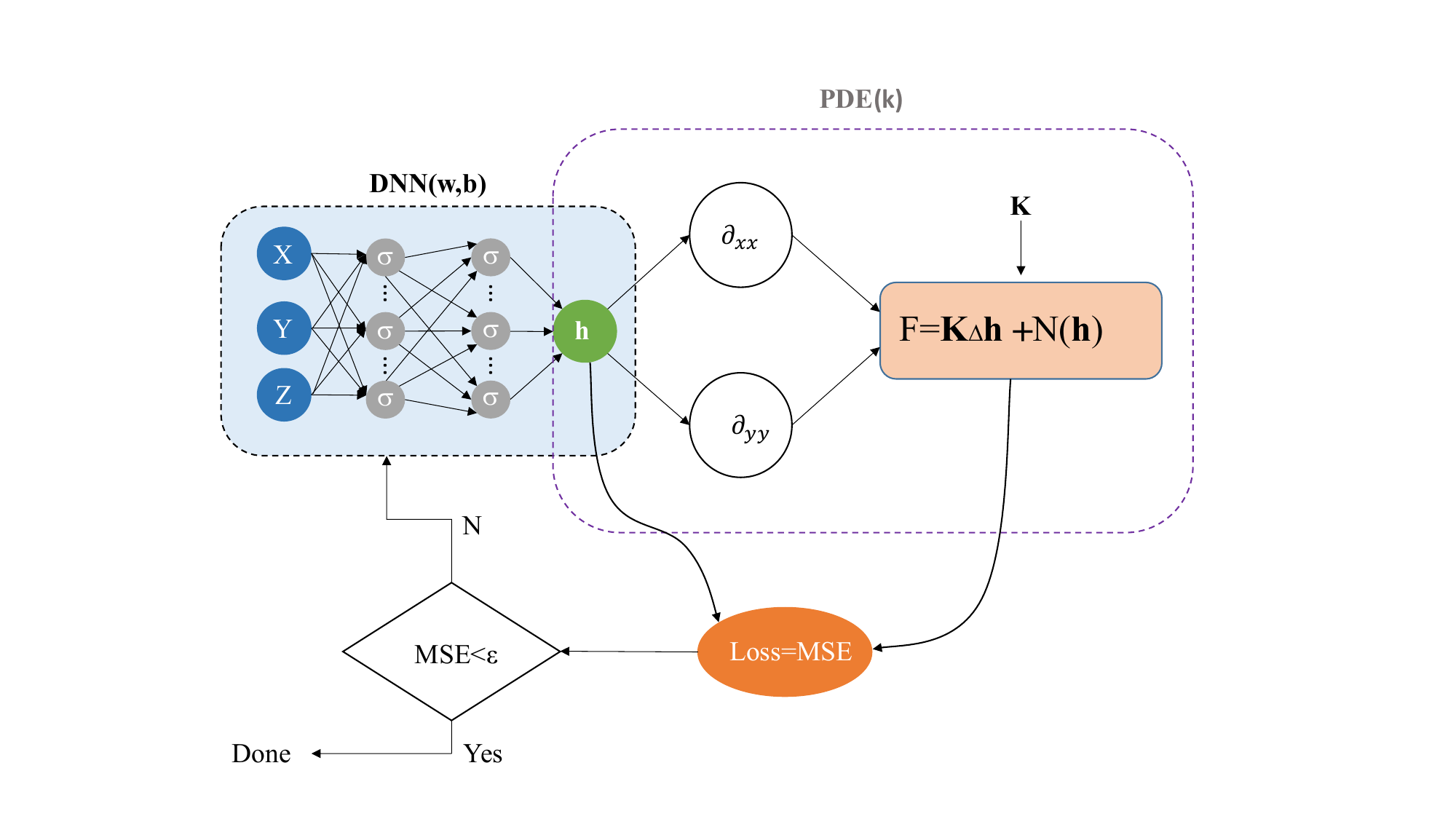}
	\centering
	\caption{Physics-informed neural networks}
	\label{fig:PINN}
\end{figure}

\subsubsection{Physics-informed neural network}
\label{subsection 3.2.1:neural network}

Physics-informed networks generator includes mainly neural networks interpreter, which represents the configuration of a NN, and physical information checker.
For the neural networks interpreter, it comprises of a deep neural networks with multiple layers: input layer, one or more hidden layers and output layer. Each layer consists of one or more nodes called neurons, shown in the Figure~\ref{fig:PINN} by small coloured circles, which is the basic unit of computation. For an interconnected structure, every two neurons in neighbouring layers have a connection, which is represented by a connection weight, see Figure~\ref{fig:PINN}. Mathematically, the output of a node is computed by:
\begin{equation}\label{eq:a}
y_{i}=\sigma_i(\sum_{j} w_{j}^{i}z_{j}^{i}+b^{i})
\end{equation}
with $z^{i}$ the input, $w^{i}$ weight, $b^{i}$ bias and $\sigma_i$ activation function. With those concepts, we can draw a definition here: 
\theoremstyle{definition}
\begin{definition}{(Feedforward Neural Network)}
	A generalized neural networks with activation can be written in a tuple form $\left((f_1,\sigma_1),...,(f_n,\sigma_n)\right)$, with $f_i$ an affine-line function $(f_i = W_i\textit{\textbf{x}}+b_i)$ that mapps $R^{i-1} \rightarrow R^{i}$ and activation $\sigma_i$ the mapping $R^{i} \rightarrow R^{i}$, which in all defines a continuous bounded function mapping $R^{D}$ to $R^{n}$:
	\begin{equation}
	FNN: \mathbb{R}^d \to \mathbb{R}^n, \; \textrm{with}\; \;   F^n\left(\textit{\textbf{x}};\theta\right) = \sigma_n\circ f_n \circ\cdots  \circ \sigma_1 \circ f_1
	\end{equation}
	where $d$ the dimension of the inputs, $n$ the number of field variables, $\theta$ consisting of hyperparameters such as weights and biases and $\circ$ denotes the element-wise composition.
\end{definition}

The universal approximation theorem \cite{FUNAHASHI1989183,HORNIK1989359} reveals that this continuous bounded function $F$ with nonlinear activation $\sigma$ can be adopted to capture the smoothness, nonlinear property of the system.
Accordingly, a theorem follows as \cite{hornik1991approximation}:
\begin{theorem}\label{theorem1}
	If $\sigma^i \in C^m(R^i)$ is non-constant and bounded, then $F^n$ is uniformly m-$dense$ in $C^m(R^n)$.
\end{theorem}

\subsubsection{Deep collocation method}
\label{subsubsection 3.2.2:dcm}
Collocation method is a widely used method seeking numerical solutions for ordinary, partial differential and integral equations \cite{atluri2005methods}. It is a a popular method for trajectory optimization in control theory. A set of randomly distributed points (also known as collocation points) is often deployed to represent a desired trajectory that minimizes the loss function while satisfying a set of constraints. The collocation method tends to be relatively insensitive to instabilities (such as blowing/vanishing gradients with neural networks) and is a viable way to train the deep neural networks \cite{agrawalcollocation}.

The modified Darcy equation~\eqref{eq:darcy with f} can be boiled down to the solution of a second order differential equations with boundary constraints. Hence we first discretize the physical domain with collocation points denoted by $\textit{\textbf{x}}\,_\Omega=(x_1,...,x_{N_\Omega})^T$. Another set of collocation points are employed to discretize the boundary conditions denoted by $\textit{\textbf{x}}\,_\Gamma(x_1,...,x_{N_\Gamma})^T$.
Then the hydraulic head $\hat{h}$ is approximated with the aforementioned deep feedforward neural network $\hat{h}^h (\bm{x};\theta)$. A loss function can thus be constructed to find the approximate solution $\hat{h}^h \left(\bm{x};;\theta\right)$ by minimizing of governing equation with boundary conditions. The mean squared error loss form is taken here.

Substituting $\hat{h}^h \left(\bm{x}\,_\Omega;\theta\right)$ into governing equation, we obtain
\begin{equation}
E'\left(\textit{\textbf{x}}\,_\Omega;\theta\right)=K(\bm{x})\hat{h}_{,ii}^{h}\left(\textit{\textbf{x}}\,_\Omega;\theta\right)+K_{,i}(\bm{x})\hat{h}^{h}_{,i}\left(\textit{\textbf{x}}\,_\Omega;\theta\right)-f\left(\textit{\textbf{x}}\,_\Omega\right),
\end{equation}
which results in a physical informed deep neural network $E'\left(\textit{\textbf{x}}\,_\Omega;\theta\right)$.

The boundary conditions illustrated in Section~\ref{section 2:hydraulic conductivity} can also be expressed by the neural network approximation $\hat{h}^h \left(\textit{\textbf{x}}\,_\Gamma;\theta\right)$ as:

\noindent On $\Gamma_{D}$, we have
\begin{equation}
\hat{h}^h \left(\textit{\textbf{x}}\,_{\Gamma_D};\theta\right)=\hat{h}_{MMS}\left(\textit{\textbf{x}}\,_{\Gamma_D}\right),
\label{eq:bound1}
\end{equation} 

\noindent On $\Gamma_{N}$,
\begin{equation}
\hat{q}_n^h \left(\textit{\textbf{x}}\,_{\Gamma_N};\theta\right) = -K\left(\textit{\textbf{x}}\,_{\Gamma_N}\right)\hat{h}_{MMS,n}\left(\textit{\textbf{x}}\,_{\Gamma_N}\right).
\label{eq:bound2}
\end{equation} 

Note the induced physical informed neural network $E'\left(\textit{\textbf{x}};\theta\right)$, $q\left(\textit{\textbf{x}};\theta\right)$ share the same parameters as $\hat{h}^h \left(\textit{\textbf{x}};\theta\right)$. Considering the generated collocation points in domain and on boundaries, they can all be learned by minimizing the mean square error loss function:
\begin{equation}\label{eq:lossform}
L\left(\theta\right)=MSE=MSE_{E'}+MSE_{\Gamma_{D}}+MSE_{\Gamma_{N}},
\end{equation}
with
\begin{equation}
\begin{aligned}
&MSE_{E'}=\frac{1}{N_d}\sum_{i=1}^{N_d}\begin{Vmatrix}
E'\left(\textit{\textbf{x}}\,_\Omega;\theta\right)
\end{Vmatrix}^2,\\
&MSE_{\Gamma_{D}}=\frac{1}{N_{\Gamma_D}}\sum_{i=1}^{N_{\Gamma_D}}\begin{Vmatrix}
\hat{h}^h \left(\textit{\textbf{x}}\,_{\Gamma_D};\theta\right)-\hat{h}_{MMS}\left(\textit{\textbf{x}}\,_{\Gamma_D}\right)
\end{Vmatrix}^2,\\ 
&MSE_{\Gamma_{N}}=\frac{1}{N_{\Gamma_N}}\sum_{i=1}^{N_{\Gamma_N}}\begin{Vmatrix}
\hat{q}_n\left(\textit{\textbf{x}}\,_{\Gamma_N};\theta\right)+K\left(\textit{\textbf{x}}\,_{\Gamma_N}\right)\hat{h}_{MMS,n}\left(\textit{\textbf{x}}\,_{\Gamma_N}\right)
\end{Vmatrix}^2,
\end{aligned}
\end{equation}
where $x\,_\Omega \in {R^N} $, $\theta \in {R^K}$ are the neural network parameters. $L\left(\theta\right)=
0$, $\hat{h}^h \left(\textit{\textbf{x}};\theta\right)$ is then a solution to hydraulic head. Here, the defined loss function measures how well the approximation satisfies the physics law (governing equation), boundaries conditions. Our goal is to find the a set of parameters $\theta$ that the  approximated potential $\hat{h}^h \left(\textit{\textbf{x}};\theta\right)$ minimizes the loss $L$. If $L$ is a very small value, the approximation $\hat{h}^h \left(\textit{\textbf{x}};\theta\right)$ is very closely satisfying governing equations and boundary conditions, namely
\begin{equation}
\hat{h}^h = \mathop{\arg\min}_{\theta \in R^K} L\left(\theta\right).
\end{equation}

The solution of groundwater flow problems by deep collocation method can be reduced to an optimization problem. To train the deep feedforward neural network, the gradient-based optimization algorithms such as Adam is employed. The idea is to take a descent step at collocation point $\textit{\textbf{x}}_{i}$ with Adam-based learning rates $\alpha_i$, 
\begin{equation}
\theta_{i+1} = \theta_{i} + \alpha_i \bigtriangledown_{\theta } L \left ( \textit{\textbf{x}}_i;\theta_i \right ).
\label{eq:Adma}
\end{equation}
And then the process in Equation~\eqref{eq:Adma} is repeated until a convergence criterion is satisfied.

The combined Adam-L-BFGS-B minimization algorithm is used to train the physics-informed neural networks in that adding physics constraints makes it more difficult to train. This strategy consists of training the network first using the Adam algorithm, and after a defined number of iterations, performing the L-BFGS-B optimization of the loss with a small limit of executions.

Further, the approximation ability of neural networks for the potential problems needs to be proved. The approximation power of neural networks for a quasilinear parabolic PDEs has been proved by Sirignano et al. \cite{sirignano2018dgm}. For groundwater flow problems, whose governing equation is an elliptic partial differential equation, the proof can be boiled down to:
\begin{equation}
\exists\;\;\hat{h}^h \in F^n, \;\;s.t. \;\;as\;\;n\rightarrow\infty,\;\;L(\theta)\rightarrow0,\;\;\hat{h}^h\rightarrow\hat{h}.
\end{equation}

The groundwater flow problems has a unique solution, s.t. $\hat{h} \in C^2(\Omega)$ with its derivatives uniformly bounded. Also, the heterogeneous hydraulic conductivity function $K(\textit{\textbf{x}})$ is assumed to be $C^{1,1}$ ($C^1$ with Lipschitz continuous derivative). The smoothness of the K ﬁeld is essentially determined by the correlation of the random ﬁeld $Y'$. According to \cite{cramer2013stationary}, the smoothness conditions are fulfilled if the correlation of $Y'$ has a Gaussian shape and is infinitely differentiable. For the source term, the smoothness of the source term is determined by the constructed manufactured solution $\hat{h}_{MMS}$, in Equations \ref{eq:mms1} and \ref{eq:mms2}, which is obvious continuous and infinitely differentiable $f\in C^\infty(\Omega )$.

\begin{theorem}\label{theorem2}
	With assumption that $\Omega$ is compact and considering measures $\ell_1$, $\ell_2$, and $\ell_3$ whose supports are constrained in $\Omega$, $\Gamma_D$, and $\Gamma_N$. Also, the governing Equation \eqref{eq:darcy with f} subject to~\ref{eq:modified boundary conditions} is assumed to have a unique classical solution and conductivity function $K(\textit{\textbf{x}})$ is assumed to be $C^{1,1}$ ($C^1$ with Lipschitz continuous derivative). Then, $\forall \;\; \varepsilon >0 $, $\exists\;\;  \lambda>0$, which may dependent on $sup_{\Omega}\left \|   \hat{h}_{ii}\right \|$ and $sup_{\Omega}\left \|   \hat{h}_{i}\right \|$, s.t. $\exists\;\;  \hat{h}^h\in F^n$, that satisfies $L(\theta)\leq \lambda\varepsilon$
\end{theorem}
\begin{proof}
	For governing Equation~\eqref{eq:darcy with f} subject to~\ref{eq:modified boundary conditions}, according to Theorem~\ref{theorem1},
	$\forall$ $\varepsilon\;\;  >0$, $\exists\;\;  \hat{h}^h\;\; \in\;\; F^n$, s.t. 
	\begin{equation}\label{eq:sup}
	\sup_{x\in \Omega}\left \|\hat{h}_{,i}\left(\textit{\textbf{x}}\,_\Omega\right)-  \hat{h}^h_{,i}\left(\textit{\textbf{x}}\,_\Omega\right)\right \|^2+\sup_{x\in \Omega}\left \|\hat{h}_{,ii}\left(\textit{\textbf{x}}\,_\Omega\right)-  \hat{h}^h_{,ii}\left(\textit{\textbf{x}}\,_\Omega\right)\right \|^2<\varepsilon	
	\end{equation}
	
Recalling that the Loss is constructed in the form shown in Equation~\eqref{eq:lossform}, for $MSE_G$, applying triangle inequality, and obtains:
\begin{equation}
\begin{aligned}
\begin{Vmatrix}
G\left(\textit{\textbf{x}}\,_\Omega;\theta\right)
\end{Vmatrix}^2\leqslant\begin{Vmatrix}
K(\bm{x}_\Omega)\hat{h}_{,ii}^{h}\left(\textit{\textbf{x}}\,_\Omega;\theta\right)
\end{Vmatrix}^2+\begin{Vmatrix}
K_{,i}(\bm{x}_\Omega)\hat{h}^{h}_{,i}\left  (\textit{\textbf{x}}\,_\Omega;\theta\right)
\end{Vmatrix}^2+\begin{Vmatrix}
f\left(\textit{\textbf{x}}\,_\Omega\right)
\end{Vmatrix}^2
\end{aligned}
\end{equation}

Also, considering the $C^{1,1}$ conductivity function $K(\textit{\textbf{x}})$, $\exists \;\;M_1>0,\;\;M_2>0$, $\exists \;\; x \in\;\Omega$, $\left \|  K(\textit{\textbf{x}})\right \|\leqslant M_1$, $\left \|  K_{,i}(\textit{\textbf{x}})\right \|\leqslant M_2$. 
From Equation~\eqref{eq:sup}, it can be obtained that:
\begin{equation}\label{eq:boundsup1}
\begin{aligned}
\int_{\Omega}K_{,i}^2(\bm{x}_\Omega)\left( \hat{h}_{,i}^h-\hat{h}_{,i} \right )^2d\ell_1\leqslant M_2^2 \varepsilon^2\ell_1(\Omega) \\
\int_{\Omega}K^2(\bm{x}_\Omega)\left ( \hat{h}_{,ii}^h-\hat{h}_{,ii} \right )^2d\ell_1\leqslant M_1^2 \varepsilon^2\ell_1(\Omega)
\end{aligned}
\end{equation}

On boundaries $\Gamma_{D}$ and $\Gamma_{N}$, 
\begin{equation}\label{eq:boundsup2}
\begin{aligned}
\int_{\Gamma_{D}}\left (\hat{h}^h \left(\textit{\textbf{x}}\,_{\Gamma_D};\theta\right)-\hat{h}\left(\textit{\textbf{x}}\,_{\Gamma_D};\theta\right)\right )^2d\ell_2\leqslant \varepsilon^2\ell_2(\Gamma_{D})\\
\int_{\Gamma_{N}}K^2(\bm{x}_{\Gamma_N})\left (\hat{h}^h_{,n} \left(\textit{\textbf{x}}\,_{\Gamma_N};\theta\right)-\hat{h}_{,n}\left(\textit{\textbf{x}}\,_{\Gamma_N};\theta\right)\right )^2d\ell_3\leqslant M_1^2\varepsilon^2\ell_3(\Gamma_{N})
\end{aligned}
\end{equation}

Therefore, using Equations~\eqref{eq:boundsup1} and~\eqref{eq:boundsup2}, as $n\rightarrow\infty$
\begin{equation}
\begin{aligned}
L\left(\theta\right)=\frac{1}{N_\Omega}\sum_{i=1}^{N_\Omega}\begin{Vmatrix}
K(\bm{x}_\Omega)\hat{h}_{,ii}^{h}\left(\textit{\textbf{x}}\,_\Omega;\theta\right)+K_{,i}(\bm{x}_\Omega)\hat{h}^{h}_{,i}\left  (\textit{\textbf{x}}\,_\Omega;\theta\right)-f\left(\textit{\textbf{x}}\,_\Omega\right)
\end{Vmatrix}^2+\\ 
\frac{1}{N_{\Gamma_D}}\sum_{i=1}^{N_{\Gamma_D}}\begin{Vmatrix}
\hat{h}^h \left(\textit{\textbf{x}}\,_{\Gamma_D};\theta\right)-\hat{h}_{MMS}\left(\textit{\textbf{x}}\,_{\Gamma_D}\right)
\end{Vmatrix}^2+\frac{1}{N_{\Gamma_N}}\sum_{i=1}^{N_{\Gamma_N}}\begin{Vmatrix}
-K(\bm{x}_{\Gamma_N})\frac{\partial \hat{h}\left(\textit{\textbf{x}}_{\Gamma_N};\theta\right)}{\partial n}+K(\bm{x}_{\Gamma_N})\frac{\partial \hat{h}{MMS}}{\partial n}
\end{Vmatrix}^2 \\
\leqslant \frac{1}{N_\Omega}\sum_{i=1}^{N_\Omega}\begin{Vmatrix}
K(\bm{x}_\Omega)\hat{h}_{,ii}^{h}\left(\textit{\textbf{x}}\,_\Omega;\theta\right)
\end{Vmatrix}^2+\frac{1}{N_\Omega}\sum_{i=1}^{N_\Omega}\begin{Vmatrix}
K_{,i}(\bm{x}_\Omega)\hat{h}^{h}_{,i}\left  (\textit{\textbf{x}}\,_\Omega;\theta\right)
\end{Vmatrix}^2\\ +\frac{1}{N_\Omega}\sum_{i=1}^{N_\Omega}\begin{Vmatrix}
f\left(\textit{\textbf{x}}\,_\Omega\right)
\end{Vmatrix}^2+\frac{1}{N_{\Gamma_D}}\sum_{i=1}^{N_{\Gamma_D}}\begin{Vmatrix}
\hat{h}^h \left(\textit{\textbf{x}}\,_{\Gamma_D};\theta\right)-\hat{h}_{MMS}\left(\textit{\textbf{x}}\,_{\Gamma_D}\right)
\end{Vmatrix}^2+\\
\frac{1}{N_{\Gamma_N}}\sum_{i=1}^{N_{\Gamma_N}}\begin{Vmatrix}
-K(\bm{x}_{\Gamma_N})\frac{\partial \hat{h}\left(\textit{\textbf{x}}_{\Gamma_N};\theta\right)}{\partial n}+K(\bm{x}_{\Gamma_N})\frac{\partial \hat{h}{MMS}}{\partial n}
\end{Vmatrix}^2 \\
\leqslant (M_2^2+M_1^2+1)\varepsilon^2\ell_1(\Omega)+\varepsilon^2\ell_2(\Gamma_{D})+M_1^2\varepsilon^2\ell_3(\Gamma_{N})=K\varepsilon
\end{aligned}
\end{equation}
\end{proof}

With the hold of Theorem~\ref{theorem2} and conditions that $\Omega$ is a bounded open subset of R, $\forall n\in N_+$, $\hat{h}^h\in \;F^n \;\in L^2(\Omega)$, it can be concluded from Sirignano et al. \cite{sirignano2018dgm} that:
\begin{theorem}\label{theorem3}
	$\forall \;p<2$, $\hat{h}^h\in \;F^n$ converges to $\hat{h}$ strongly in $L^p(\Omega)$ as $n\rightarrow \infty$ with $\hat{h}$ being the unique solution to the potential problems.
\end{theorem}
In summary, for feedforward neural networks $F^n \in L^p$ space ($p<2$), the approximated solution $\hat{h}^h\in F^n$  will converge to the solution to this PDE. This will justify the application of physics-informed and data-driven deep learning method in solving groundwater flow problems.

\subsection{Sensitivity analyses(SA)}
\label{subsection 3.3:SA}
Sensitivity analysis is also a very useful tool in the model parameter calibration process, which aims to determine which aspects of the model is easiest to introduce uncertainty into the system description. Sensitivity analysis can be used to determine the influence of each parameter of the model on the output, and those that are most important to consider in the model calibration process. Parameters that have a large impact on the output can be disregarded if they have little or no effect on the model results. This will significantly reduce the workload of model calibration \cite{gardner1981comparison,henderson1996sensitivity,majkowski1981multiplicative}. 

At this stage, global sensitivity analysis of parameters in hydrological models has gradually become a research hotspot. In this work, the  parameter sensitivity analysis experiment contributes to the whole NAS model by offering prior knowledge of the DCM model, which helps to reduce dimensions of the search space and further improve the computational efficiency for optimization method.

 Global sensitivity analysis methods can be subdivided into qualitative, such as Morris method \cite{morris1991factorial}, Fourier amplitude sensitivity test (FAST) \cite{mcrae1982global}, and quantitative analysis methods, such as sobol method \cite{nossent2011sobol}, extend FAST \cite{zhang2012sensitivity}. Scholars have conducted numerous experiments to compare the advantages and disadvantages between different methods \cite{wang2019practical,herman2013method,brevault2013comparison}. The results shows, that Sobol' method can provide quantitative results for SA, but it requires a large number of runs to obtain stable results. eFAST is more efficient and stable than Sobol' method, and can therefore be seen as a good alternative to the Sobol method. The method of Morris is able to correctly screen the most and least sensitive parameters for a highly parameterized model with 300 times fewer model evaluations than the Sobol' method. Therefore, we will take the same approach as Crosetto \cite{crosetto2001uncertainty} did, i.e., first test all the hyper-parameters using the Morris method, remove the two most and least influential parameters, then filter them again, but with the eFAST method. In this way, we can get the highest accuracy in a relatively small amount of time.

\subsubsection{Morris method}
\label{subsubsection 3.3.1:morris}
Morris method \cite{morris1991factorial} is proposed by Morris in 1991. It works as follows: select a variable $X$ from the model parameters, and also specify an objective function $y(x)=f(X_1,X_2,...X_n)$, change these variables $X_i$ by specific ranges, calculate the error $e_i$ to discern the extent to which parameter changes affect the output function. Taking $n$ elements as input parameters to the model, the sensitivity of $n$ parameters can be obtained. It's very effective in calculations. 
\begin{equation}\label{eq:morris}
EE_i=\frac{f(x_1,...,x_i+\Delta_i,...,x_n)-f(x)}{\Delta_i}
\end{equation}
where $f(x)$ represents the prior point in the trajectory. Using the single trajectories shown in Equation~\eqref{eq:morris}, the basic effect of each parameter can be calculated with only $p+1$ model evaluations. After sampling the trajectories, the resulting set of basic effects are then averaged to obtain total-order sensitivity of the $i$-th parameter $\mu_{i}^{*}$,
\begin{equation}\label{eq:mu}
\mu_{i}^{*}=\frac{1}{n}\sum_{j=1}^{n}\left| EE_{i}^{j}\right|
\end{equation}
and its variance can be calculated by,
\begin{equation}\label{eq:sigma}
\sigma_{i}^{2}=\frac{1}{n-1}\sum_{j=1}^{n}( EE_{i}^{j}-\mu_i)^2
\end{equation}

Both the mean $\mu^*$ and the variance $\sigma^2$ are important to the sensitivity, so we consider $\sqrt{\sigma^2+{\mu^*}^2}$ as the criteria for judging.

\subsubsection{eFast method}
\label{subsubsection 3.3.2;eFAST}
The eFAST method \cite{saltelli1999quantitative} obtains the spectrum of the Fourier series by Fourier transformations, through which the spectrum curve is obtained by each parameter and the parameter's Variance of model results due to interactions. According to a suitable search function, the model $y(x)=f(X_1,X_2,...X_n)$ can be transformed by the Fourier transform into $y= f(s)$ 
\begin{equation}\label{eq:FT}
y=f(s)=\sum_{j=-\infty}^{+\infty}\big (A_j cos(js)+B_j sin(js)\big ),
\end{equation}
with,
\begin{equation}\label{eq:FT1}
A_j=\frac{\pi}{2}\int_{\frac{\pi}{2}}^{-\frac{\pi}{2}}f(s)cos(js)\mathrm{d}s,
\end{equation}
\begin{equation}\label{eq:FT2}
B_j=\frac{\pi}{2}\int_{\frac{\pi}{2}}^{-\frac{\pi}{2}}f(s)sin(js)\mathrm{d}s.
\end{equation}

The spectral curve of the Fourier progression is defined as$\Lambda_j=A_j^2+B_j^2$,the variance of the model results due to the uncertainty in the parameter $X_i$ is
\begin{equation}\label{eq:D}
D_i=\sum_{p\in Z_0}\Lambda_p\omega_i,
\end{equation}
with $\omega_1$, parametric frequency, $\Lambda$, Spectrum of Fourier transforms, $Z_0$, Non-zero integers.

The total variance can be obtained by cumulatively summing the spectra at all frequencies,
\begin{equation}\label{eq:Dsum}
D=2\sum_{j=1}^{\infty}\Lambda_j.
\end{equation}

The sensitivity of the parameters to the output is,
\begin{equation}\label{eq:S}
S_{i}^{FAST}=\frac{D_i}{D}.
\end{equation}

When finding the total sensitivity of $X_i$, the frequency of $X_i$ is set to $\omega_i$, while a different frequency $\omega'$ is set for all other parameters. By calculating the frequency $\omega_i$ and its higher resonance $p\omega_i$ spectra, the output variance $D_{-i}$ due to the influence of all parameters except $X_i$ and their interrelationships can be obtained.Thus,
\begin{equation}\label{eq:Ssum}
S_{T_{i}}=\frac{D-D_{-i}}{D}.
\end{equation}

\subsubsection{SALib}
\label{subsubsection 3.3.3:SALib}
SALib \cite{herman2017salib} is a sensitivity analysis library in Python, which provides several sensitivity analysis methods, developed by Jon Herman and his coworkers, such as Sobol, Morris, and FAST. One can easily use it to perform sensitivity analysis on a model that has built already. And it comes also with a sensitivity result visualization feature that helps us to present our results in a good way.

\subsection{Search methods for NNs}
\label{subsection 3.4:operator}
After the sensitivity analysis screening, hyperparameter optimization is tapped to find the detailed most suitable parameters for the physics-informed neural architectures. Hyper-parameter optimization is a combinatorial optimization problem and cannot be simply optimized by the gradient descent methods which may incorporate general parameters. The time cost in evaluating a set of hyperparameter configurations is usually very high, especially when the model is complex with many parameters. So finding an appropriate algorithm is crucial. There are a lot of optimization methods available these days. In this application, the classical randomization search method and Bayesian optimization method and some recent proposed optimization method, Hyperband algorithm and Jaya algorithm, with the latter a branch of heuristic learning method. 

\subsubsection{Randomization search method(RSM)}
\label{subsubsection 3.4.1:RSM}
Randomization search method is one of the most basic algorithms, simply draw the hyper-parameters from the search space in random combinations, to find a configuration that performs best. The algorithm itself is very simple and does not make use of correlations between different combinations of hyperparameters. If there is a narrow cost-minimum region close to the boundary, then any solutions close to it may be excluded because those solutions' costs are all very high, so we have almost no probability of getting such a global minimum pathway. This is a common flaw of randomization algorithms and there is no good way to solve it. 

\subsubsection{Bayesian optimization}
\label{subsubsection 3.4.2:BO}
Bayesian optimization is an adaptive hyper-parametric search method that predicts the next combination that is likely to bring the most benefit based on the currently tested hyper-parametric combinations \cite{snoek2012practical}. Assuming that the function $f(x)$ of hyperparameter optimization obeys the Gaussian process, then $p\big (f(x)\mid x\big )$ is a normal distribution. The Bayesian optimization process is modeled as a Gaussian process based on the results of existing $N$ group experiments, $H=\left\lbrace x_n,y_n\right\rbrace _{n=1}^{N}$, and calculate the posterior distribution $p\big (f(x)\mid x,H\big )$ of $f(x)$.

After obtaining the posterior distribution of the objective function, an acquisition function $a(x,H)$ is defined to trade off in sampling where the model predicts a high objective and sampling at locations where the prediction uncertainty is high. The goal is left to maximize the acquisition function to determine the next sampling point. Assuming $y^*=\min{y_n,1\le n\le N}$ is the optimal value in the currently existing sample, the desired improvement function is,
\begin{equation}\label{eq:EI}
EI(x,H)=\int_{-\infty}^{\infty}max(y^*-y,0)p(y\mid x,H)\mathrm{d}y.
\end{equation}

The algorithm works as follows:
\begin{table}[H]
	\captionsetup{width=0.85\columnwidth}
	\caption{\textbf {Bayesian optimization algorithm}} % Table caption, can be commented out if no caption is required
	\vspace{-0.3cm}
	\centering
	\resizebox{0.85\columnwidth}{!}{%
		\begin{tabular}{l} 
			\toprule % Top horizontal line
			%\quad \quad \quad \quad \quad \quad \quad \quad \quad \quad %\textbf {The pseudo-code}\\ % Column names row
			%\midrule % In-table horizontal line
			\quad \textbf {input : $f(x)$, $T$,$a(x,H)$}\\ % Content row 1
			1 \textbf{$H\gets\theta$};\\
			2 \textbf{Random initialization of Gaussian processes, calculate $p\big (f(x)\mid x,H\big )$};\\
			3 \textbf{For} $t\gets 1$ \textbf{to} $T$ \textbf{do}\\
			4 \quad $x^{'}\gets arg max_x a(x,H)$;\\
			5 \quad evaluate $y^{'}=f(x^{'})$;\\
			6 \quad $H\gets H \bigcup (x^{'},y^{'})$;\\
			7 \quad Remodeling Gaussian processes according to $H$, calculate $p\big (f(x)\mid x,H\big )$;\\
			8 \textbf{end} ;\\
			\quad \textbf{output: $H$}\\
			\bottomrule % Bottom horizontal line
		\end{tabular}
	}
	\label{tab:bo code} 
\end{table}

\subsubsection{Hyperband algorithm}
\label{sunsubsection 3.4.3:hb}
Hyperband \cite{li2017hyperband} is an algorithm derived from the SuccessiveHalving algorithm \cite{jamieson2016non}. The core idea of this algorithm is that, assume that there are $n$ sets of hyperparameter combinations, then uniformly allocate budgets to these $n$ sets of hyperparameters and perform a validation evaluation, eliminate half of the poorly performing hyperparameter sets based on the validation results, and then iterate the above process until an optimal hyperparameter combination is found. The disadvantage of this algorithm is that it is not possible to control the ratio between the budget and the hyperparameters, which leads to the fact that when either side is too large, the predictions will not be very good. Hyberband, on the other hand, offers a way of weighing the two. The specific algorithm is shown in Table~\ref{tab:Hyperband code}:
\begin{table}[H]
	\captionsetup{width=0.85\columnwidth}
	\caption{\textbf {Hyperband ALgorithm}} % Table caption, can be commented out if no caption is required
	\vspace{-0.3cm}
	\centering
	\resizebox{0.75\columnwidth}{!}{%
		\begin{tabular}{l} 
			\toprule % Top horizontal line
			%\quad \quad \quad \quad \quad \quad \quad \quad \quad \quad %\textbf {The pseudo-code}\\ % Column names row
			%\midrule % In-table horizontal line
			\quad \textbf {input : R, $\eta$}\\ % Content row 1
			1  \textbf {initialization}: ${ s }_{ max }=\left\lfloor { log }_{ \eta  }R \right\rfloor $ and $B=\left( { s }_{ max }+1 \right) R$;\\
			2 \quad \textbf {for} $s\in \left\{ { s }_{ max }\quad ,\quad { s }_{ max }-1,...,\quad 0 \right\}$ \quad \textbf {do}\\
			3 \quad \quad \quad  $n=\left\lceil \frac { B }{ R } \frac { { \eta  }^{ s } }{ s+1 }  \right\rceil $,  $r=R{ \eta  }^{ -s }$;\\
			4 \quad \quad \quad  ${ \Lambda  }_{ s }$= get$\_$hyperparameter$\_ $configuration(n);\\
			5 \quad \quad \quad  \textbf {for}  $i\in \left\{ 0,...,s \right\}$ \textbf {do}\\
			6 \quad \quad \quad \quad \quad  ${ n }_{ i }=\left\lfloor n{ \eta  }^{ -i } \right\rfloor $,   ${ r }_{ i }=r{ \eta  }^{ i }$;\\
			7 \quad \quad \quad \quad \quad  $\L \left( { \Lambda  }_{ s } \right) =\left\{ run\_ then\_ return\_ val\_ loss(\lambda ,\quad { r }_{ i })|\lambda \in { \Lambda  }_{ s } \right\} $;\\
			8 \quad \quad \quad \quad \quad ${ \Lambda  }_{ s }=top\_ k\left( { \Lambda  }_{ s },\quad \L \left( { \Lambda  }_{ s } \right) ,\quad \left\lfloor \frac { { n }_{ i } }{ \eta  }  \right\rfloor  \right)$ ;\\
			9 \quad \quad \quad \quad \textbf {end};\\
			10 \quad \textbf {end};\\
			11 \quad \textbf {output:} configuration $\lambda$ with lowest validation loss seen so far;\\
			\bottomrule % Bottom horizontal line
		\end{tabular}
	}
	\label{tab:Hyperband code} 
\end{table}

\subsubsection{Jaya algorithm}
\label{sunsubsection 3.4.4:jaya}
The Jaya algorithm is easy to implement and does not require adjustment of any algorithm-related parameters. During each iteration, the solution is randomly updated with the following equation:
\begin{equation}\label{eq:jaya}
A(i+1,j,k)=A(i,j,k)+r(i,j,1)\big (A(i,j,b)-\left|A(i,j,k)\right| )-r(i,j,2)(A(i,j,\omega)-\left| A(i,j,k)\right|\big ),
\end{equation}
$r(i,j,1)$ and $r(i,j,1)$ are a random number taken from the range $[0,1]$, to make sure good diversity of algorithms. The main goal of the Jaya algorithm is to improve the adaptation of each candidate solution in the population. Thus, the Jaya algorithm attempts to move the objective function value of each solution towards the best solution by updating the values of the variables. Once the values of the variables are updated, the updated solution is compared to the corresponding old solution, and the next generation considers only the good solution. In the Jaya algorithm, each generation of solutions is close to the best solution, while the candidate solutions move away from the worst solution. Thus, good concentration and diversity are achieved in the search process.

With the Jaya algorithm, the objective function value gradually approaches the optimal solution by updating the value of the variable. In the process, the fitness of each candidate solution in the population is improved. The process of the Jaya algorithm is presented in following flowchart \cite{rao2019jaya}.

\begin{figure}[H]
	\captionsetup{width=0.9\columnwidth}
	\includegraphics[height=9cm]{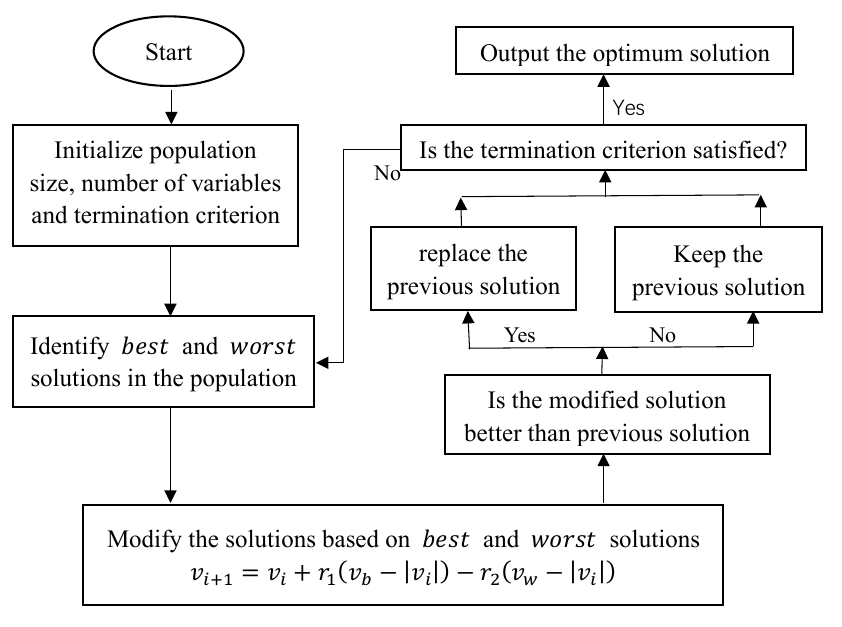}
	\centering
	\caption{Flowchart of Jaya algorithm}
	\label{fig:jayaflowchart}
\end{figure}

The performance of Jaya algorithm is reflected by the minimum function estimation.

\subsection{Transfer learning (TL)}
\label{subsection 3.5:TL}
It is introduced in the previous section that, for the physics-informed neural networks, since physical constraints are added to the loss function, the model becomes difficult to train. The combined optimizer is adopted for the model training. To improve the computational efficiency and inherit the learnt knowledge from the trained model, transfer learning algorithm is added to the entire model. Transfer learning is a research method in the field of machine learning. It focuses on storing knowledge gained while solving one problem and applying it to a different but related problem. The basic architecture of Transfer learning method of this Model is shown in Figure~\ref{fig:TL}. It is mainly composed of a Pre-train model and several Fine-tune models. For Pre-train model, during the neural architecture procedure, the optimum neural architecture configuration is obtained through a hyperparameter optimization algorithm and meanwhile saving the corresponding weights and biases. Then the weights and biases are transferred for fine-tuning model. It has been proven in the numerical example section that this inheritance method can greatly shorten the time of program calculation and improve learning efficiency. With transfer learning, for different statistical parameters involved in the random log-hydraulic conductivity ﬁeld, there is no need to train the whole model from the scratch and the solution to the modified Darcy equation can be easily yielded in Equation \eqref{eq:darcy with f} with less iteration, lower learning rate and even more desirable accuracy. 

\begin{figure}[H]
	\captionsetup{width=0.9\columnwidth}
	\includegraphics[height=9cm]{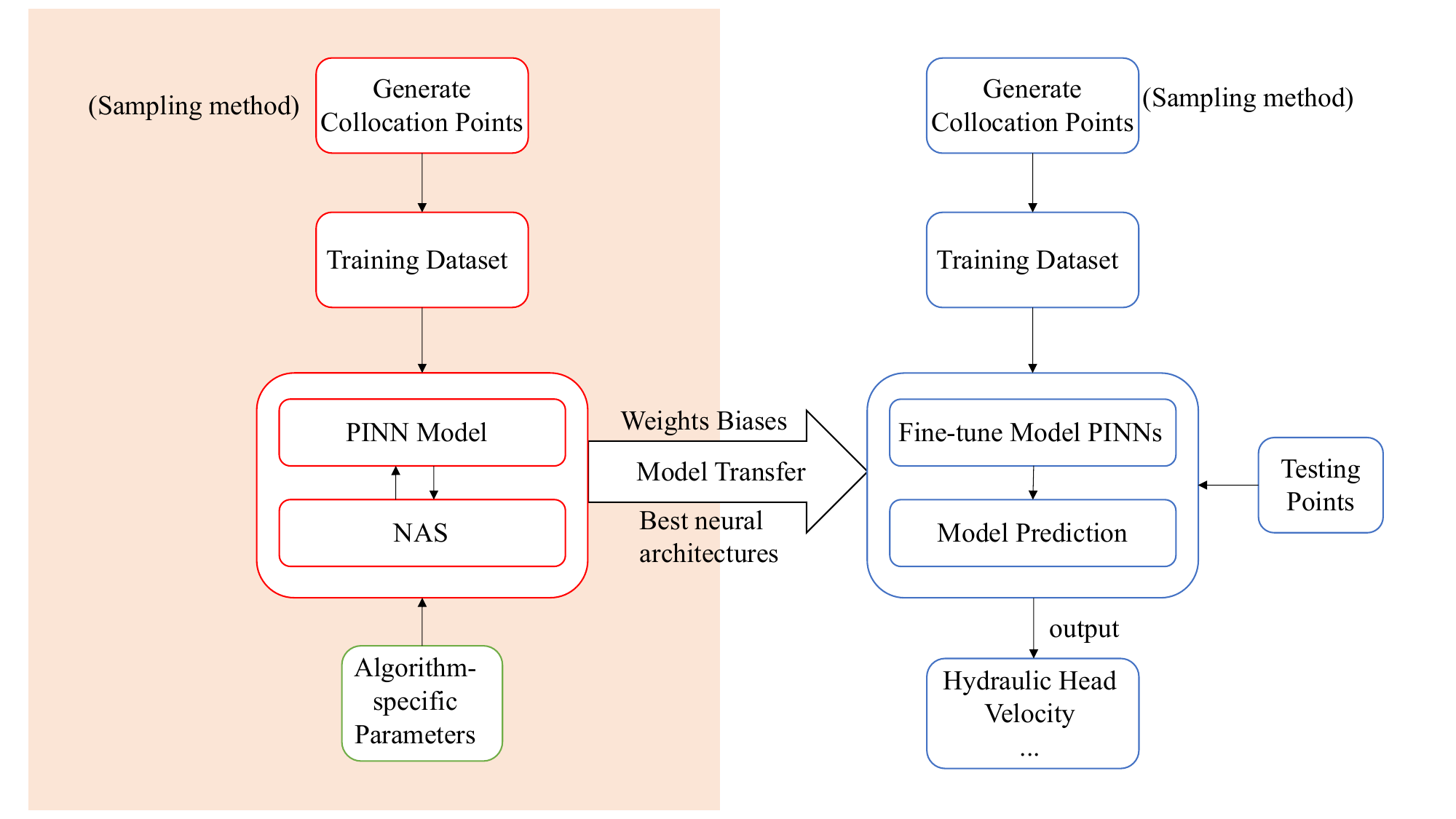}
	\centering
	\caption{Transfer Learning Schematic}
	\label{fig:TL}
\end{figure}

\section{Numerical examples}
\label{section 4:numerical examples}
In this section, numerical examples in different dimensions and boundary conditions are studied and compared. Firstly, the choice of exponential and Gaussian correlation functions is discussed, and after experimentally arriving at the optimal choice, conduct our numerical experiments with the heterogeneous hydraulic conductivity field constructed. Next, we filter the algorithm-specific parameters by means of sensitivity analysis and select the parameters that have the greatest impact on the model as our search space. Then four different hyperparameter optimization algorithms are compared in both accuracy and efficiency in hope to find a trade-off search method for the NAS model. The relative error in Equation \eqref{eq:Estimation Strategy} between the predicted results and the manufactured solution are obtained to built the search strategy for the NAS model. The results of the above selection from the proposed NAS based model are then substituted into the already built PINNs, and we can start solving the groundwater flow problem. For comparison, we use a finite difference method to fit the partial differential equations under the same conditions as the PINNs model. The simulations are done on a 64-bit Windows 10 server with Intel(R) Core(TM) i7-7700HQ CPU, 8GB memory. The accuracy of the numerical results by using the relative error of the hydraulic head. The relative error is defined as:
\begin{equation}\label{eq:error}
\delta \hat{h}=\frac{\| \hat{h}_{predict}-\hat{h}_{MMS}\|}{\|\hat{h}_{MMS}\|}
\end{equation} 
Here, $\|\cdot\|$ refers to the $l^2-norm$.

\subsection{Comparison of Gaussian and exponential correlations}
\label{subsection 4.1:comparison of gauss and exp}
Before we proceed to the next step of our analysis, we first compare the two correlation coefficients, Gaussian and exponential. Those two are the most widely used correlations for random field generation, it is extremely significant to figure out the most suitable one for neural network approximators. We calculated the results obtained with these two correlation coefficients for the one-dimensional (1D), two-dimensional (2D), and three-dimensional (3D) stochastic groundwater flow cases, respectively, with the same choice of parameters. The number of hidden layers and the neurons per layer are uniformly set to 6 and 16. Meanwhile, the model with transfer learning (TL) and without transfer learning are compared accordingly. The results are shown below:
\subsubsection{One dimensional groundwater flow with both correlations}
\label{subsection 4.1.1:1D comparison of gauss and exp}
The non-homogeneous 1D flow problem for Darcy equation can be reduced to solve Equation \eqref{eq:darcy with f} subjected to Equation \eqref{eq:b1}. The hydraulic conductivity $K$ is constructed from Equation \eqref{eq:u} by Radom spectral method as Equation \eqref{eq:k1}. The source term $f$ adopting manufactured solution Equation \eqref{eq:u1} can thus be obtained as Equation \eqref{eq:f11}. The detailed derivation can be retrieved from \ref{appendix b}. 

The relative errors $\delta \hat{h}$ of the predicted hydraulic head results for exponential and Gaussian correlation of the $ln(K)$ ﬁeld are shown in Tables~\ref{tab:Table3} and \ref{tab:Table4}. For the presented benchmark, it is obvious the Gaussian correlation is much more accurate for all $N$ and $\sigma^2$. With transfer learning model, the accuracy even improves. The predicted hydraulic head and velocity and the manufactured solution for both exponential and Gaussian correlations with $\sigma^2=0.1$ and $N=2000$ are shown in Figure \ref{fig:1Dtest}. It is clear the predicted results with proposed deep collocation method nearly coincides with the manufactured solution \eqref{eq:u1} in the 1D domain.

\begin{table}[H]  
	\captionsetup{width=0.85\columnwidth}
	\caption{$\delta \hat{h}$ for 1D case computed with exponential correlation in different variance and number of modes}
	\vspace{-0.3cm}
	\centering
	\resizebox{0.8\columnwidth}{!}{
		\begin{tabular}{l|c|c|c|c|c|c} 
			\toprule 
			\toprule 
			\multirow{2}*{\diagbox{$N$}{$\sigma^2$}}&\multicolumn{2}{c|}{0.1}&\multicolumn{2}{c|}{1}&\multicolumn{2}{c}{3}\\
			\cline{2-7}
			~ &without TL&with TL&without TL&with TL&without TL&with TL\\  
			\midrule
			500&1.184e-3&1.797e-4&1.100e-2&4.884e-4&1.159e-1&5.360e-4\\ 
			\midrule
			1000&2.437e-2&2.354e-4&9.026e-3&5.282e-4&3.752e-2&1.754e-3\\ 
			\midrule
			2000&5.789e-4&1.007e-4&3.813e-3&5.939e-4&3.532e-2&4.316e-3\\ 
			\bottomrule
		\end{tabular}
	}
	\label{tab:Table3}
\end{table}

\begin{table}[H]  
	\captionsetup{width=0.85\columnwidth}
	\caption{$\delta \hat{h}$ for 1D case computed with Gaussian correlation in different variance and number of modes for one dimensional case}
	\vspace{-0.3cm}
	\centering
	\resizebox{0.8\columnwidth}{!}{
		\begin{tabular}{l|c|c|c|c|c|c} 
			\toprule 
			\toprule 
			\multirow{2}*{\diagbox{$N$}{$\sigma^2$}}&\multicolumn{2}{c|}{0.1}&\multicolumn{2}{c|}{1}&\multicolumn{2}{c}{3}\\
			\cline{2-7}
			~ &without TL&with TL&without TL&with TL&without TL&with TL\\ 
			\midrule
			500&1.211e-4&1.137e-4&9.065e-4&1.204e-4&7.539e-3&1.690e-4\\
			\midrule
			1000&1.317e-4&1.133e-4&8.312e-4&1.200e-4&2.864e-3&3.662e-4\\
			\midrule
			2000&1.158e-4&1.333e-4&2.811e-4&1.538e-4&2.904e-3&5.756e-4\\
			\bottomrule
		\end{tabular}
	}
	\label{tab:Table4}
\end{table}

\begin{figure}[H]
	\captionsetup{width=0.85\columnwidth}
	\centering
	\begin{subfigure}[b]{6.0cm}
		\centering\includegraphics[height=6cm,width=6.0cm]{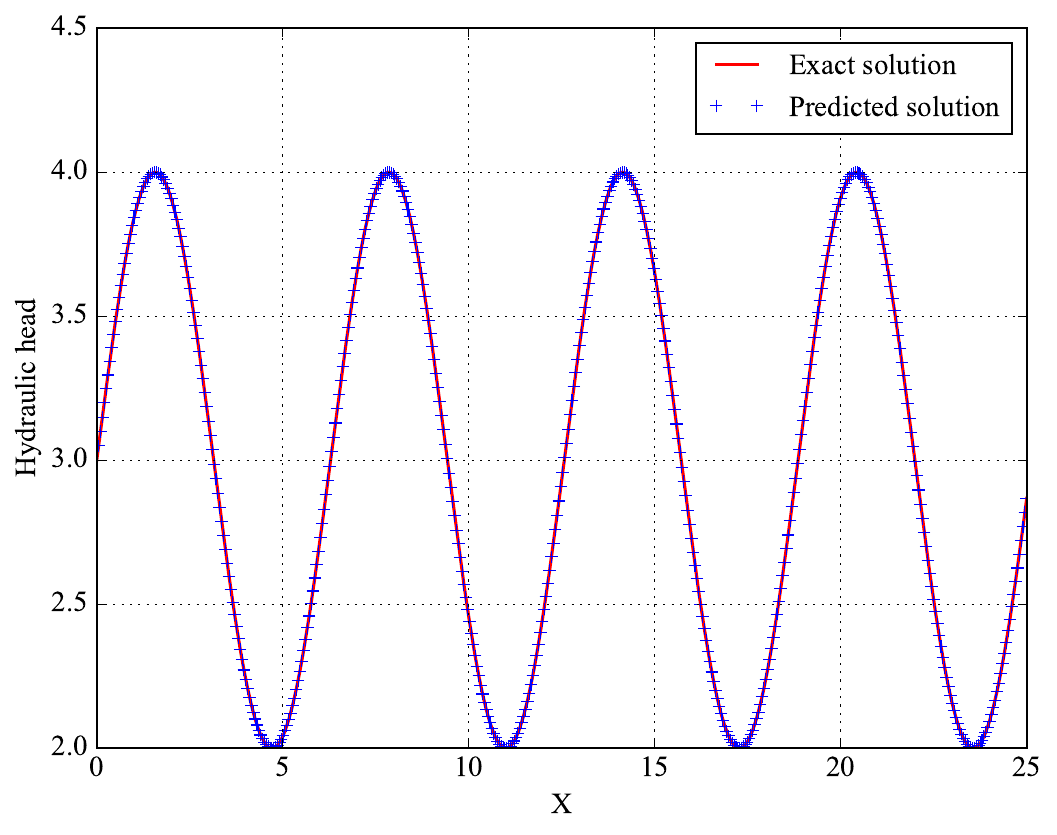}   
		\caption{}
	\end{subfigure}%
	\hspace{0.5cm}
	\begin{subfigure}[b]{6.0cm}
		\centering\includegraphics[height=6cm,width=6.0cm]{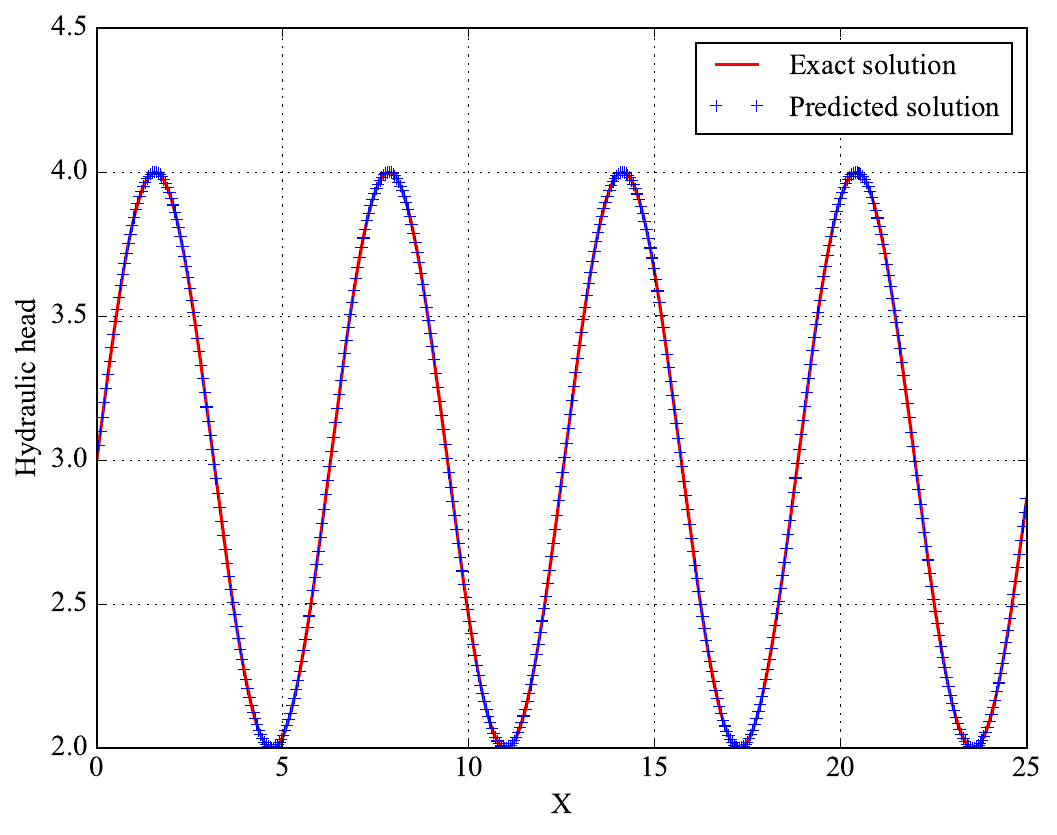}   
		\caption{}
	\end{subfigure}% 
	\caption{One dimensional hydraulic head when $\sigma^2=0.1$, $N=2000$ with $\left(a\right)$ exponential correlation and $\left(b\right)$ Gaussian correlation}
	\label{fig:1Dtest}
\end{figure}

\begin{figure}[H]
	\captionsetup{width=0.85\columnwidth}
	\centering
	\begin{subfigure}[b]{6.0cm}
		\centering\includegraphics[height=6cm,width=6.0cm]{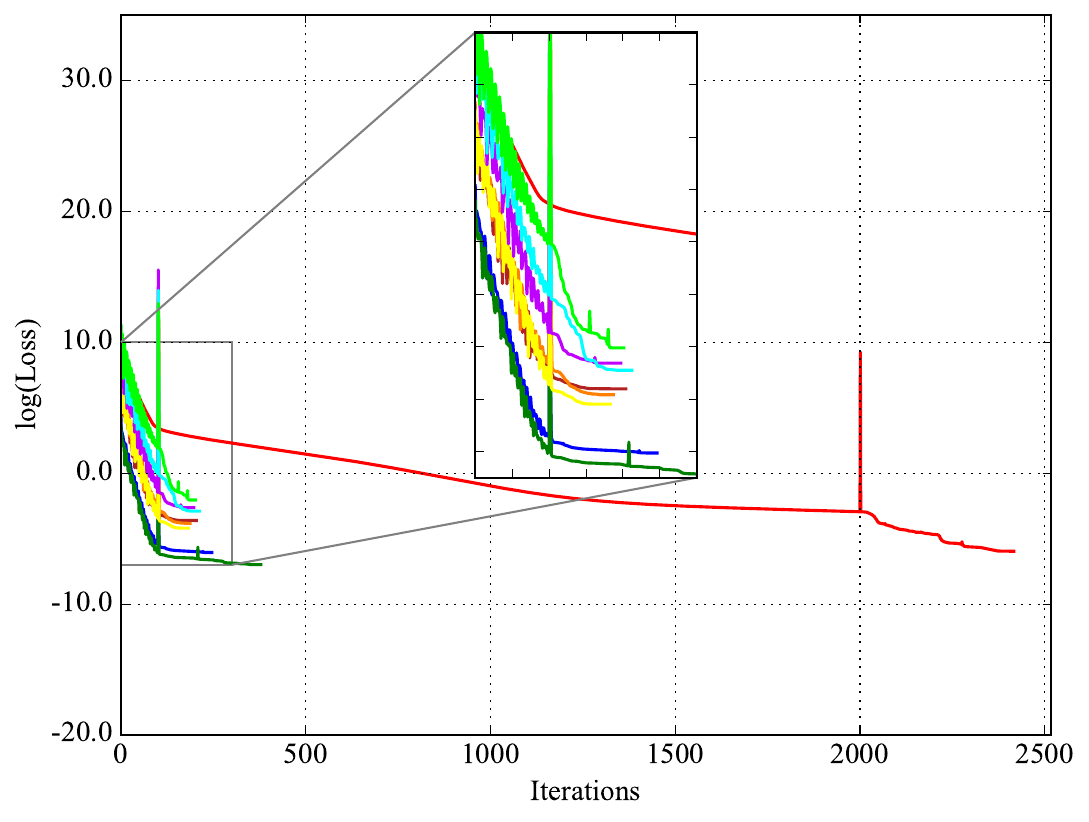}   
		\caption{}
	\end{subfigure}%
	\hspace{0.5cm}
	\begin{subfigure}[b]{6.0cm}
		\centering\includegraphics[height=6cm,width=7.0cm]{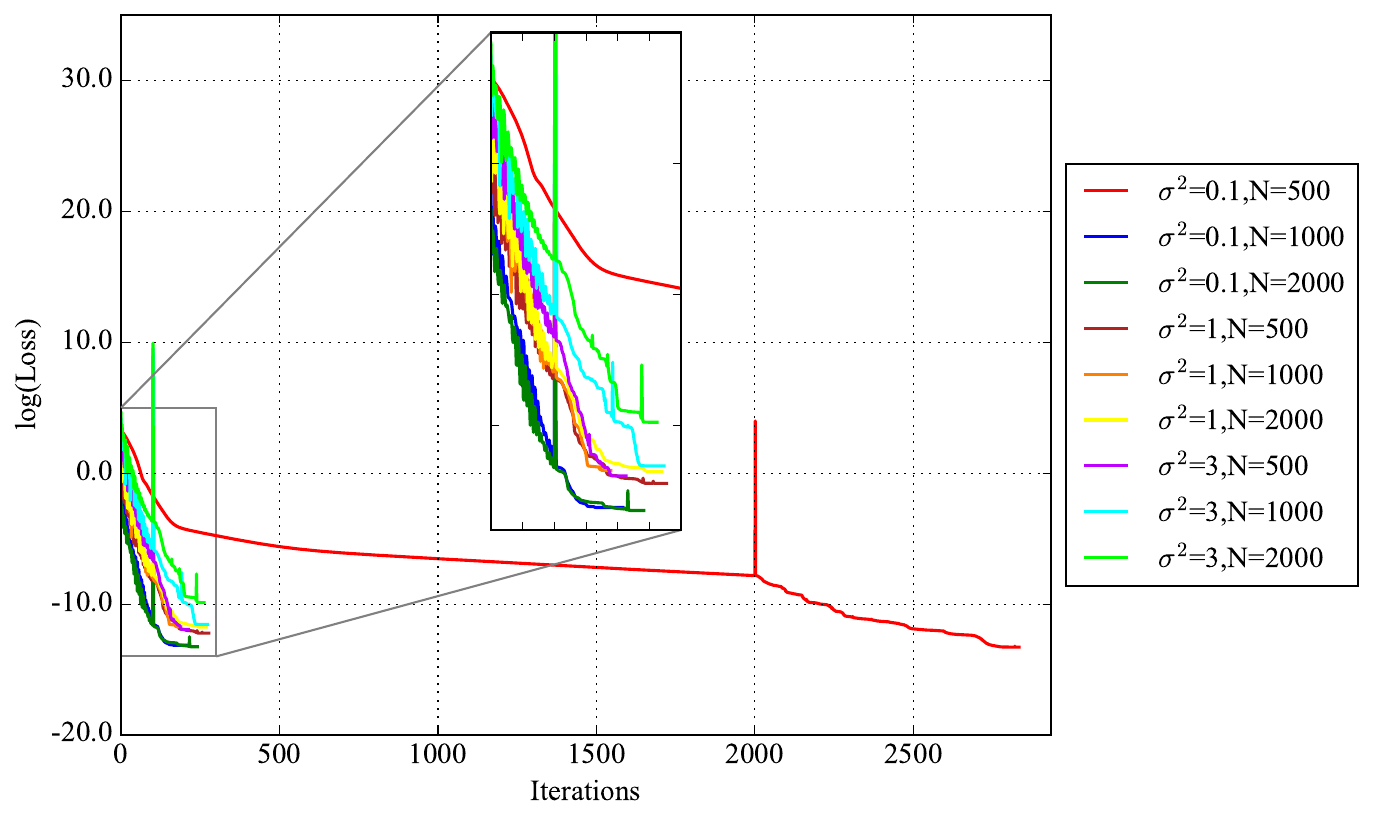}   
		\caption{}
	\end{subfigure}% 
	\caption{One dimensional logarithm loss function with $\left(a\right)$ exponential correlation and $\left(b\right)$ Gaussian correlation}
	\label{fig:loss1Dtest}
\end{figure}

Additionally, shown in Figure~\ref{fig:loss1Dtest} is the $log(Loss)$ vs. iteration graph for different parameters in constructing the random log-hydraulic conductivity ﬁeld. 
It can be observed from the $log(Loss)$ vs. iteration graph the following two crucial aspects: 1).the loss value for the Gaussian correlation is much smaller than the exponential correlation for all  $\sigma^2$ and $N$ values; 2).with transfer learning, the loss function drops significantly faster and the number of required iterations is greatly reduced, which will eventually obtain more accurate results with much less calculation time. Judging from above mentioned performance, the Gaussian correlation outweigh the exponential correlation in generating random log-hydraulic conductivity ﬁeld for the physics-informed machine learning. 

\subsubsection{Two dimensional groundwater flow with both correlations}
\label{subsection 4.1.2:2D comparison of gauss and exp}
To solve the non-homogeneous 2D flow problem for Darcy equation, the manufactured solution in Equation \ref{eq:u2} is adopted for model verification. The hydraulic conductivity $K$ is constructed from Equation \eqref{eq:u} by Radom spectral method as Equation \eqref{eq:k2}. The source term $f$ is taken the form in Equation \eqref{eq:f22}. Also, for detailed derivation, it is added in \ref{appendix b}. Also, the exponential and Gaussian correlations for the heterogeneous hydraulic conductivity are tested with varying $\sigma^2$ and $N$ values. The same conclusion can be drawn from Tables ~\ref{tab:Table5} and \ref{tab:Table6} that 1). with increasing $N$, the predict hydraulic head becomes more accurate, however when $\sigma^2$ becomes bigger, the accuracy deteriorates in most cases. Without doubt, the PINNs with Gaussian correlation based hydraulic conductivity produces much better results than the exponential correlations. With transfer learning technique, the accuracy further improves. The contour plots of the predicted hydraulic head and velocity and the manufactured solution for both exponential and Gaussian correlations with $\sigma^2=0.1$ and $N=2000$ are shown in Figures \ref{fig:2Dtestexp}, \ref{fig:2Dtestexpv}, \ref{fig:2Dtestgauss} and \ref{fig:2Dtestgaussv}. It is obvious the predicted physical patterns with proposed deep collocation method agrees well with the manufactured solution \eqref{eq:u1} in the 2D case. For the velocity, once again, the performance of the Gaussian correlation outweigh the exponential correlation.

\begin{table}[H] 
	\captionsetup{width=0.85\columnwidth}
	\caption{$\delta \hat{h}$ for 2D case computed with exponential correlation in different variance and number of modes}
	\vspace{-0.3cm}
	\centering 
	\resizebox{0.8\columnwidth}{!}{
		\begin{tabular}{l|c|c|c|c|c|c} 
			\toprule 
			\toprule 
			\multirow{2}*{\diagbox{$N$}{$\sigma^2$}}&\multicolumn{2}{c|}{0.1}&\multicolumn{2}{c|}{1}&\multicolumn{2}{c}{3}\\
			\cline{2-7}
			~ &without TL&with TL&without TL&with TL&without TL&with TL\\ 
			\midrule
			500&6.777e-2&9.345e-2&3.817e-2&4.635e-2&2.560e-1&5.5080e-2\\ 
			\midrule
			1000&1.479e-2&4.832e-2&1.790e-3&8.157e-2&9.739e-2&7.201e-2\\ 
			\midrule
			2000&7.147e-3&4.829e-2&4.471e-2&4.924e-2&9.357e-2&1.187e-1\\ 
			\bottomrule
		\end{tabular}
	}
	\label{tab:Table5}
\end{table}

\begin{table}[H]   
	\captionsetup{width=0.85\columnwidth}
	\caption{$\delta \hat{h}$ for 2D case computed with Gaussian correlation in different variance and number of modes}
	\vspace{-0.3cm}
	\centering
	\resizebox{0.8\columnwidth}{!}{
		\begin{tabular}{l|c|c|c|c|c|c} 
			\toprule 
			\toprule 
			\multirow{2}*{\diagbox{$N$}{$\sigma^2$}}&\multicolumn{2}{c|}{0.1}&\multicolumn{2}{c|}{1}&\multicolumn{2}{c}{3}\\
			\cline{2-7}
			~ &without TL&with TL&without TL&with TL&without TL&with TL\\  
			\midrule % In-table horizontal line
			500&9.974e-4&9.842e-4&3.530e-3&7.900e-4&3.053e-2&2.475e-3\\ 
			\midrule
			1000&2.980e-4&6.954e-4&6.270e-3&1.527e-3&3.855e-2&2.904e-3\\
			\midrule
			2000&7.299e-4&5.717e-4&7.719e-3&1.704e-3&7.486e-2&2.506-2\\ 
			\bottomrule
		\end{tabular}
	}
	\label{tab:Table6}
\end{table}

\begin{figure}[H]
	\captionsetup{width=0.85\columnwidth}
	\centering
	\begin{subfigure}[b]{6.0cm}
		\centering\includegraphics[height=6cm,width=6.0cm]{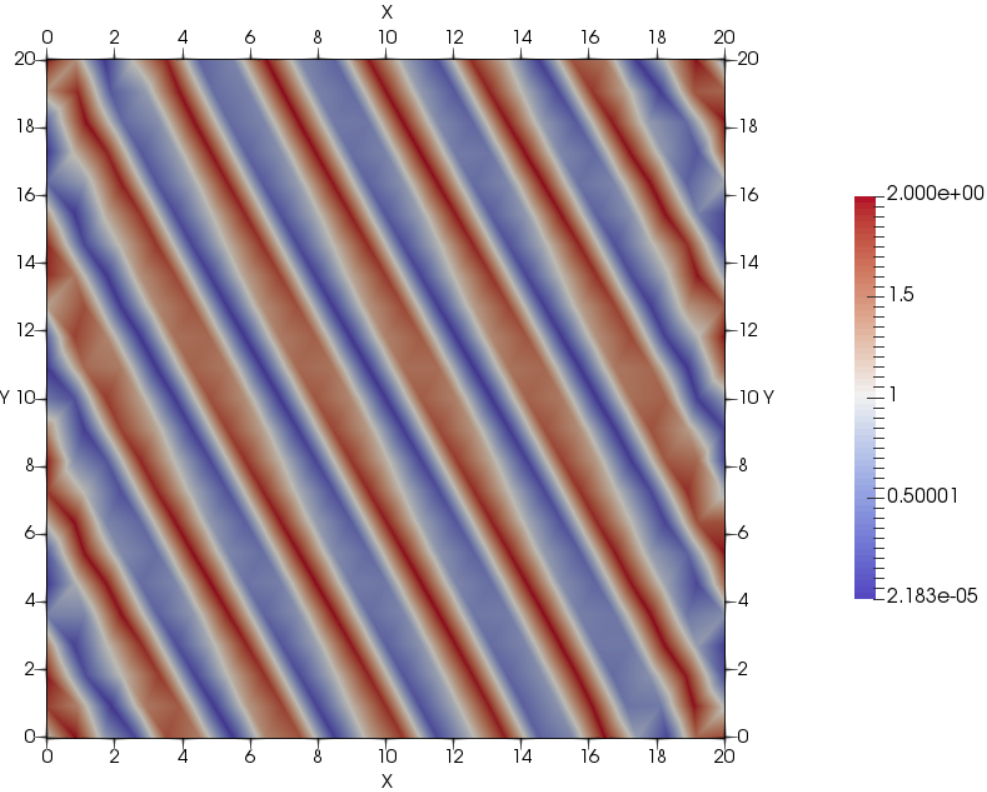}   
		\caption{}
	\end{subfigure}%
	\hspace{0.5cm}
	\begin{subfigure}[b]{6.0cm}
		\centering\includegraphics[height=6cm,width=6.0cm]{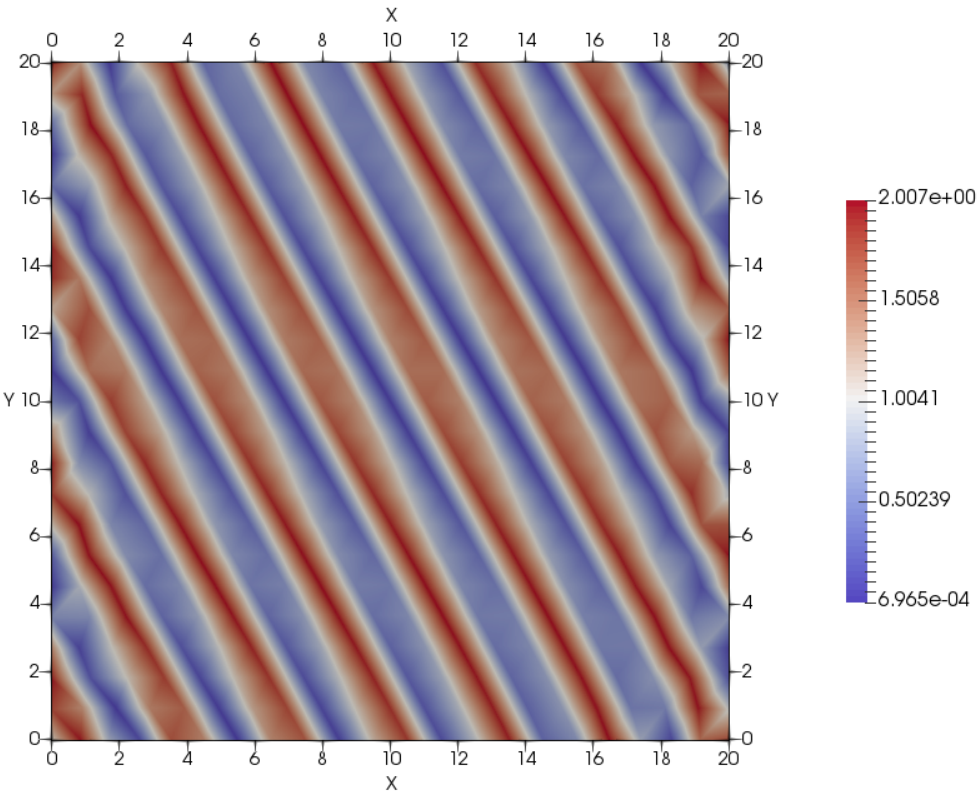}   
		\caption{}
	\end{subfigure}% 
	\caption{Two dimensional hydraulic head when $\sigma^2=0.1$, $N=2000$ with exponential correlation $\left(a\right)$ exact solution and $\left(b\right)$ predict solution}
	\label{fig:2Dtestexp}
\end{figure}

\begin{figure}[H]
	\captionsetup{width=0.85\columnwidth}
	\centering
	\begin{subfigure}[b]{6.0cm}
		\centering\includegraphics[height=6cm,width=6.0cm]{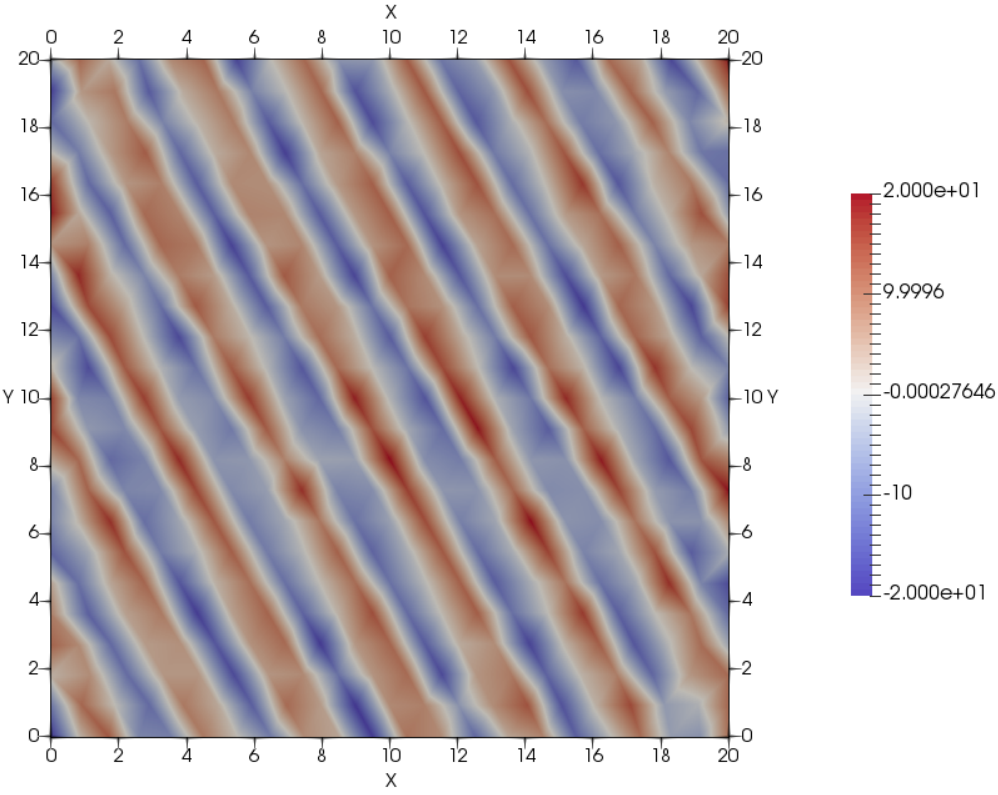}   
		\caption{}
	\end{subfigure}%
	\hspace{0.5cm}
	\begin{subfigure}[b]{6.0cm}
		\centering\includegraphics[height=6cm,width=6.0cm]{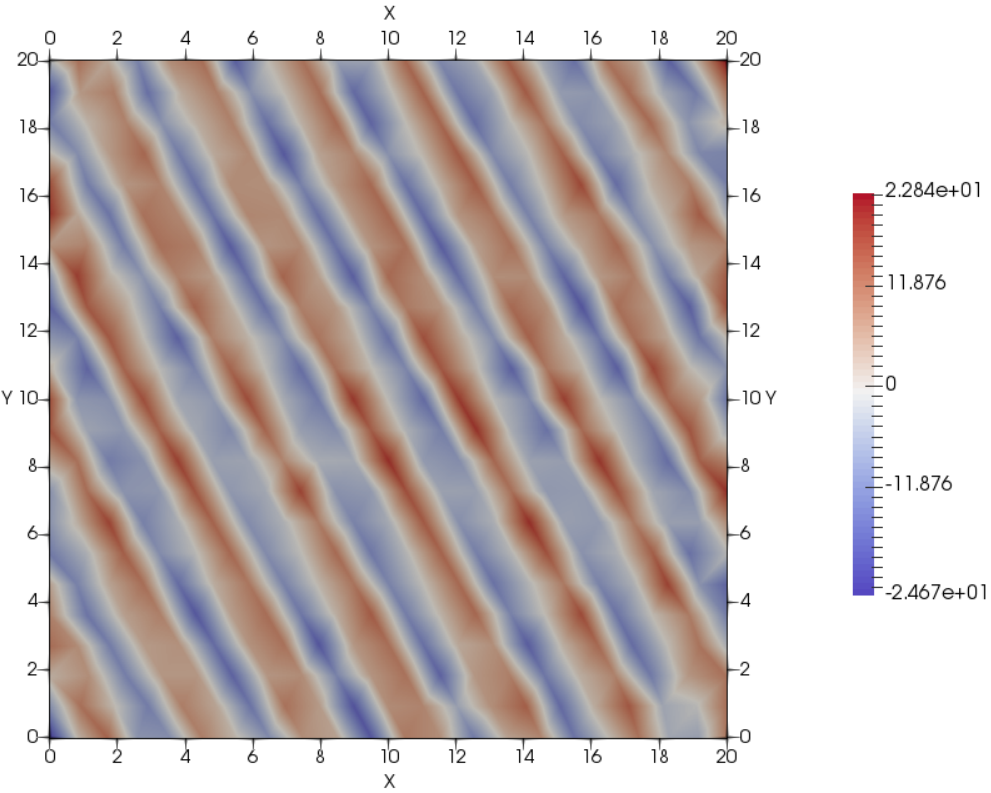}   
		\caption{}
	\end{subfigure}% 
	\caption{Two dimensional velocity when $\sigma^2=0.1$, $N=2000$ with exponential correlation $\left(a\right)$ exact solution and $\left(b\right)$ predict solution}
	\label{fig:2Dtestexpv}
\end{figure}

\begin{figure}[H]
	\captionsetup{width=0.85\columnwidth}
	\centering
	\begin{subfigure}[b]{6.0cm}
		\centering\includegraphics[height=6cm,width=6.0cm]{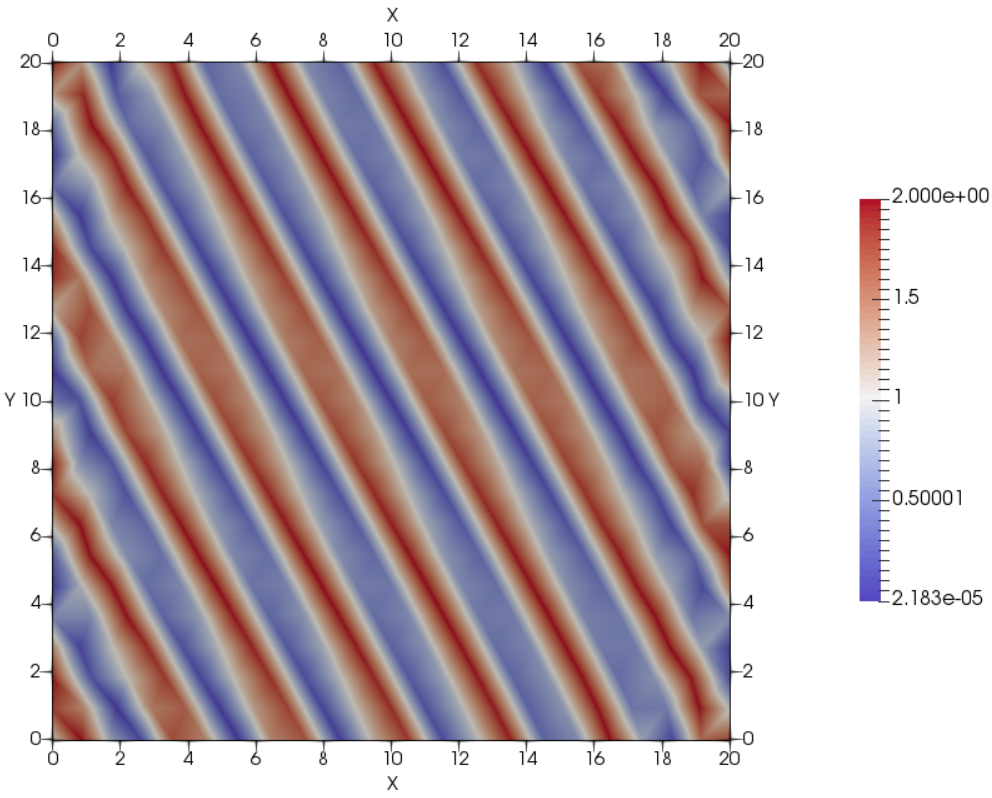}   
		\caption{}
	\end{subfigure}%
	\hspace{0.5cm}
	\begin{subfigure}[b]{6.0cm}
		\centering\includegraphics[height=6cm,width=6.0cm]{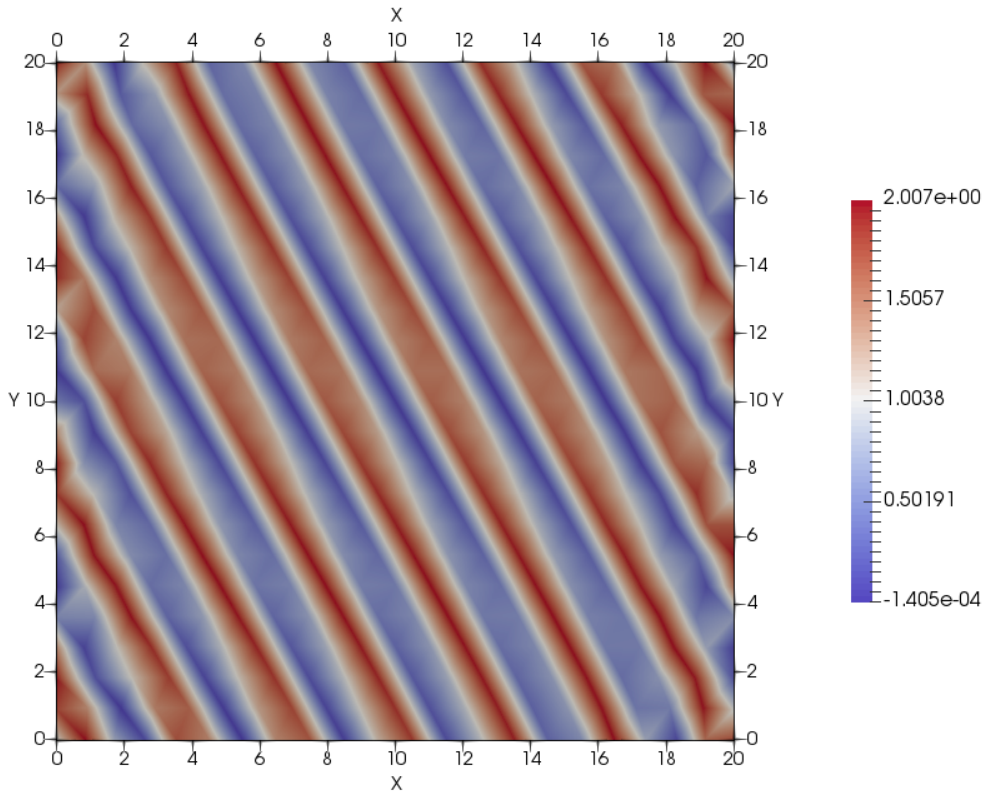}   
		\caption{}
	\end{subfigure}% 
	\caption{Two dimensional hydraulic head when $\sigma^2=0.1$, $N=2000$ with Gaussian correlation $\left(a\right)$ exact solution and $\left(b\right)$ predict solution}
	\label{fig:2Dtestgauss}
\end{figure}

\begin{figure}[H]
	\captionsetup{width=0.85\columnwidth}
	\centering
	\begin{subfigure}[b]{6.0cm}
		\centering\includegraphics[height=6cm,width=6.0cm]{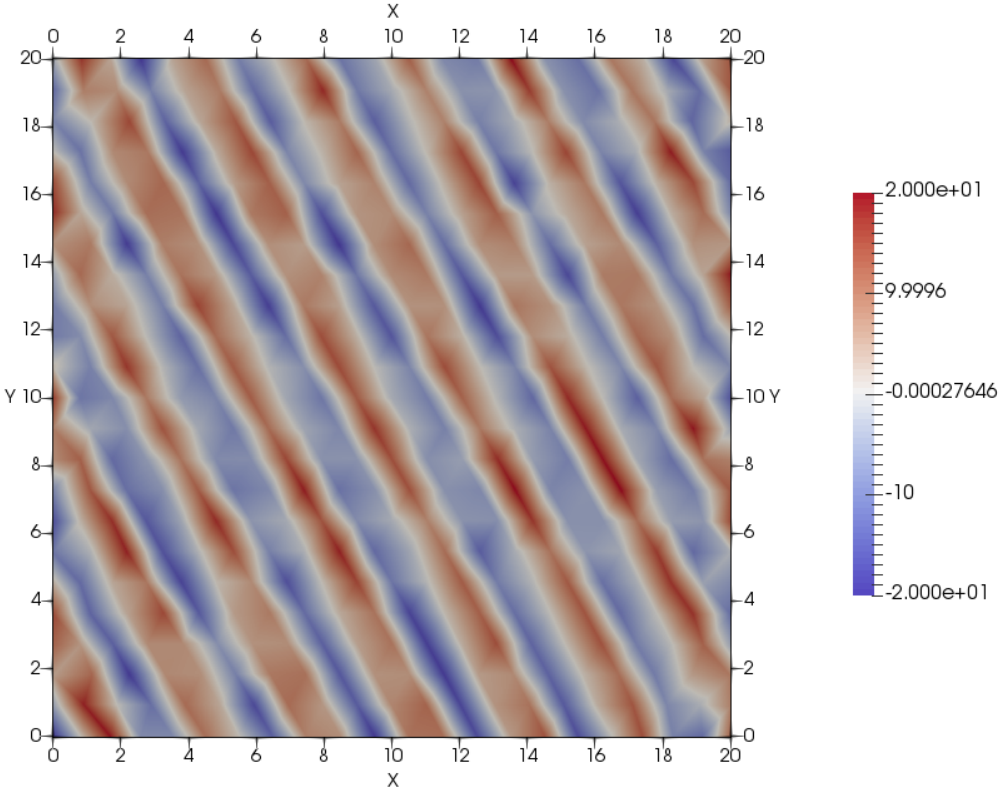}   
		\caption{}
	\end{subfigure}%
	\hspace{0.5cm}
	\begin{subfigure}[b]{6.0cm}
		\centering\includegraphics[height=6cm,width=6.0cm]{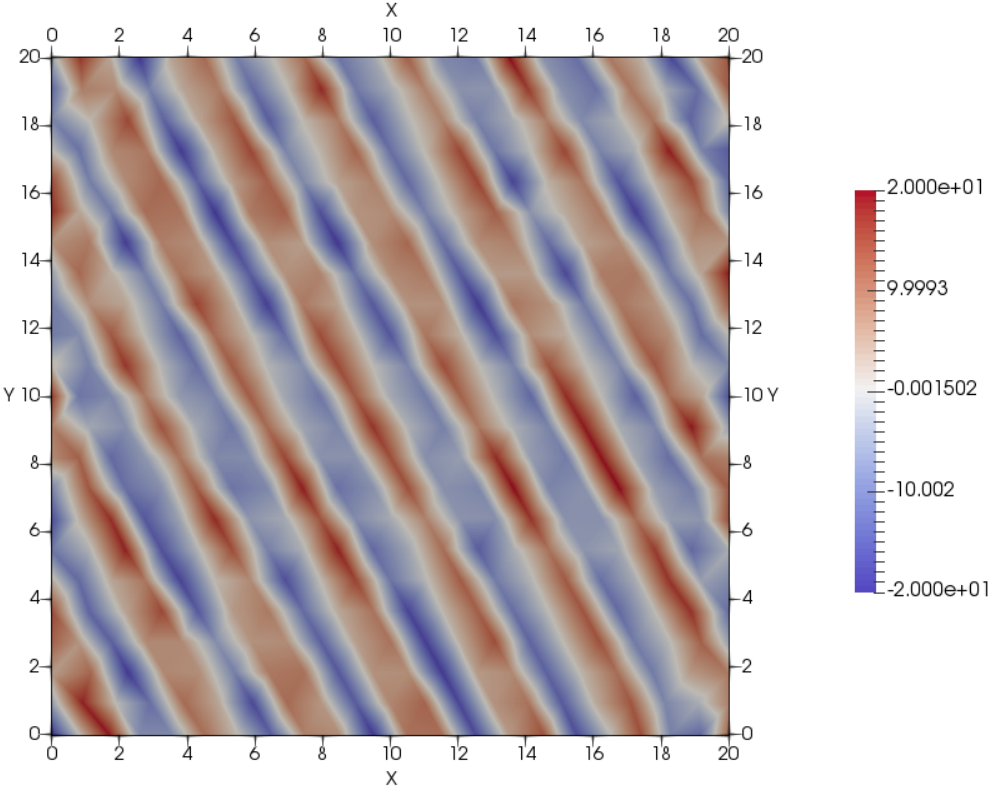}   
		\caption{}
	\end{subfigure}% 
	\caption{Two dimensional velocity when $\sigma^2=0.1$, $N=2000$ with Gaussian correlation $\left(a\right)$ exact solution and $\left(b\right)$ predict solution}
	\label{fig:2Dtestgaussv}
\end{figure}

Further, the $log(Loss)$ vs. iteration graph for different parameters in constructing the random log-hydraulic conductivity ﬁeld is demonstrated in Figure~\ref{fig:loss2Dtest}. It becomes more conspicuous that the loss for PINNs with Gaussian correlations is much smaller and decreases faster. For exponential correlations, although the iteration close fast in the L-BFGS stage, the loss is not fully minimized, this is largely improved for PINNs with Gaussian correlations. The effects of transfer learning is well shown in the convergence graph that the loss function plummets with less iterations, which largely reduces the training time. Thus for two dimensional groundwater flow, the Gaussian correlation shows much better performance than the exponential correlation for the physics-informed machine learning. The transfer learning technique can be conducive in boosting the whole model.
\begin{figure}[H]
	\captionsetup{width=0.85\columnwidth}
	\centering
	\begin{subfigure}[b]{6.0cm}
		\centering\includegraphics[height=6cm,width=6.0cm]{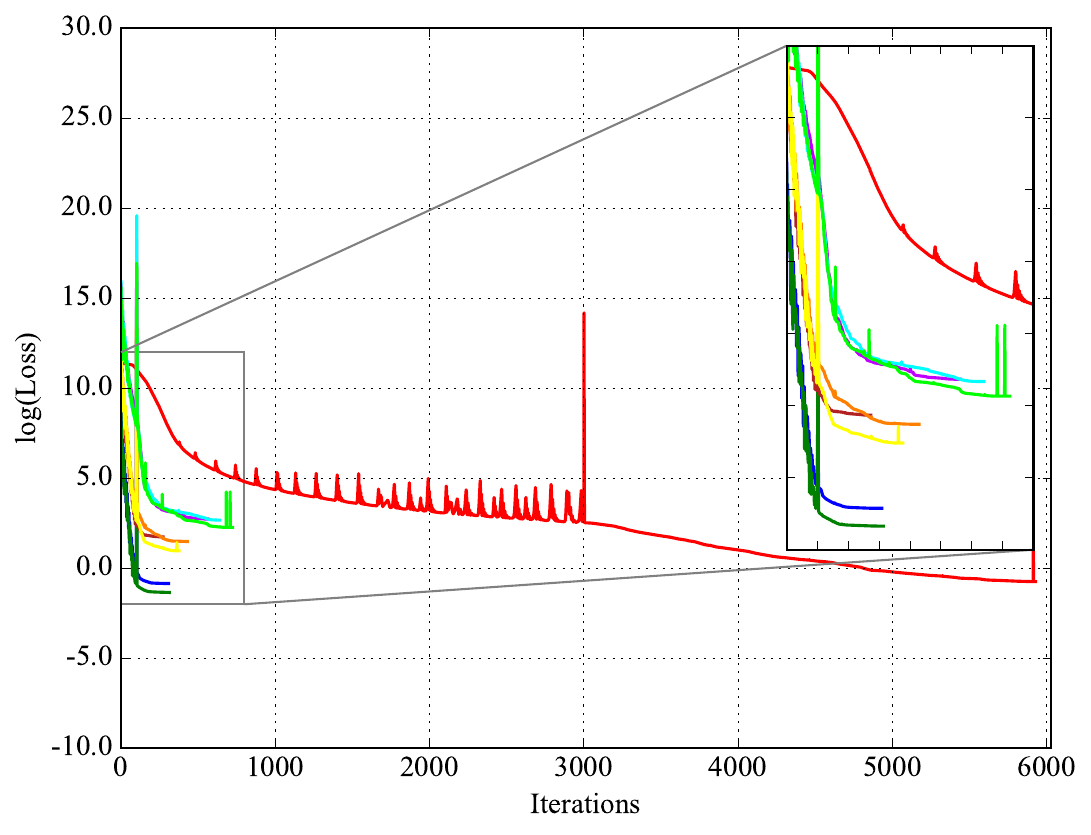}   
		\caption{}
	\end{subfigure}%
	\hspace{0.5cm}
	\begin{subfigure}[b]{6.0cm}
		\centering\includegraphics[height=6cm,width=8.0cm]{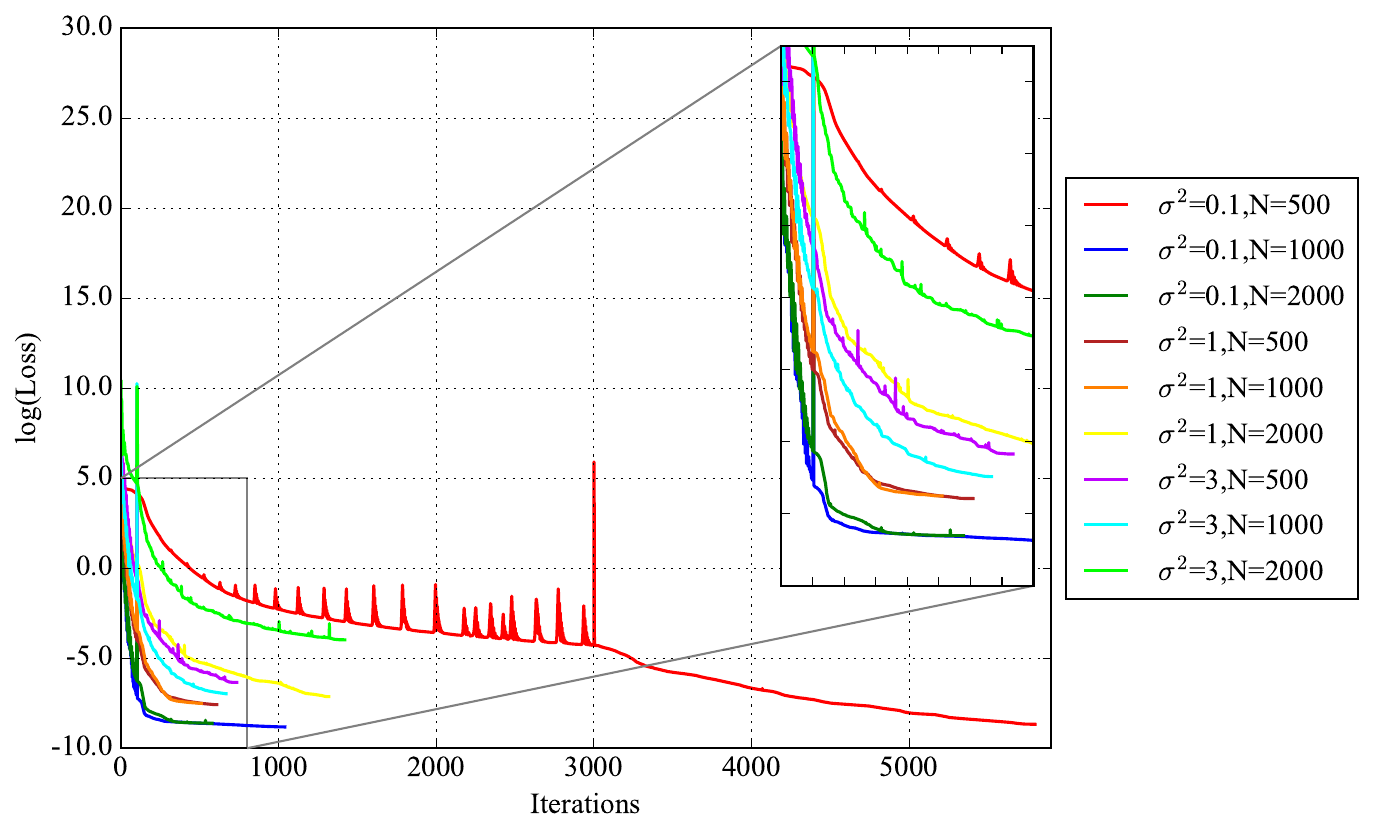}   
		\caption{}
	\end{subfigure}% 
	\caption{Two dimensional logarithm loss function with $\left(a\right)$ exponential correlation and $\left(b\right)$ Gaussian correlation}
	\label{fig:loss2Dtest}
\end{figure}

\subsubsection{Three dimensional groundwater flow with both correlations}
\label{subsection 4.1.3:3D comparison of gauss and exp}
The non-homogeneous 3D flow problem for Darcy equation is further studied and verified, which complements the benchmark construction in \cite{alecsa2019benchmark}. The manufactured solution in Equation \ref{eq:u31} is adopted for model verification of three dimensional groundwater flow simulation. The hydraulic conductivity $K$ is constructed according to  Equation \eqref{eq:k3}. The source term $f$ is employing the form in Equation \eqref{eq:f31}.  The exponential and Gaussian correlations for the heterogeneous hydraulic conductivity are tested with varying $\sigma^2$ and $N$ values. Tables ~\ref{tab:Table7} and \ref{tab:Table8} listed the relative error of hydraulic head for DCM with transfer learning or without transfer learning. It can be observed that with for different $\sigma^2$ and $N$ values, the performance of the PINNs with both correlations varies largely. In a whole, the Gaussian correlation yields much better numerical results. With transfer learning model, the predicted results is much more accurate for most cases. This is really promising, since transfer learning not only reduces the iteration number, which will result in less computational time, but improve the accuracy of the model for different $\sigma^2$ and $N$ values and for both correlation functions. The hydraulic head predicted by both correlation functions with $\sigma^2=0.1$ and $N=2000$ are then shown in Figures \ref{fig:3Dtestexp}, \ref{fig:3Dtestexpv}, \ref{fig:3Dtestgauss} and \ref{fig:3Dtestgaussv}. It can be seen that both correlation functions predicts results agreeable with the exact solutions. For the loss-vs.-iteration graph, it is obvious that both formulations with exponential and Gaussian correlation decrease to a stable value around $y=0$ axis, however, the transfer learning model Gaussian correlation function converges with less iteration numbers.

\begin{table}[H] 
	\captionsetup{width=0.85\columnwidth}
	\caption{$\delta \hat{h}$ for 3D case computed with exponential correlation in different variance and number of modes}
	\vspace{-0.3cm}
	\centering 
	\resizebox{0.8\columnwidth}{!}{
		\begin{tabular}{l|c|c|c|c|c|c} 
			\toprule 
			\toprule 
			\multirow{2}*{\diagbox{$N$}{$\sigma^2$}}&\multicolumn{2}{c|}{0.1}&\multicolumn{2}{c|}{1}&\multicolumn{2}{c}{3}\\
			\cline{2-7}
			~ &without TL&with TL&without TL&with TL&without TL&with TL\\ 
			\midrule
			500&3.419e-3&1.529e-2&1.131e0&5.885e-2&6.264e-1&7.340e-2\\ 
			\midrule
			1000&1.219e-1&8.333e-3&3.257e-1&5.668e-2&1.055e0&8.982e-2\\ 
			\midrule
			2000&5.667e-2&1.230e-2&4.287e-1&6.161e-2&1.204e0&5.313e-2\\ 
			\bottomrule
		\end{tabular}
	}
	\label{tab:Table7} 
\end{table}

\begin{table}[H]  
	\captionsetup{width=0.85\columnwidth}
	\caption{$\delta \hat{h}$ for 3D case computed with Gaussian correlation in different variance and number of modes} 
	\vspace{-0.3cm}
	\centering 
	\resizebox{0.8\columnwidth}{!}{
		\begin{tabular}{l|c|c|c|c|c|c}
			\toprule 
			\toprule 
			\multirow{2}*{\diagbox{$N$}{$\sigma^2$}}&\multicolumn{2}{c|}{0.1}&\multicolumn{2}{c|}{1}&\multicolumn{2}{c}{3}\\
			\cline{2-7}
			~ &without TL&with TL&without TL&with TL&without TL&with TL\\
			\midrule
			500&1.161e-2&5.439e-3&6.247e-3&1.540e-2&9.294e-2&1.078e-2\\ 
			\midrule
			1000&1.187e-3&5.620e-3&4.087e-2&1.205e-2&3.004e-1&1.996e-2\\ 
			\midrule
			2000&8.342e-3&1.661e-2&1.218e-2&1.278e-2&1.562e-1&1.952-2\\ 
			\bottomrule
		\end{tabular}
	}
	\label{tab:Table8}
\end{table}

\begin{figure}[H]
	\captionsetup{width=0.85\columnwidth}
	\centering
	\begin{subfigure}[b]{6.0cm}
		\centering\includegraphics[height=6cm,width=6.0cm]{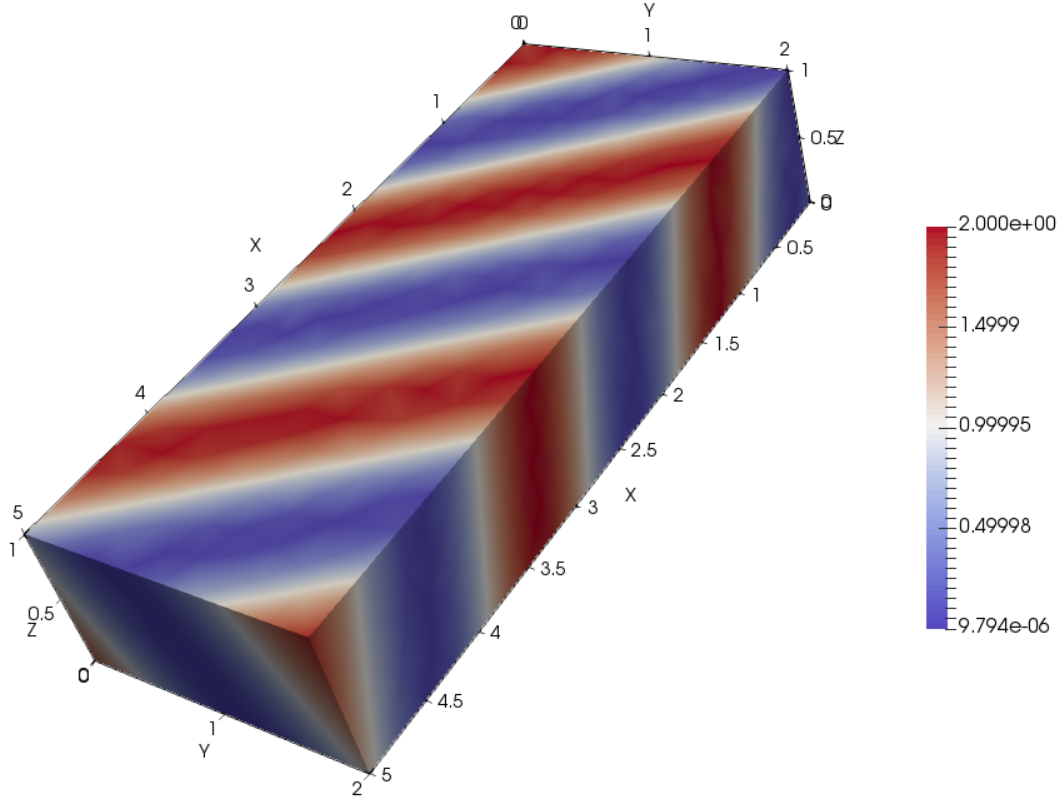}   
		\caption{}
	\end{subfigure}%
	\hspace{0.5cm}
	\begin{subfigure}[b]{6.0cm}
		\centering\includegraphics[height=6cm,width=6.0cm]{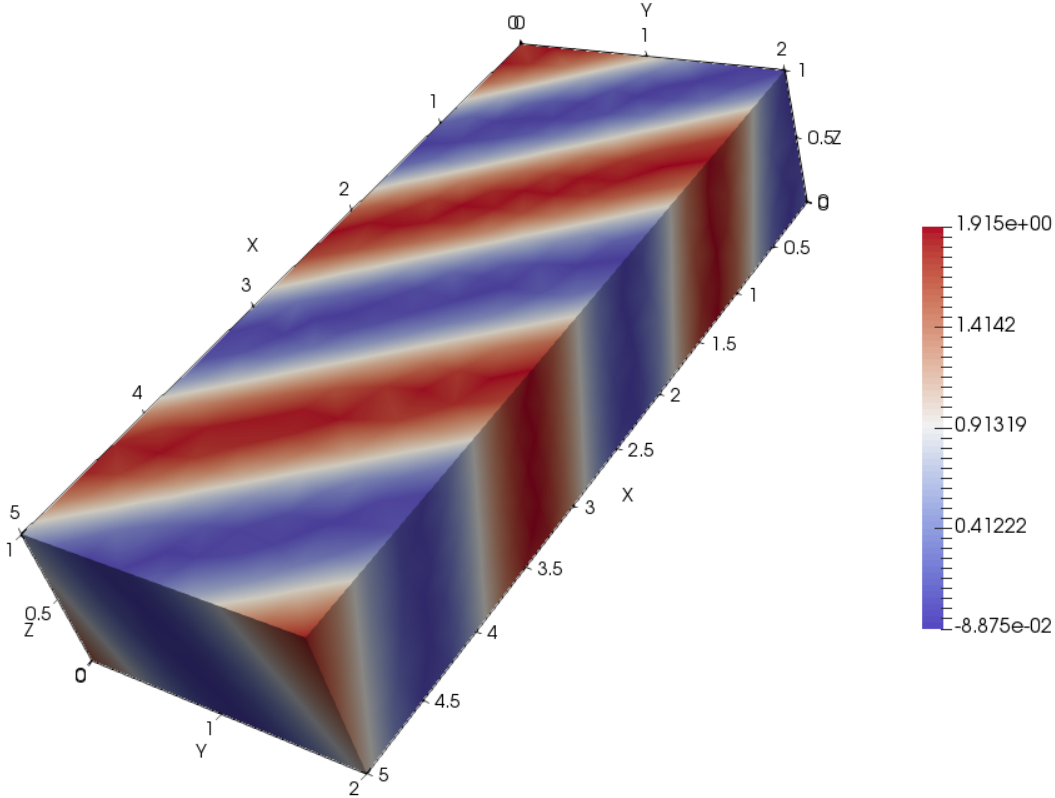}   
		\caption{}
	\end{subfigure}% 
	\caption{Three dimensional hydraulic head when $\sigma^2=0.1$, $N=2000$ with exponential correlation $\left(a\right)$ exact solution and $\left(b\right)$ predict solution}
	\label{fig:3Dtestexp}
\end{figure}

\begin{figure}[H]
	\captionsetup{width=0.85\columnwidth}
	\centering
	\begin{subfigure}[b]{6.0cm}
		\centering\includegraphics[height=6cm,width=6.0cm]{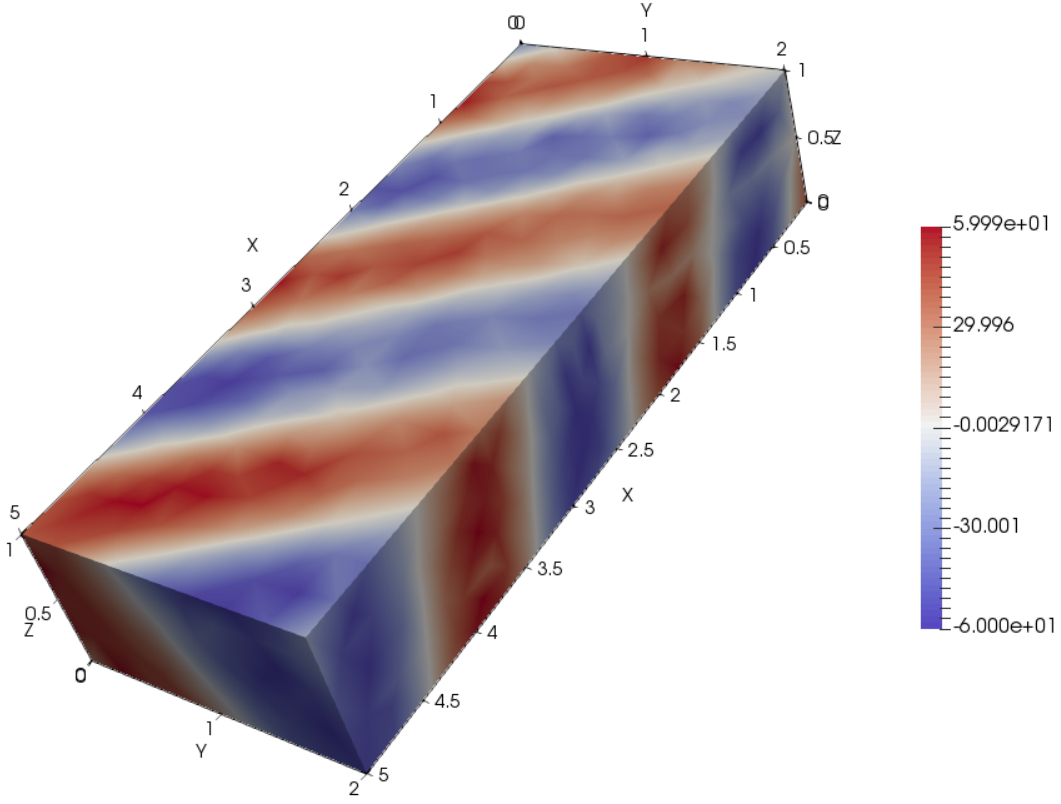}   
		\caption{}
	\end{subfigure}%
	\hspace{0.5cm}
	\begin{subfigure}[b]{6.0cm}
		\centering\includegraphics[height=6cm,width=6.0cm]{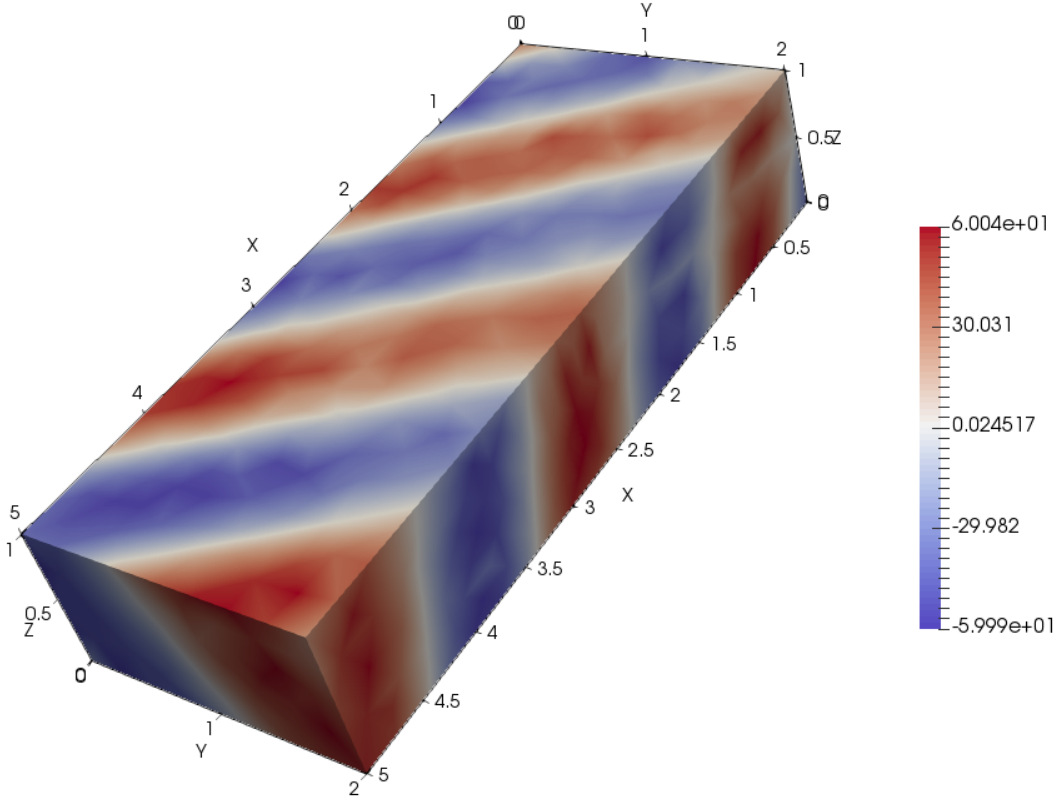}   
		\caption{}
	\end{subfigure}% 
	\caption{Three dimensional velocity when $\sigma^2=0.1$, $N=2000$ with exponential correlation $\left(a\right)$ exact solution and $\left(b\right)$ predict solution}
	\label{fig:3Dtestexpv}
\end{figure}

\begin{figure}[H]
	\captionsetup{width=0.85\columnwidth}
	\centering
	\begin{subfigure}[b]{6.0cm}
		\centering\includegraphics[height=6cm,width=6.0cm]{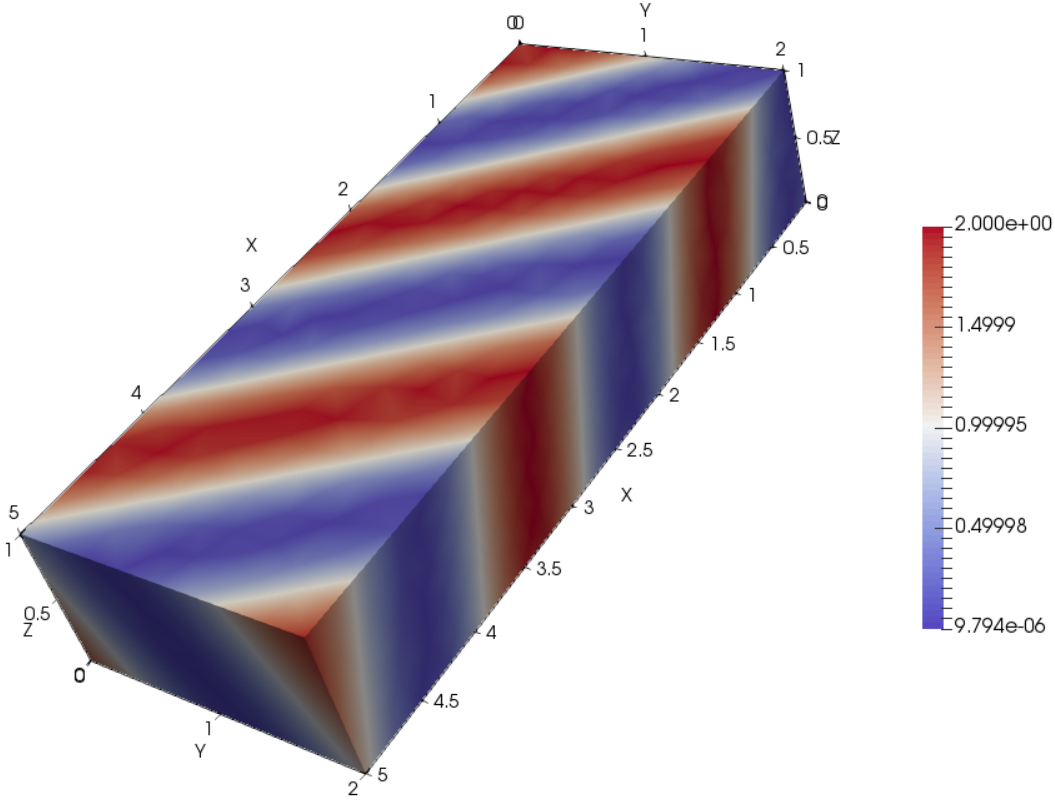}   
		\caption{}
	\end{subfigure}%
	\hspace{0.5cm}
	\begin{subfigure}[b]{6.0cm}
		\centering\includegraphics[height=6cm,width=6.0cm]{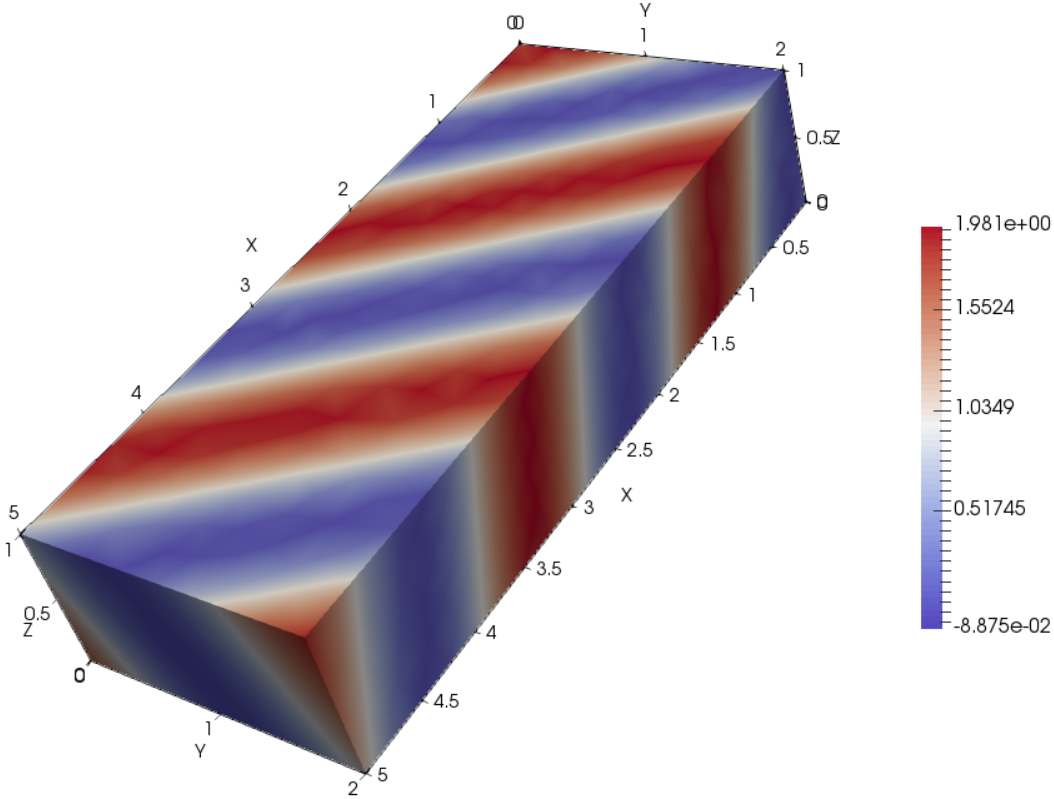}   
		\caption{}
	\end{subfigure}% 
	\caption{Three dimensional hydraulic head when $\sigma^2=0.1$, $N=2000$ with Gaussian correlation $\left(a\right)$ exact solution and $\left(b\right)$ predict solution}
	\label{fig:3Dtestgauss}
\end{figure}

\begin{figure}[H]
	\captionsetup{width=0.85\columnwidth}
	\centering
	\begin{subfigure}[b]{6.0cm}
		\centering\includegraphics[height=6cm,width=6.0cm]{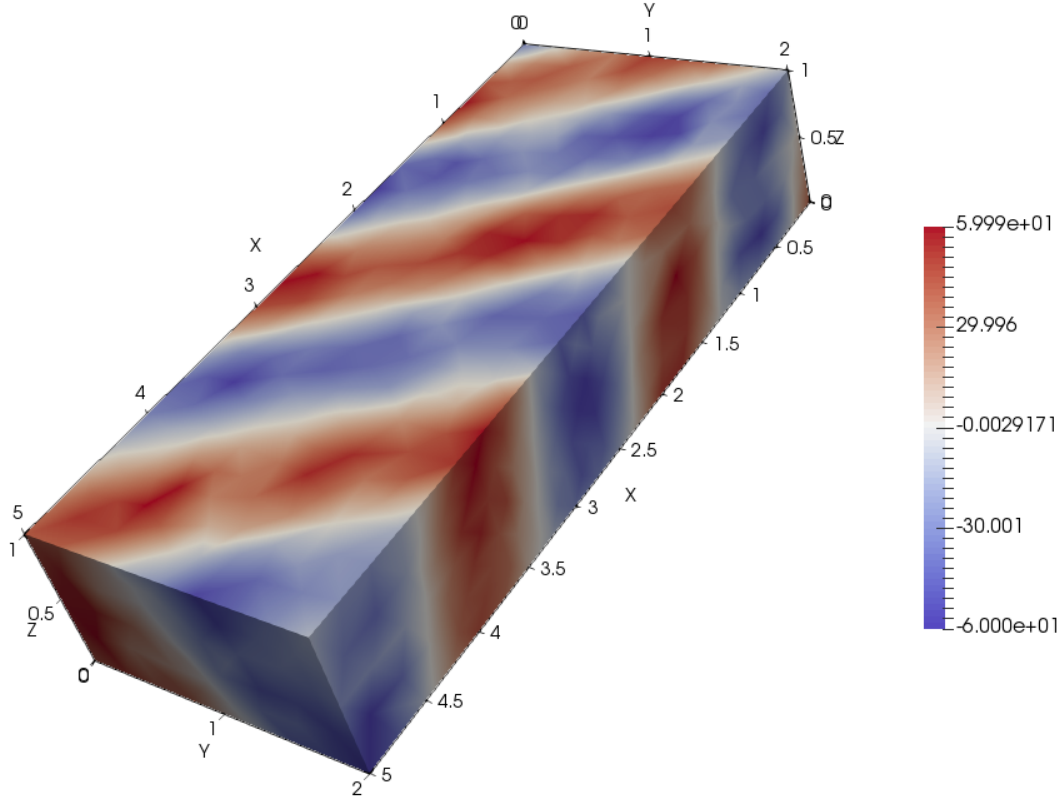}   
		\caption{}
	\end{subfigure}%
	\hspace{0.5cm}
	\begin{subfigure}[b]{6.0cm}
		\centering\includegraphics[height=6cm,width=6.0cm]{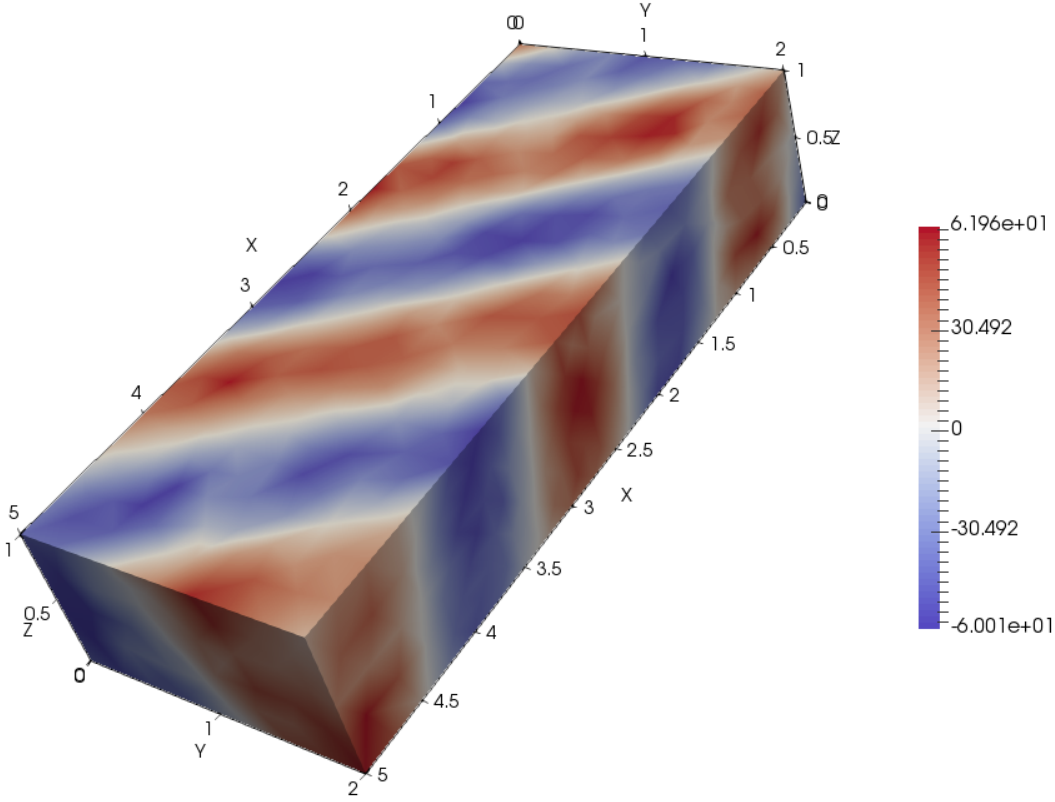}   
		\caption{}
	\end{subfigure}% 
	\caption{Three dimensional velocity when $\sigma^2=0.1$, $N=2000$ with Gaussian correlation $\left(a\right)$ exact solution and $\left(b\right)$ predict solution}
	\label{fig:3Dtestgaussv}
\end{figure}

\begin{figure}[H]
	\captionsetup{width=0.85\columnwidth}
	\centering
	\begin{subfigure}[b]{6.0cm}
		\centering\includegraphics[height=6cm,width=6.0cm]{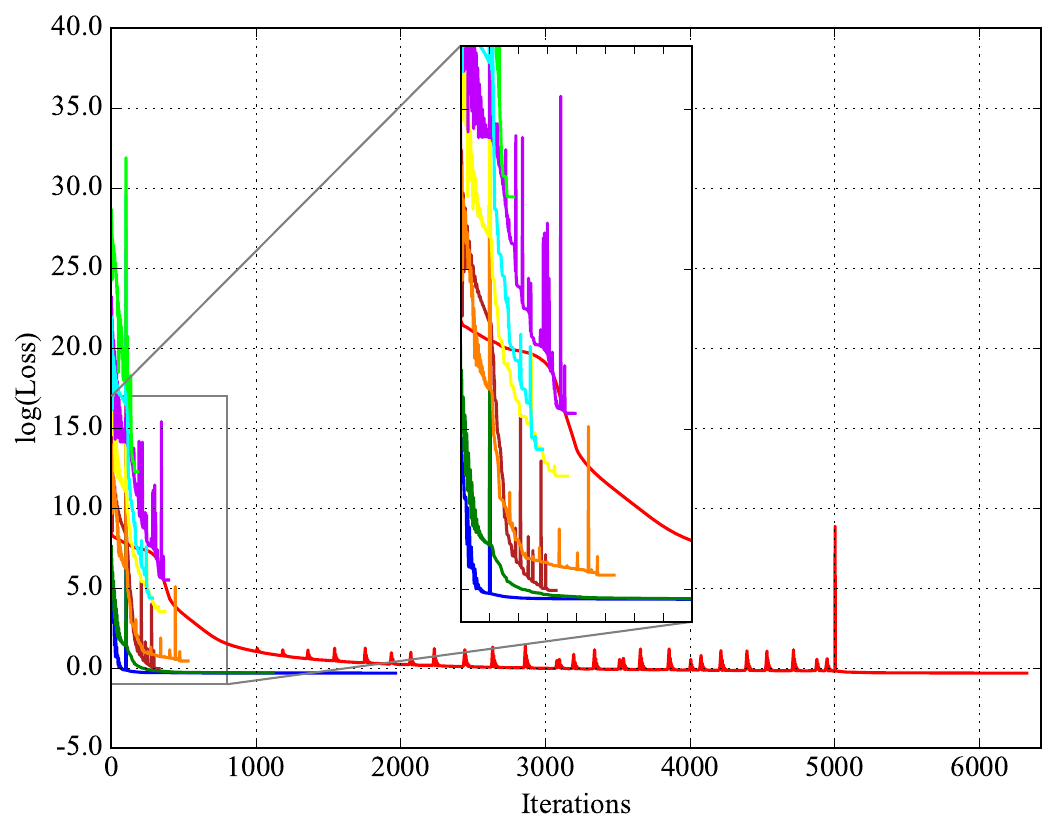}   
		\caption{}
	\end{subfigure}%
	\hspace{0.5cm}
	\begin{subfigure}[b]{6.0cm}
		\centering\includegraphics[height=6cm,width=8.0cm]{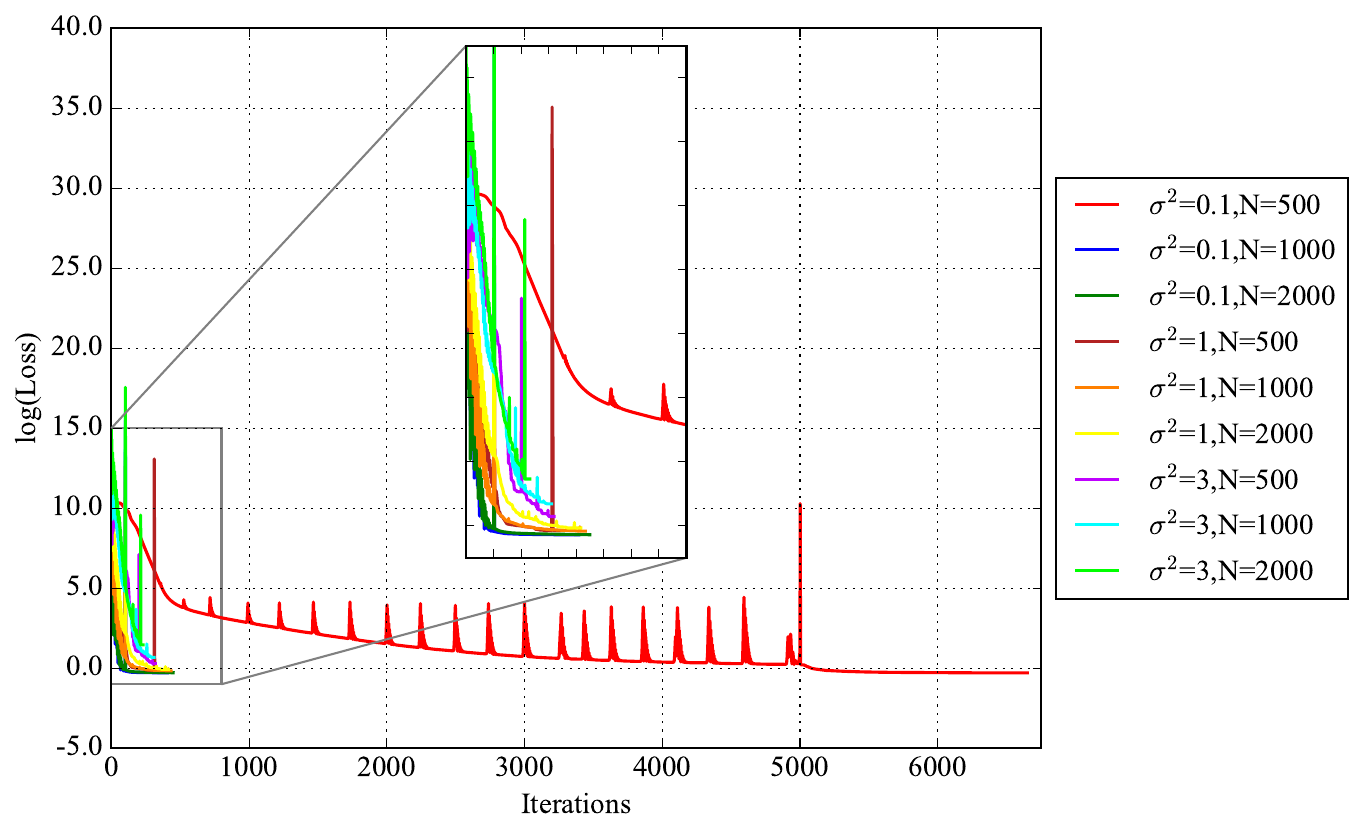}   
		\caption{}
	\end{subfigure}% 
	\caption{Three dimensional logarithm loss function with $\left(a\right)$ exponential correlation and $\left(b\right)$ Gaussian correlation}
	\label{fig:loss3Dtest}
\end{figure}

Above all, the calculation training times involving for DCM with both correlation functions are shown in Table \ref{tab:Table9}, it is obvious that the Gaussian correlation function occupies less calculation time. When equipped with transfer learning model, the training time decreases exceedingly, which is favorably considering the loss with physics-informed constraints is usually difficult to train.
\begin{table}[H] 
	\captionsetup{width=0.85\columnwidth}
	\caption{Calculation time required in different dimensions} 
	\vspace{-0.3cm}
	\centering 
	\resizebox{0.9\columnwidth}{!}{
		\begin{tabular}{l|c|c|c|c|c|c}
			\toprule 
			\toprule 
			\multirow{2}*{\diagbox{Correlation}{Dimension}}&\multicolumn{2}{c|}{1}&\multicolumn{2}{|c|}{2}&\multicolumn{2}{|c}{3}\\ 
			\cline{2-7}
			~ &without TL&with TL&without TL&with TL&without TL&with TL\\
			\midrule
			Exponential&30s&3.0s&97s&6.5s&58s&7.8s\\ 
			\midrule
			Gaussian&28s&3.0s&58s&9.8s&52s&5.9s\\
			\bottomrule
		\end{tabular}
	}
	\label{tab:Table9} % A label for referencing this table elsewhere, references are used in text as \ref{label}
\end{table}

In summary, the comparison reveals that the loss function in the Gaussian correlation tends to decrease faster than the exponential, and that the error under Gaussian is much smaller and more stable than the exponential. Regarding the training time, in all three dimensions, using Gaussian correlations requires less computation time than exponential.

It is worth noting that in the three-dimensional case, the loss function of the exponential correlation coefficient explodes in a gradient explosion when the number of collocation points is greater than 150, while in the Gaussian correlation coefficient this does not occur even if the number of collocation points is equal to 1000.

From this we conclude that the PINNs model fits better with Gaussian correlations, and therefore we will use it in the next experiments.

\subsection{Sensitivity analysis results}
\label{subsection 4.2:SARESULTS}
The purpose of sensitivity analysis is to help us eliminate irrelevant variables, especially when there are a lot of parameter variables, and to help us reduce the time it takes to run the hyperparameter optimizer. It is often used in the analysis of practical problems because of the complexity of the actual situation. The hyper-parameters in this flow problem are assumed to be these five:

\begin{table}[H] 
	\captionsetup{width=0.85\columnwidth}
	\caption{Hyper-parameters and their intervals in groundwater flow problem}
	\vspace{-0.3cm}
	\centering 
	\resizebox{0.65\columnwidth}{!}{
		\begin{tabular}{c| c}
			\toprule
			\toprule 
			Hyper-parameters & Intervals\\
			\midrule
			Layers of NNs & $[2,30]$\\ 
			\midrule
			Neurons per layer&$[10,50]$\\ 
			\midrule
			Number of iterations&$[1500,3000]$\\
			\midrule
			Number of collation points&$[800,2000]$\\
			\midrule
			Maximum line search of L-BFGS algorithm&$[30,300]$\\
			\bottomrule 
		\end{tabular}
	}
	\label{tab:Table10} 
\end{table}

Substituting these parameters in Table~\ref{tab:Table7} into our model for the sensitivity analysis, we obtain the following results:

\begin{figure}[H]
	\captionsetup{width=0.9\columnwidth}
	\includegraphics[height=9cm]{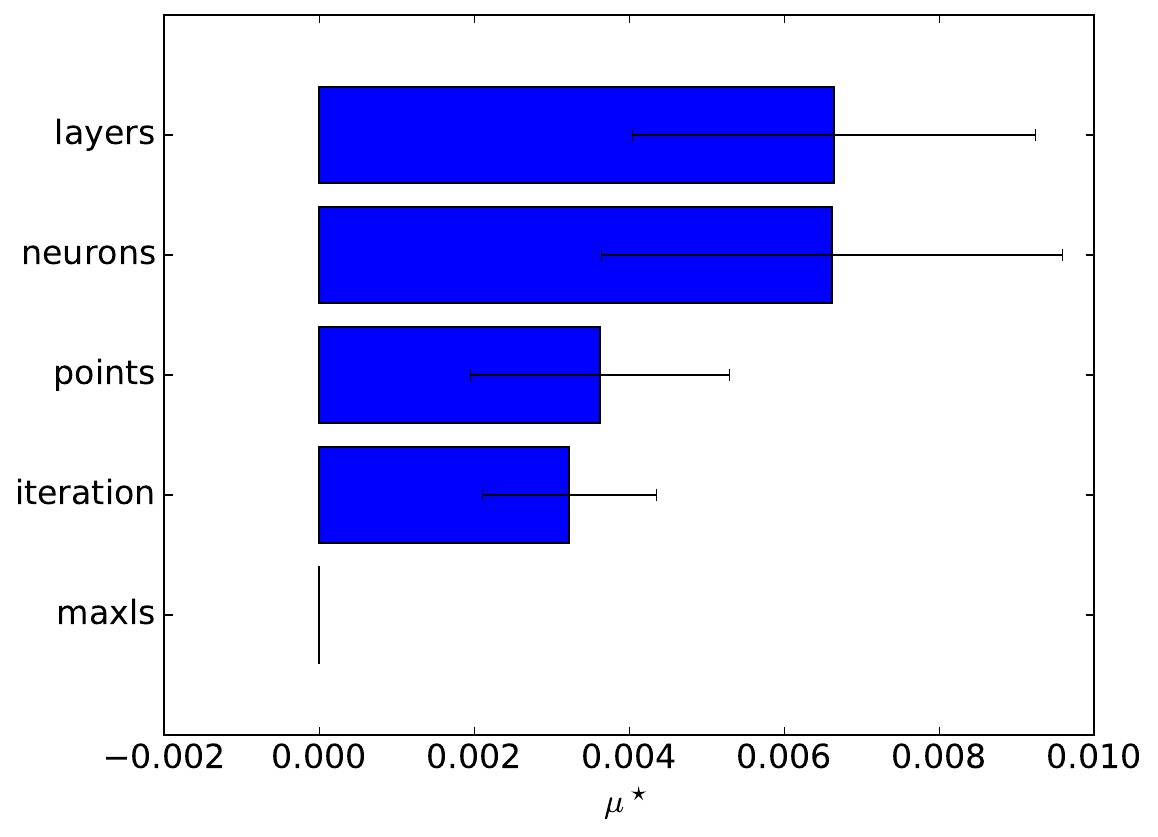}
	\centering
	\caption{Sensitivity histogram of Morris}
	\label{fig:morris1}
\end{figure}

\begin{figure}[H]
	\captionsetup{width=0.9\columnwidth}
	\includegraphics[height=9cm]{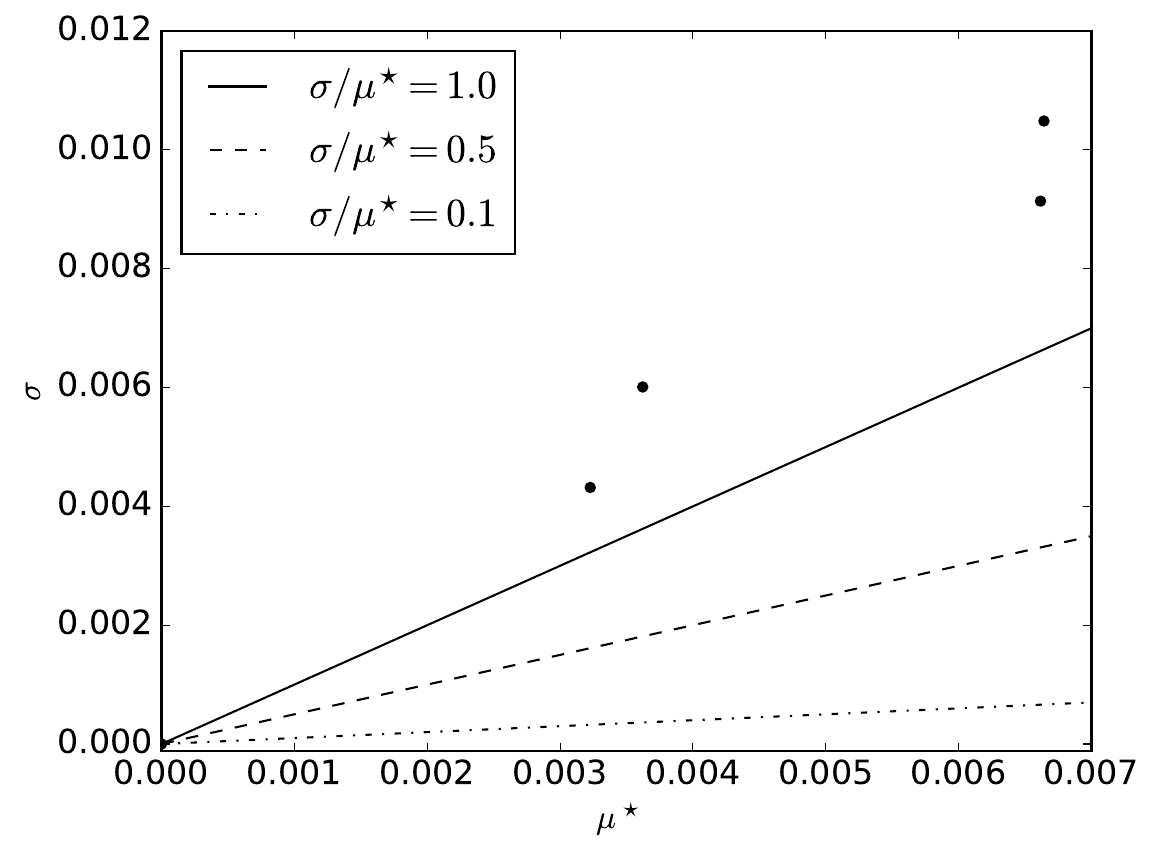}
	\centering
	\caption{Morris $\mu^*$ and $\sigma$ computed using the Morris screening algorithm}
	\label{fig:morris2}
\end{figure}

\begin{figure}[H]
	\captionsetup{width=0.9\columnwidth}
	\includegraphics[height=5cm]{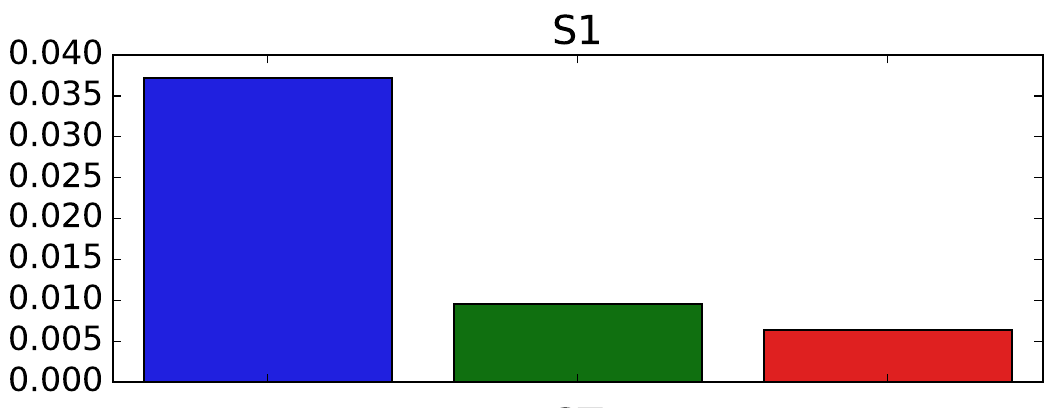}
	\centering
	\caption{Sensitivity histogram of eFAST}
	\label{fig:eFAST}
\end{figure}

From Figure~\ref{fig:morris1} and Figure~\ref{fig:morris2} we can easily conclude that the number of layers have the greatest impact on system sensitivity, and the neurons is almost as large as it, in contrast, the maximum line search of L-BFGS has almost no effect on it. So we remove the number of layers and the maxls, and continue to calculate the sensitivity of the remaining three in the eFast model, the results is shown in Figure~\ref{fig:eFAST}. We can draw that, neurons is the second parameter that we should emphasize on. As a result, the layers and the neurons are chosen as the hyper-parameters in search space to be selected by the automated machine learning program.

\subsection{Hyperparameter optimizations method comparison}
\label{subsection 4.3:hpo}
\quad In order to select the hyperparameter optimization algorithm that is more suitable for our model, namely with small fitness value and fast computational speed, we use the two hyperparameters selected from the sensitivity analysis~\ref{subsection 4.2:SARESULTS} as search variables in search space and compute the four algorithms presented in the previous chapter, all remaining conditions setting equal. The horizontal coordinates in the Figure \ref{fig:HyperOpt} below represent the number of neurons per layer and the vertical coordinates represent the number of hidden layers, and the points circled by the box are the points where the optimal configuration is located, corresponding to the number of optimal layers and the number of neurons. The colormap is setting to be the value of the performance estimation strategy defined in Equation \ref{eq:Estimation Strategy}.

\begin{figure}[H]
	\captionsetup{width=0.85\columnwidth}
	\centering
	\begin{subfigure}[b]{7.0cm}
		\centering\includegraphics[height=6cm,width=7.5cm]{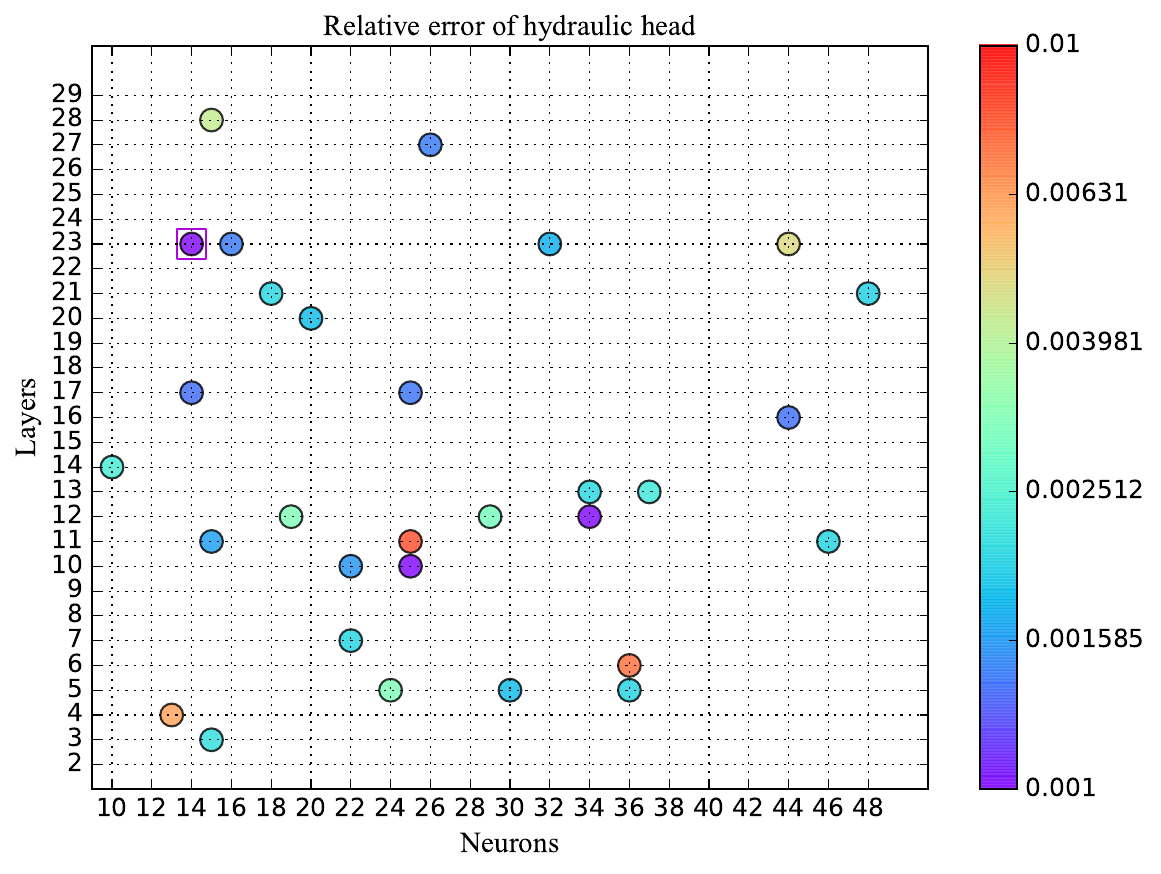}   
		\caption{}
	\end{subfigure}%
	\hspace{0.1cm}
	\begin{subfigure}[b]{7.0cm}
		\centering\includegraphics[height=6cm,width=7.5cm]{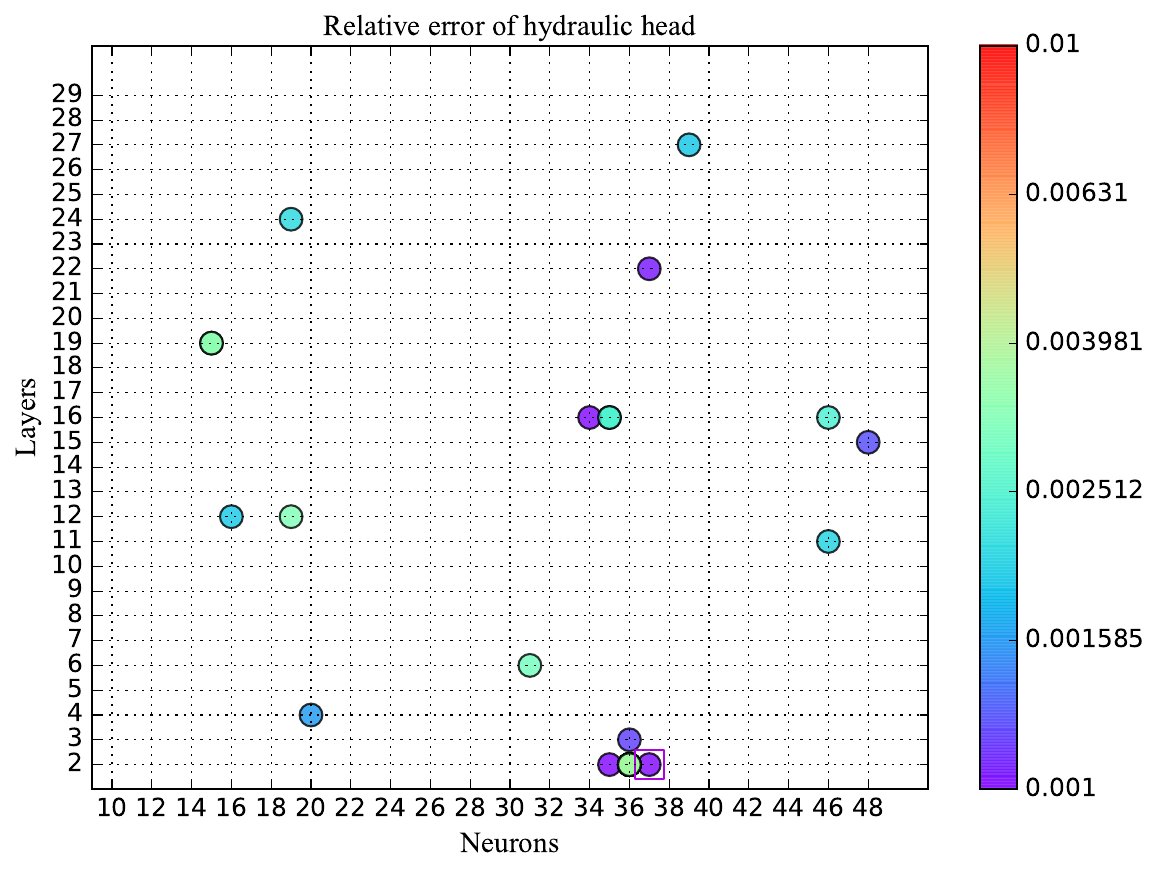}   
		\caption{}
	\end{subfigure}%
	\vspace{0.2cm}
	\begin{subfigure}[b]{7.0cm}
		\centering\includegraphics[height=6cm,width=7.5cm]{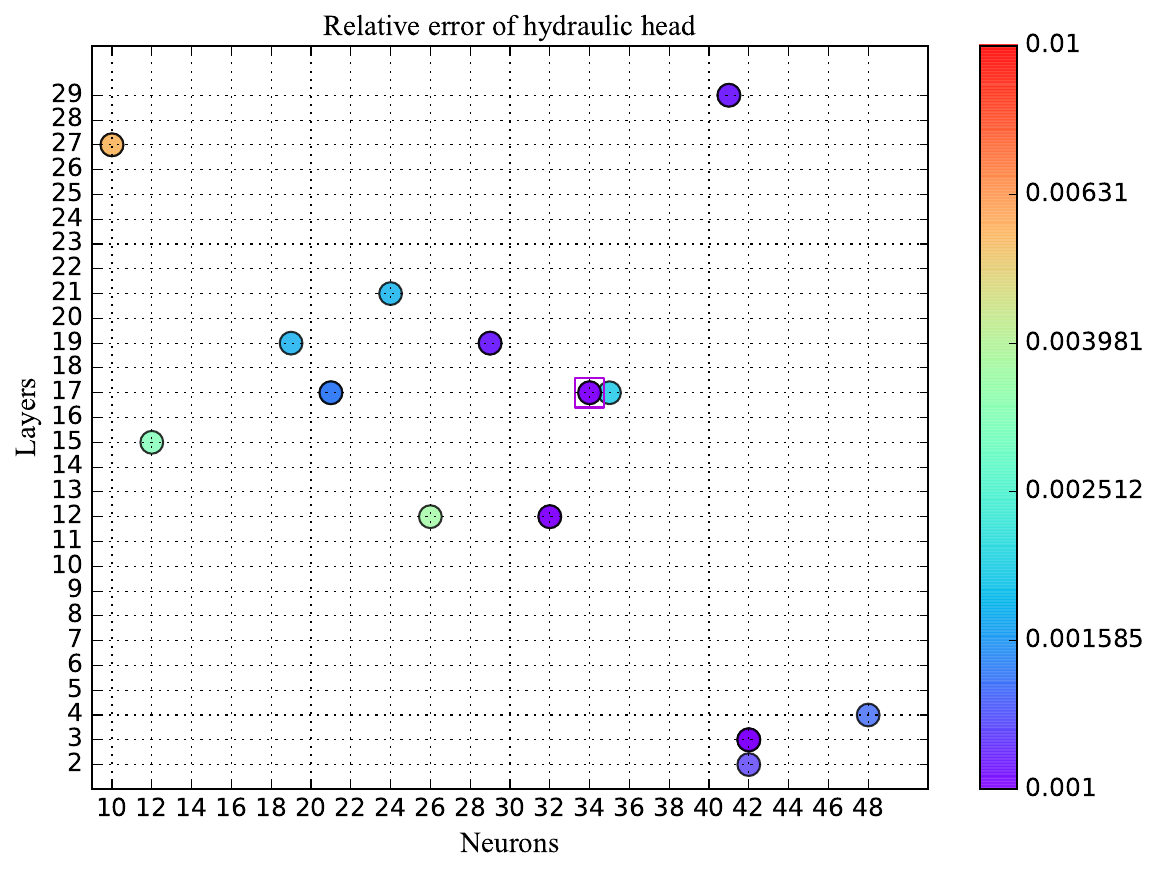}   
		\caption{}
	\end{subfigure}%
	\hspace{0.1cm}
	\begin{subfigure}[b]{7.0cm}
		\centering\includegraphics[height=6cm,width=7.5cm]{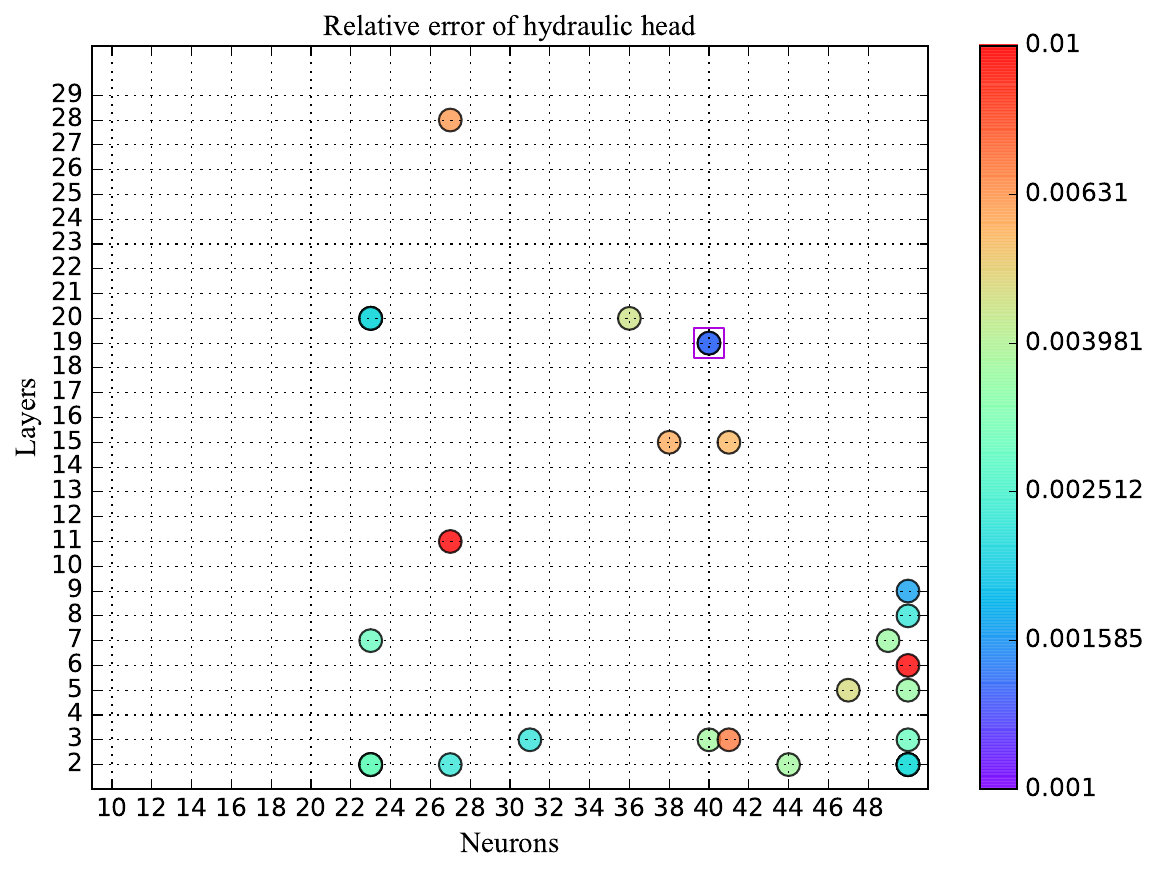}   
		\caption{}
	\end{subfigure}%  
	\caption{Neural network configuration search results with $\left(a\right)$.randomization search method;$\left(b\right)$.Bayesian optimization;$\left(c\right)$.Hyperband optimization;$\left(d\right)$.Jaya optimization}
	\label{fig:HyperOpt}
\end{figure}

The time required for each method and the search accuracy are shown in the table \ref{tab:Table11}:
\begin{table}[H]
	\captionsetup{width=0.55\columnwidth}
	\caption{Hyper-parameters search results with different algorithms}
	\vspace{-0.3cm}
	\centering 
	\resizebox{0.4\columnwidth}{!}{
		\begin{tabular}{c|c|c} 
			\toprule 
			\toprule 
			Algorithms & Time & Relative error\\ 
			\midrule 
			RSM  & 1830s & 0.00051\\
			\midrule
			Bayesian  &  1395s & 0.00032\\ 
			\midrule
			Hyperband  & 1449s & 0.00058\\ 
			\midrule
			Jaya   &  1757s & 0.00139\\ 
			\bottomrule 
		\end{tabular}
	}
	\label{tab:Table11}
\end{table}

From this we can conclude that for our model, the Bayesian method gives the best accuracy in the shortest possible time. Therefore, in the following calculations, we will take the Bayesian algorithm.

It should be noted that, due to the limited numbers of search, the optimal solution searched by the algorithm is not necessarily the best and will gradually approach the optimal configuration as the number of searches increases, but our experiments show that the configuration searched by us can still achieve a particularly good accuracy even when the number of search is small.

We continue the screening of 2D and 3D neural network configurations using a Bayesian approach, with the following results:

\begin{figure}[H]
	\captionsetup{width=0.9\columnwidth}
	\includegraphics[height=9cm]{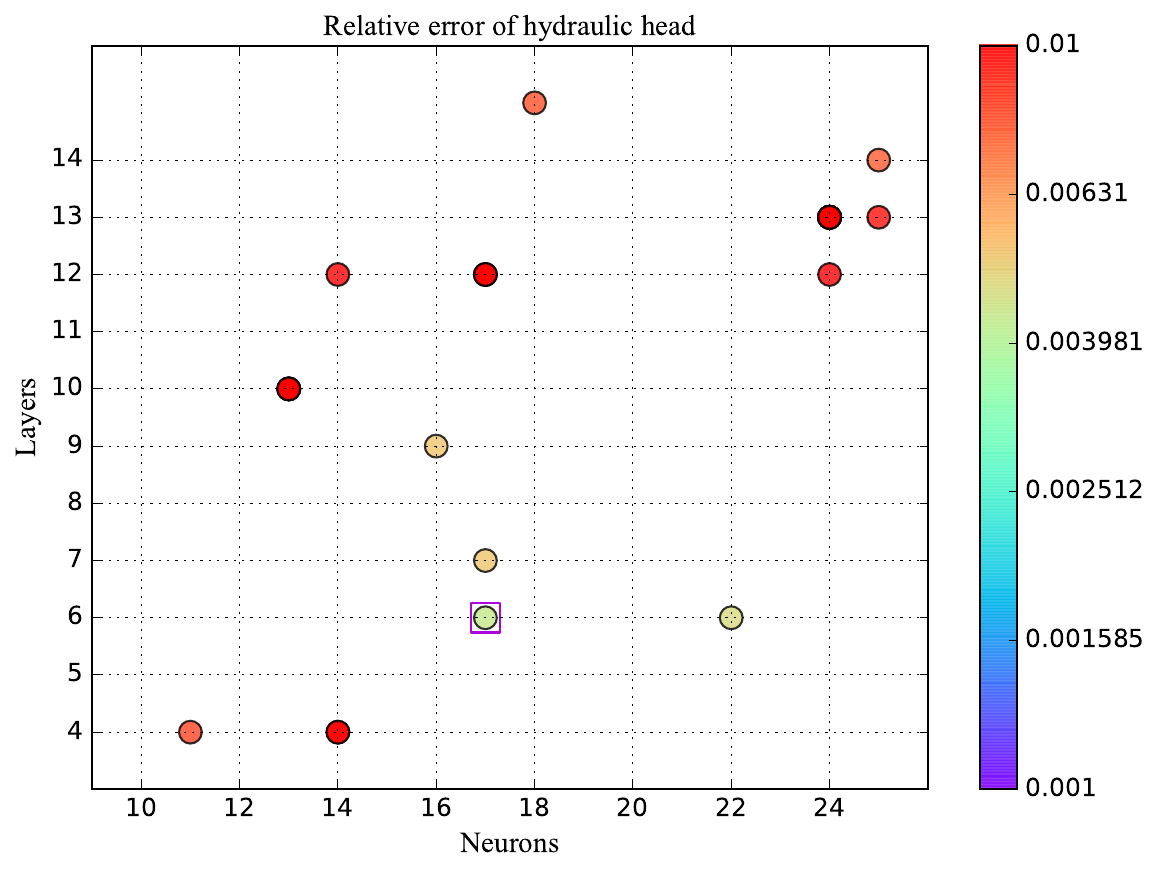}
	\centering
	\caption{Neural network configuration search results of Bayesian optimization in two dimension}
	\label{fig:bo2d}
\end{figure}

\begin{figure}[H]
	\captionsetup{width=0.9\columnwidth}
	\includegraphics[height=9cm]{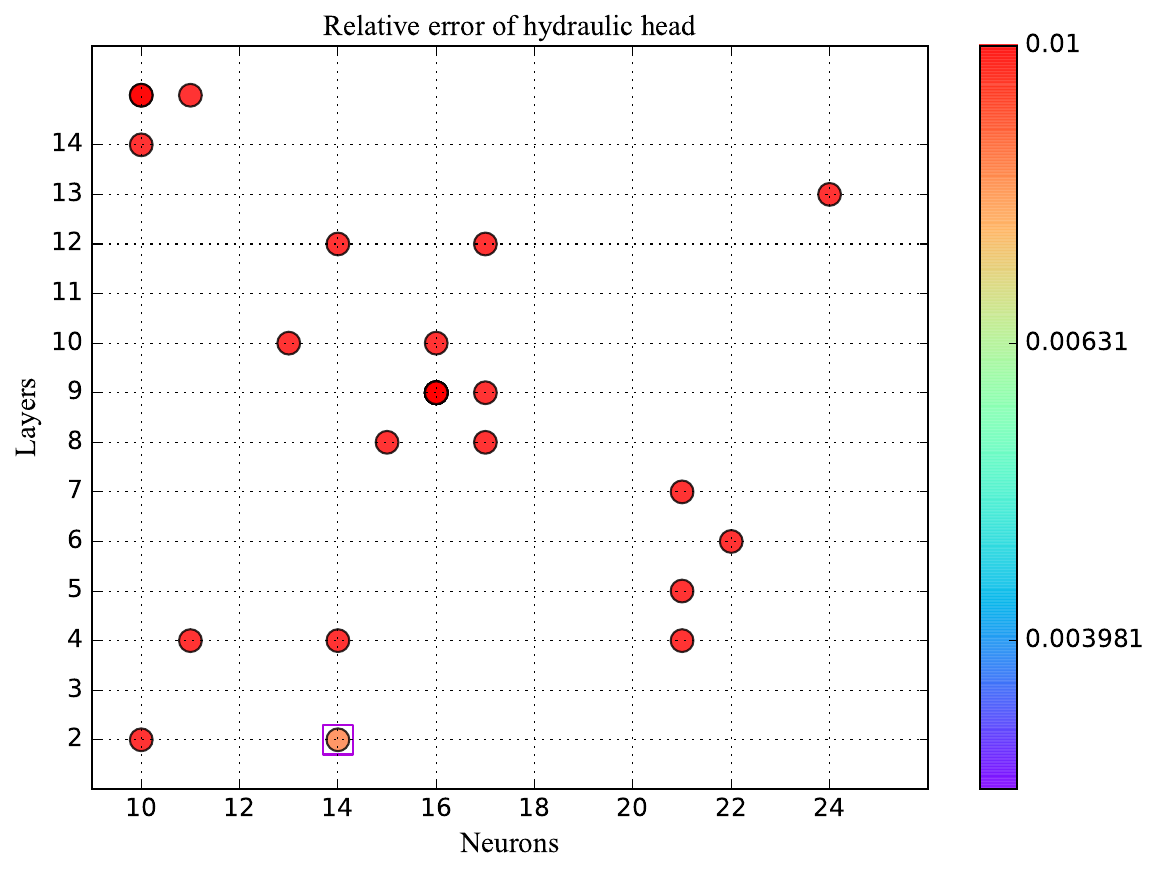}
	\centering
	\caption{Neural network configuration search results of Bayesian optimization in three dimension}
	\label{fig:bo3d}
\end{figure}

The optimal configuration obtained after screening is shown in the following table:

\begin{table}[H]
	\captionsetup{width=0.85\columnwidth}
	\caption{Neural architecture search results with Bayesian optimization}
	\vspace{-0.3cm}
	\centering 
	\resizebox{0.4\columnwidth}{!}{
		\begin{tabular}{c|c|c}
			\toprule 
			\toprule 
			Dimension & Layer & Neurons\\
			\midrule
			1D  & 2 & 37\\ 
			\midrule
			2D  &  6 & 17\\ 
			\midrule
			3D  & 2 & 14\\ 
			\bottomrule 
		\end{tabular}
	}
	\label{tab:Table12}
\end{table}

These neural network configurations will be used as input for the next numerical tests.

\subsection{Finite difference method(FDM)}
\label{subsection 4.4:FDM}
In this section, the Finite difference methods(FDM) for solving the modified governing equation is introduced, which will be used as comparisons for later model evaluation. We also tried FEM for solving this groundwater flow in heterogeneous porous media, the numerical results are rather bad, especially with the increasing of $\sigma^2$ in the heterogeneous random hydraulic conductivity. Instead, FDM results is more advantageous. Finite difference methods(FDM) is a type of the numerical method of differential equations seeks approximate solutions to differential equations by approximating the derivative by finite differences. When the target region is divided approximately closely, the more accurate the approximate solution obtained with the finite difference method will be, although the corresponding computation time will be considerably higher. The modified Darcy equation is solved with the following FDM scheme \cite{li2017numerical}:

For 1D:
\begin{equation}\label{eq:FDM1}
\begin{split}
K(x_i-\frac{\Delta x}{2})\hat{h}_{i-1}-[K(x_i+\frac{\Delta x}{2})+K(x_i-\frac{\Delta x}{2})]\hat{h}_i+K(x_i+\frac{\Delta x}{2})\hat{h}_{i+1}=\Delta x^2 f_i,
\end{split}
\end{equation}
where $f_i$ represents the data $f$ at the collation points $x_i$ and $\hat{h}_i$ is the value of the numerical solution at the same collation points.

For 2D:
\begin{equation}\label{eq:FDM2}
\begin{split}
A_{i,j} \hat{h}_{i-1,j}+B_{i,j} \hat{h}_{i,j-1}+C_{i,j} \hat{h}_{i,j}+D_{i,j} \hat{h}_{i+1,j}+E_{i,j} \hat{h}_{i,j+1}=f_{i,j},
\end{split}
\end{equation}
where $f_{i,j}$ represents the data $f$ at the grid points $(x_i,y_j)$ and $\hat{h}_{i,j}$ is the value of the numerical solution at the same grid points. In formula~\eqref{eq:FDM2}, we have used the following notations:

\begin{equation}\label{eq:FDMparameter2}
\left\{
\begin{array}{lr}
A_{i,j}:=\frac{1}{(\Delta x)^2}K(x_i-\frac{\Delta x}{2},y_j), &  \\
B_{i,j}:=\frac{1}{(\Delta y)^2}K(x_i,y_i-\frac{\Delta y}{2}), &  \\
D_{i,j}:=\frac{1}{(\Delta z)^2}K(x_i+\frac{\Delta x}{2},y_j), &  \\ 
E_{i,j}:=\frac{1}{(\Delta x)^2}K(x_i,y_j+\frac{\Delta y}{2}), &  \\
C_{i,j}:= -[A_{i,j}+B_{i,j}+D_{i,j}+E_{i,j}].
\end{array}
\right.
\end{equation}

For 3D:
\begin{equation}\label{eq:FDM3}
\begin{split}
&A_{i,j,k} \hat{h}_{i-1,j,k}+B_{i,j,k} \hat{h}_{i,j-1,k}+C_{i,j,k} \hat{h}_{i,j,k-1}+D_{i,j,k} \hat{h}_{i,j,k}\\
&+E_{i,j,k} \hat{h}_{i+1,j,k}+F_{i,j,k} \hat{h}_{i,j+1,k}+G_{i,j,k}\hat{h}_{i,j,k+1}=f_{i,j,k},
\end{split}
\end{equation}
where $f_{i,j,k}$ represents the data $f$ at the grid points $(x_i,y_j,z_k)$ and $\hat{h}_{i,j,k}$ is the value of the numerical solution at the same grid points. In formula~\eqref{eq:FDM3}, we have used the following notations:

\begin{equation}\label{eq:FDMparameter3}
\left\{
\begin{array}{lr}
A_{i,j,k}:=\frac{1}{(\Delta x)^2}K(x_i-\frac{\Delta x}{2},y_j,z_k), &  \\
B_{i,j,k}:=\frac{1}{(\Delta y)^2}K(x_i,y_i-\frac{\Delta y}{2},z_k), &  \\
C_{i,j,k}:=\frac{1}{(\Delta z)^2}K(x_i,y_j,z_k-\frac{\Delta z}{2}), &  \\ 
E_{i,j,k}:=\frac{1}{(\Delta x)^2}K(x_i+\frac{\Delta x}{2},y_j,z_k), &  \\
F_{i,j,k}:=\frac{1}{(\Delta y)^2}K(x_i,y_i+\frac{\Delta y}{2},z_k), &  \\
G_{i,j,k}:=\frac{1}{(\Delta z)^2}K(x_i,y_j,z_k+\frac{\Delta z}{2}), &  \\ 
D_{i,j,k}:= -[A_{i,j,k}+B_{i,j,k}+C_{i,j,k}+E_{i,j,k}+F_{i,j,k}+G_{i,j,k}].
\end{array}
\right.
\end{equation}

Applying above Equations, we can write matlab programs to obtain approximate solutions.

\subsection{Model validation in different dimensional}
\label{subsetion 4.5: numerival tests}
Now we solve the modified Darcy Equation~\eqref{eq:darcy with f} by the DCM method with PINNs configurations searched from modified NAS model, which is constructed in section ~\ref{section 3:neural networks}, and the FDM method in section~\ref{subsection 4.4:FDM} is added in the comparison, to test the feasibility of our approach. We fix $\sigma_2$ to 0.1, and $N$ to 1000, and observe the hydraulic head and velocity distribution inside the physical domain. In order to reflect the adaptability of our algorithm, we construct a new set of manufactured solutions, the specific form and source term are detailed in~\ref{appendix b}.

\subsubsection{One dimensional case model validation}
\label{subsubsection 4.5.1:1D}
The manufactured solution we used for 1D case is Equation~\eqref{eq:u1}. We validate the two methods by comparing hydraulic head in the $x$-direction over the interval $[0,25]$.

\begin{figure}[H]
	\captionsetup{width=0.9\columnwidth}
	\includegraphics[height=9cm]{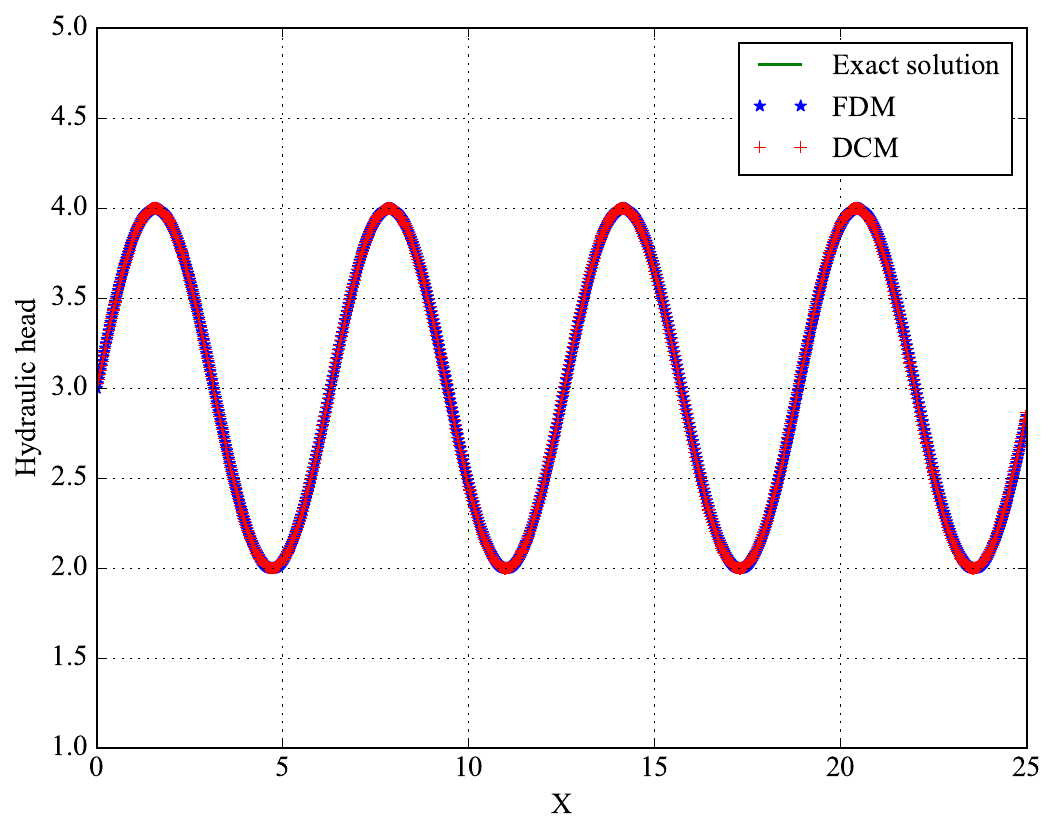}
	\centering
	\caption{Hydraulic head calculated by FDM and DCM methods in one dimension}
	\label{fig:fp1d}
\end{figure}

It is clear from Figure~\ref{fig:fp1d} that, in one dimension,  both methods work particularly well for our approximation of hydraulic head for partial differential equations.

\subsubsection{Two dimensional case model validation}
\label{subsubsection 4.5.2:2D}
The manufactured solution used for 2D case is Equation~\eqref{eq:u21}. We validate the two methods by comparing hydraulic head and velocity calculated in the $x$-direction, along midlines $y=10$, over the interval $[0,20]\times[0,20]$.

\begin{figure}[H]
	\captionsetup{width=0.9\columnwidth}
	\includegraphics[height=9cm]{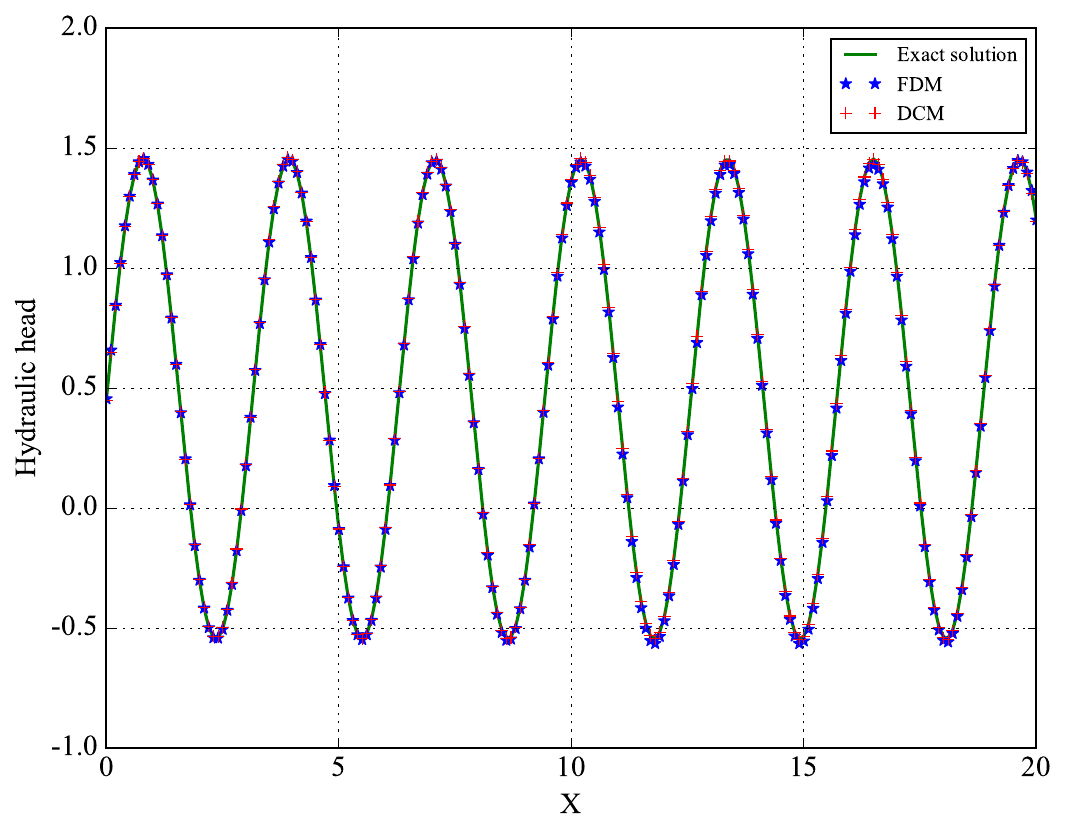}
	\centering
	\caption{Hydraulic head along $y=10$ calculated by FDM and DCM methods in two dimension}
	\label{fig:fp2d}
\end{figure}
\begin{figure}[H]
	\captionsetup{width=0.9\columnwidth}
	\includegraphics[height=9cm]{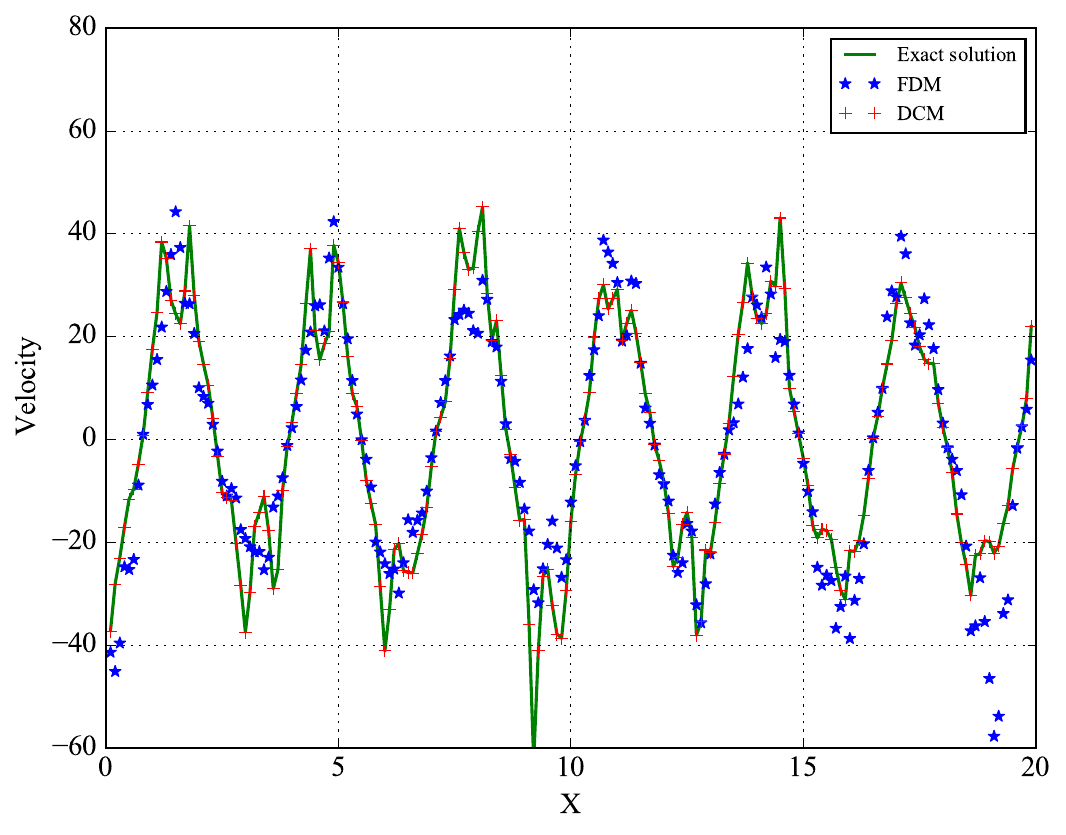}
	\centering
	\caption{Velocity in x-direction along $y=10$ calculated by FDM and DCM methods in two dimension}
	\label{fig:fpv2d}
\end{figure}

Seen from Figure~\ref{fig:fp2d}, for this two dimensional case, both methods match well with the exact solution as to compute the hydraulic head. But as shown in Figure~\ref{fig:fpv2d}, the FDM method's simulation results of $v_x$ are a bit off, while it still agrees well for results predicted by the proposed method.

\subsubsection{Three dimensional case model validation}
\label{subsubsection 4.5.3:3D}
The manufactured solution we used for 3D case is in Equation~\eqref{eq:u31}. We validate the two methods by comparing hydraulic head and velocity in the $x$-direction, along along the $y=1, z=0.5$, over the interval $[0,5]\times[0,2]\times[0,1]$.

The situation is quite different in 3D case, with the mesh density, the results obtained with the FDM method are far diverting from the actual situation and do not fit at all, while the results predicted with the DCM method are still in good agreement of exact solution. The specific accuracy and calculation times for the two methods are shown in Tables~\ref{tab:Table13} and \ref{tab:Table14}:

\begin{table}[H] 
	\captionsetup{width=0.85\columnwidth}
	\caption{Solving darcy equation with FDM} 
	\vspace{-0.3cm}
	\centering 
	\resizebox{0.55\columnwidth}{!}{
		\begin{tabular}{l|c|c|c} 
			\toprule 
			\toprule
			\diagbox{Results}{Dimension}&1&2&3\\ 
			\midrule 
			Relative error&6.443e-5&0.017&5.711\\ 
			\midrule
			Time&2.8s&180s&1245s\\ 
			\bottomrule 
		\end{tabular}
	}
	\label{tab:Table13}
\end{table}

\begin{table}[H]   
	\captionsetup{width=0.85\columnwidth}
	\caption{Solving darcy equation with DCM} 
	\vspace{-0.3cm}
	\centering 
	\resizebox{0.8\columnwidth}{!}{
		\begin{tabular}{l|c|c|c|c|c|c} 
			\toprule 
			\toprule 
			\multirow{2}*{\diagbox{Results}{Dimension}}&\multicolumn{2}{c|}{1}&\multicolumn{2}{c|}{2}&\multicolumn{2}{c}{3}\\
			\cline{2-7}
			~ &without TL&with TL&without TL&with TL&without TL&with TL\\   
			\midrule
			Relative error&1.369e-4&1.195e-4&4.262e-3&4.405e-3&8.915e-3&8.864e-3\\ 
			\midrule
			Time&14.3s&1.5s&108.5s&9.8s&32.2s&1.9s\\ 
			\bottomrule 
		\end{tabular}
	}
	\label{tab:Table14}
\end{table}

\begin{figure}[H]
	\captionsetup{width=0.9\columnwidth}
	\includegraphics[height=9cm]{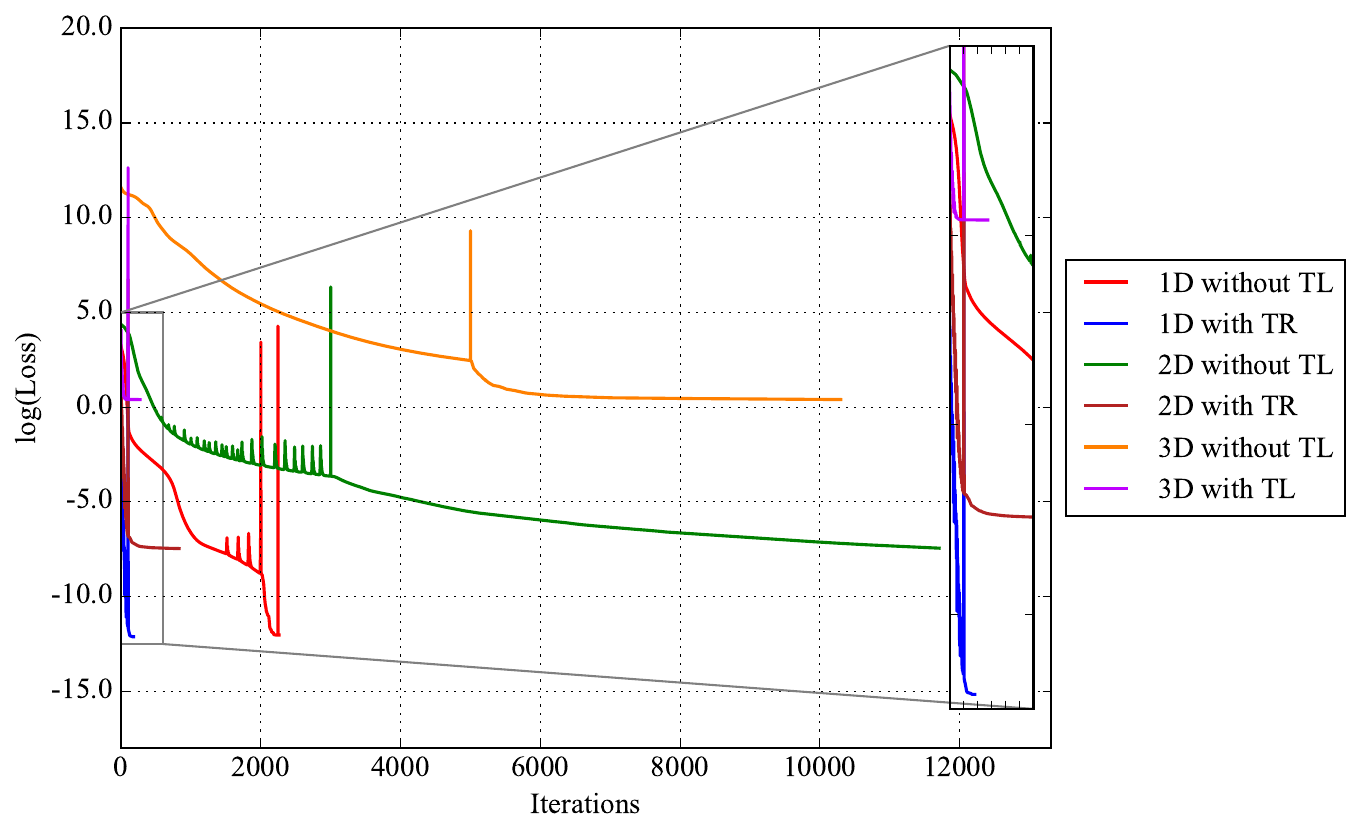}
	\centering
	\caption{Logarithm loss function with and without transfer learning}
	\label{fig:lossgauss}
\end{figure}

The results showed that the higher the dimensionality, the more pronounced the difference between the effects of the two methods. It's predictable, that, while the results with FDM method become more accurate as the mesh is more finely divided, this takes an exceptionally long time and has absolutely no value for use in practical applications. However for the proposed method, with transfer learning, the calculation time is less reduced, whose computational even increased in parallel with FDM for one dimensional case. For higher dimensions, the computational time with Transfer learning is also greatly reduced, and with accuracy even slightly improved. 

Then, the contour for hydraulic head and velocity for the three dimensional case with the proposed method is shown below:
\begin{figure}[H]
	\captionsetup{width=0.85\columnwidth}
	\centering
	\begin{subfigure}[b]{6.0cm}
		\centering\includegraphics[height=6cm,width=6.0cm]{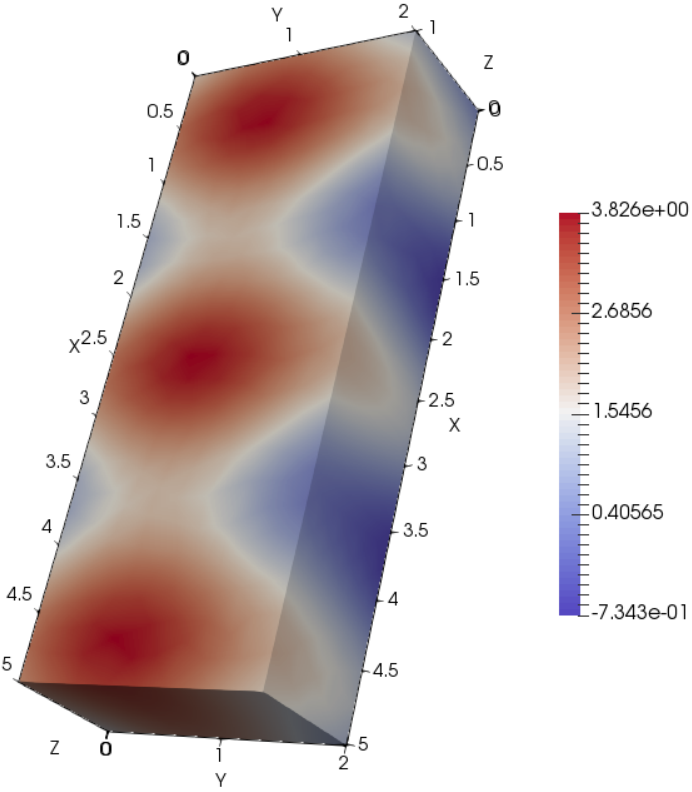}   
		\caption{}
	\end{subfigure}%
	\hspace{0.5cm}
	\begin{subfigure}[b]{6.0cm}
		\centering\includegraphics[height=6cm,width=6.0cm]{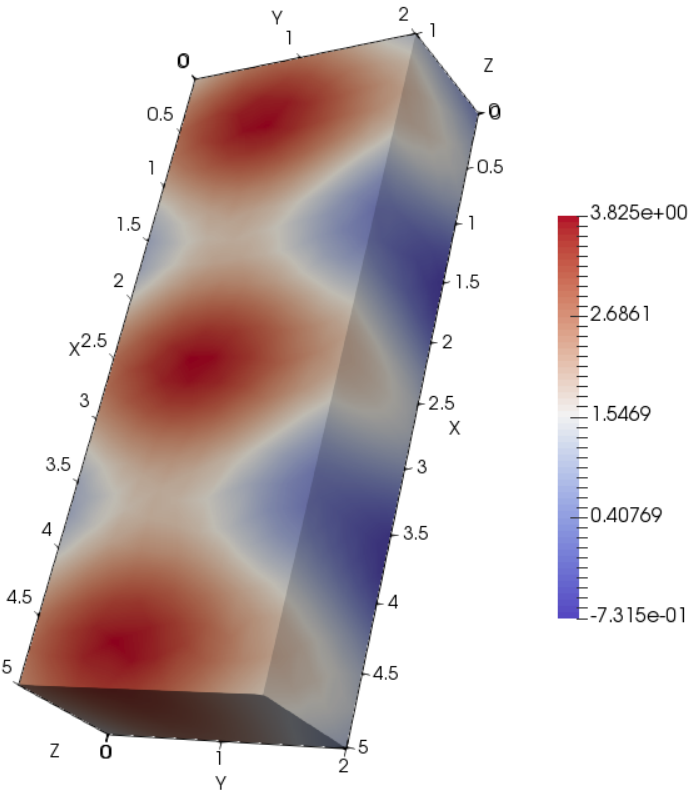}   
		\caption{}
	\end{subfigure}% 
	\caption{Three dimensional hydraulic head when $\sigma^2=0.1$, $N=1000$ with Gaussian correlation $\left(a\right)$ exact solution and $\left(b\right)$ predict solution}
	\label{fig:3Dgaussu}
\end{figure}
It demonstrates that the proposed method agrees well with the adopted exact solution and the Gaussian random field with Gaussian correlation function fits really well with the PINNs based model. 
\begin{figure}[H]
	\captionsetup{width=0.85\columnwidth}
	\centering
	\begin{subfigure}[b]{6.0cm}
		\centering\includegraphics[height=6cm,width=6.0cm]{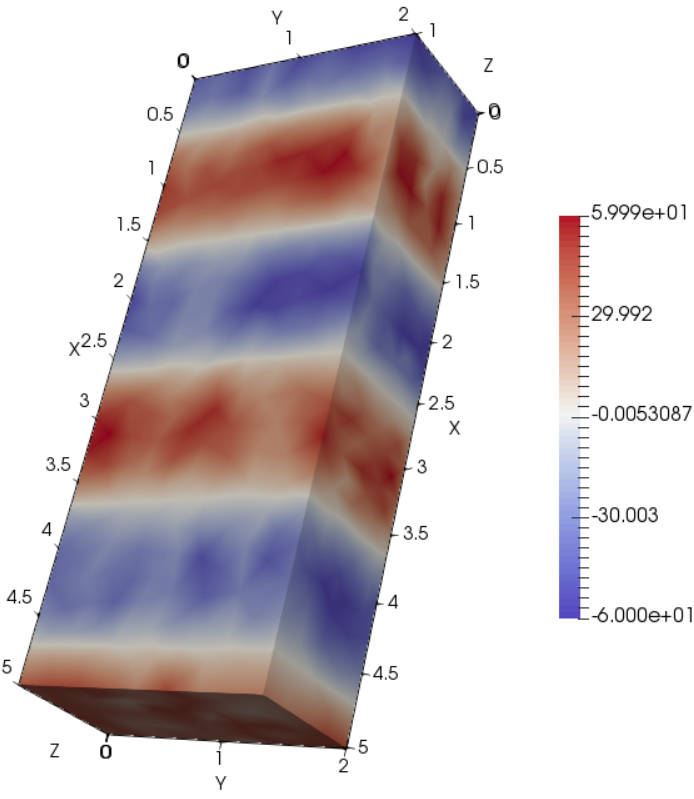}   
		\caption{}
	\end{subfigure}%
	\hspace{0.5cm}
	\begin{subfigure}[b]{6.0cm}
		\centering\includegraphics[height=6cm,width=6.0cm]{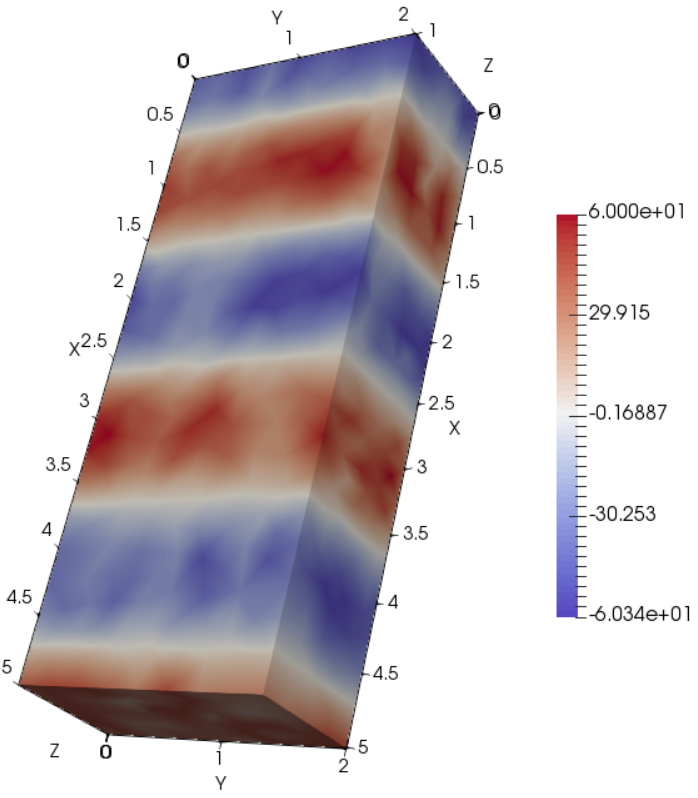}   
		\caption{}
	\end{subfigure}% 
	\caption{Three dimensional velocity when $\sigma^2=0.1$, $N=1000$ with Gaussian correlation $\left(a\right)$ exact solution and $\left(b\right)$ predict solution}
	\label{fig:3Dgaussv}
\end{figure}

The isosurface diagram of the predicted head and velocity in the 3D case is shown below:
\begin{figure}[H]
	\captionsetup{width=0.85\columnwidth}
	\centering
	\begin{subfigure}[b]{6.0cm}
		\centering\includegraphics[height=6cm,width=6.0cm]{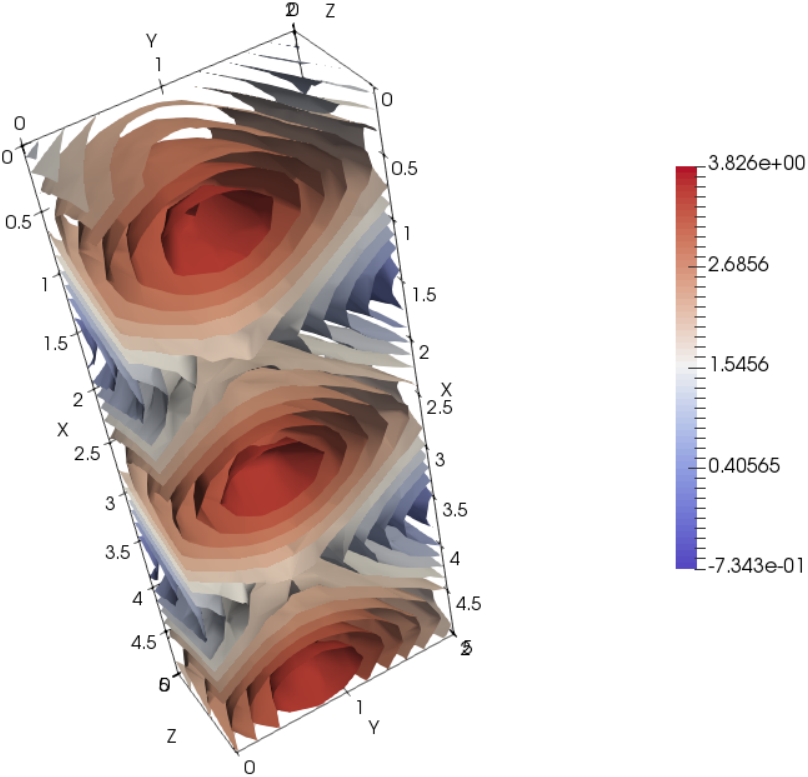}   
		\caption{}
	\end{subfigure}%
	\hspace{0.5cm}
	\begin{subfigure}[b]{6.0cm}
		\centering\includegraphics[height=6cm,width=6.0cm]{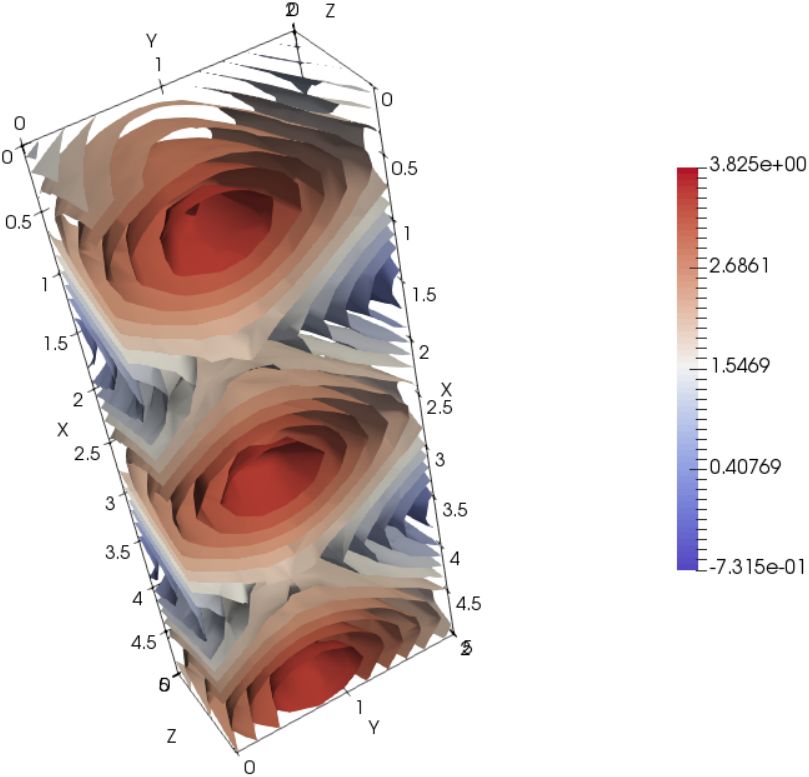}   
		\caption{}
	\end{subfigure}% 
	\caption{Three dimensional isosurface diagram of hydraulic head when $\sigma^2=0.1$, $N=1000$ with Gaussian correlation $\left(a\right)$ exact solution and $\left(b\right)$ predict solution}
	\label{fig:3Dgaussuc}
\end{figure}

\begin{figure}[H]
	\captionsetup{width=0.85\columnwidth}
	\centering
	\begin{subfigure}[b]{6.0cm}
		\centering\includegraphics[height=6cm,width=6.0cm]{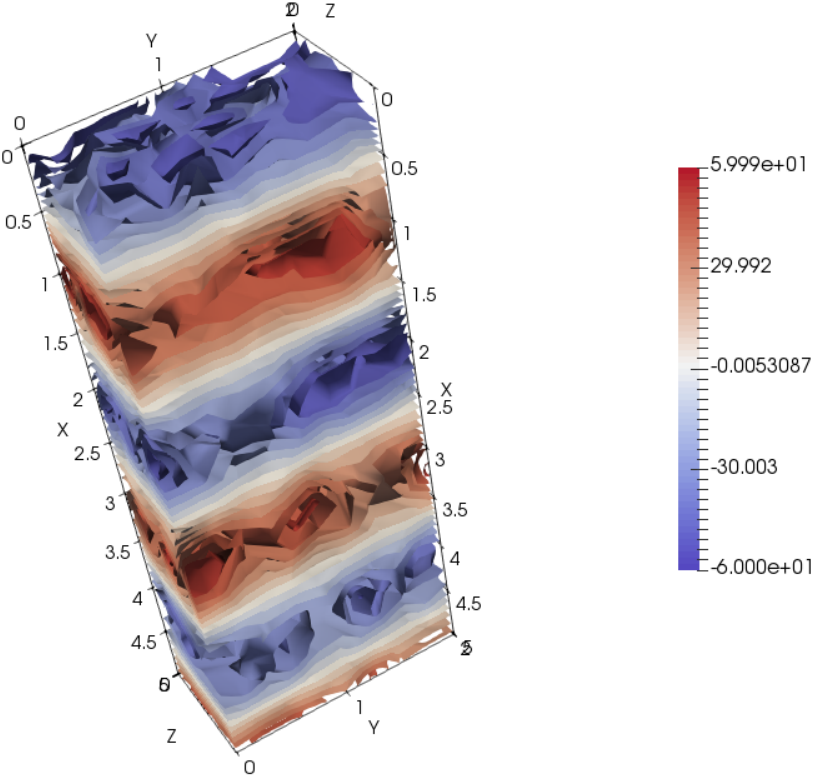}   
		\caption{}
	\end{subfigure}%
	\hspace{0.5cm}
	\begin{subfigure}[b]{6.0cm}
		\centering\includegraphics[height=6cm,width=6.0cm]{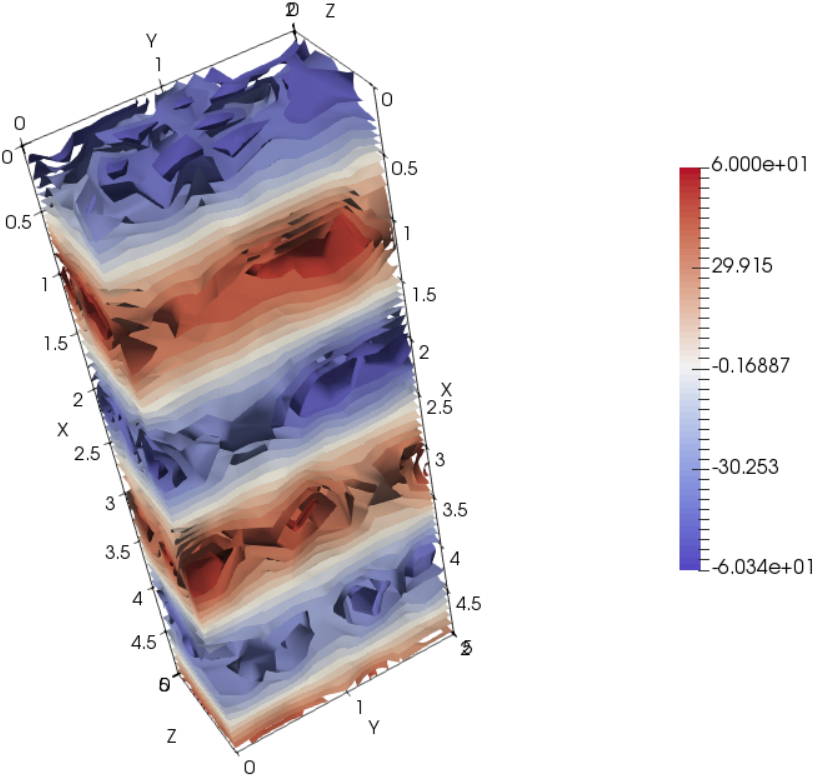}   
		\caption{}
	\end{subfigure}% 
	\caption{Three dimensional isosurface diagram of velocity when $\sigma^2=0.1$, $N=1000$ with Gaussian correlation $\left(a\right)$ exact solution and $\left(b\right)$ predict solution}
	\label{fig:3Dgaussvc}
\end{figure}	

We can also see from the Figures~\ref{fig:3Dgaussu}-\ref{fig:3Dgaussvc} that our model performs the task of fitting the modified Darcy equation for the three-dimensional case very well with higher precision and less computational time compared with FDM.

\section{Conclusion}
\label{section 5:conclusion}

In this paper, we mainly describe the composition of the physics-informed modified NAS model fitted to the deep collocation method with sensitivity analysis and transfer learning. The random spectral method in closed form is adopted for the generation of log-normal hydraulic conductivity fields. It was calibrated computationally to generate the heterogeneous hydraulic conductivity field with Gaussian correlation function, and by sensitivity analysis and comparing hyperparameter selection optimizers, Bayesian algorithm was identified as the most suitable optimizer for the search strategy in NAS model. Further, the three dimensional benchmark for groundwater flow in highly heterogeneous aquifers are constructed and used as the verification of the proposed method. By introducing transfer learning, the training time is greatly reduced and accuracy improved. 

Since no feature engineering is involved in our physics-informed model, the modified NAS based deep collocation method can be considered as truly automated "meshfree" method and can be used to approximate any continuous function, which is verified to be very suitable for the analysis of groundwater flow problems. The automated deep collocation method is very simple in implementation, which can be further applied in a wide variety of engineering problems.

From our numerical experiments, it is clearly demonstrated that the modified NAS based deep collocation method with randomly distributed collocations and a physics-informed neural networks perform very well with a combined MSE loss function minimized by the two step Adam and L-BFGS optimizer. In order to achieve the same accuracy, the FDM method requires a tighter grid, and the computation time, although only takes 2.8s in one dimensional with FDM method, time required with our method with transfer learning is much faster and much more accurate in higher dimensions than with the FDM method, especially in the case of 3D. This lends more credence to extend the application of the proposed approaches to more complex practical situations. Most importantly, once those deep neural networks are trained, they can be used to evaluate the solution at any desired points with minimal additional computation time.

\section*{Acknowledgement}
\label{Acknowledgement}
The authors would like to thank Dawei Liang for his help regarding code development.

\newpage
\appendix
\section{Choices of random variables $\bm{k}$}\label{appendix a}
From section~\ref{subsection 2.2:generate the hydraulic conductivity fields}, we can derive the following probability density function(PDF) $p(\bm{k})$ for exponential and Gaussian correlations:
\begin{equation}\label{eq:exp pdf}
p(\bm{k})= \lambda^d \frac{\Gamma[\frac{d+1}{2}]}{(\pi(1+(\bm{k}\lambda)^2))^{\frac{d+1}{2}}}
\end{equation}
\begin{equation}\label{eq:gauss pdf}
p(\bm{k})= \pi^{d/2}\lambda^d e^{-(\pi\bm{k}\lambda)^2}
\end{equation}
where $\Gamma$ means the gamma function, with $\Gamma(n)=(n-1)!$ and $\Gamma(n+1/2)=\frac{\sqrt{\pi}}{2^n}(2n-1)!!$ for $n=1,2,3...$

For Gaussian correlations in three dimensional case, the PDF of $\bm{k}$ can be transformed into the following form:
\begin{equation}\label{eq:gauss pdf2}
p(\bm{k})= (\sqrt{\pi}\lambda_1 e^{-(\pi k_1\lambda_1)^2})(\sqrt{\pi}\lambda_2 e^{-(\pi k_2\lambda_2)^2})(\sqrt{\pi}\lambda_3 e^{-(\pi k_3\lambda_3)^2})
\end{equation}

Each part of equation~\eqref{eq:gauss pdf2} can be considered as a normal distribution with $\mu=0$ and $\sigma=\frac{1}{\sqrt{2}\pi\lambda_i}$. Thus, the random vector $\bm{k}$ can be simulated by the formula $\bm{k}=\frac{1}{\sqrt{2}\pi}(\mu_1/\lambda_1,\mu_2/\lambda_2,\mu_3/\lambda_3)$, where $\mu_i$ are independent standard Gaussian random variables.

The $\bm{k}$ in the two-dimensional case can be derived by analogy from the above inference as $\bm{k}=\frac{1}{\sqrt{2}\pi}(\mu_1/\lambda_1,\mu_2/\lambda_2)$.

For exponential correlations in two dimensional case, the PDF of $\bm{k}$ can be transformed into the following form:

\begin{equation}\label{eq:exp 2D}
p(\bm{k})= \frac{\lambda_1 \lambda_2} {2\pi\big (1+(k_1\lambda_1)^2+(k_2\lambda_2)^2\big )^{\frac{3}{2}}}
\end{equation}

A possible solution to calculate the cumulative distribution function(CDF) is the transformation from Cartesian into polar coordinates, i.e.
a representation like:
\begin{equation}\label{eq:exp k2}
\begin{split}
k_1=r cos(2\pi \hat{h})/\lambda_1,\\
k_2=r sin(2\pi \hat{h})/\lambda_2.
\end{split}
\end{equation}

Here $\hat{h}$ is a uniformly distributed random variable and $r$ is a
random variable distribute according to
\begin{equation}\label{eq:exp E2}
p_r(r)=\frac{2\pi r p(r)}{\lambda_1\lambda_2}.
\end{equation}

Integrating the equation~\eqref{eq:exp E2} yields the CDF

\begin{equation}\label{eq:exp cdf2}
\begin{split}
F(r)&=\int_{-\infty}^{r} p_r(r) \mathrm{d}r\\
&=\int_{-\infty}^{r} \frac{r}{(1+r^2)^{3/2}} \mathrm{d}r\\
&=-\frac{1}{(1+r^2)^{1/2}}\bigg|_{-\infty}^{r}\\
&=-\frac{1}{(1+r^2)^{1/2}}
\end{split}
\end{equation}

Choose a uniformly distributed random variable $\mu$, the inverse function $k=F^{-1}(\mu)$ can be obtained

\begin{equation}\label{eq:exp r}
\begin{split}
\mu&=-\frac{1}{(1+r^2)^{1/2}}\\
\mu^2&=\frac{1}{1+r^2}\\
r&=\sqrt{1/\mu^2-1}
\end{split}
\end{equation}

Substitute equation~\eqref{eq:exp k2} into equation~\eqref{eq:exp k2}, we get the $k_1=(1/\mu^2-1)^{1/2}cos(2\pi \hat{h})/\lambda_1$, and $k_2=(1/\mu^2-1)^{1/2}sin(2\pi \hat{h})/\lambda_2$.

For exponential correlations in three dimensional case, the PDF of $\bm{k}$ can be transformed into the following form:

\begin{equation}\label{eq:exp 3D}
p(\bm{k})= \frac{\lambda_1 \lambda_2 \lambda_3} {\pi^2(1+(k_1\lambda_1)^2+(k_2\lambda_2)^2+(k_3\lambda_3)^2)^2}
\end{equation}

A similar procedure can be used, where spherical instead of polar coordinates are used

\begin{equation}
\begin{split}
k_1&=r sin(\theta) cos(2\pi \gamma)/\lambda_1,\\
k_2&=r sin(\theta) sin(2\pi \gamma)/\lambda_2,\\
k_3&=r cos(\theta)/\lambda_3.
\end{split}
\end{equation}

Here $\gamma$ is again a uniformly distribute random variable and $\theta$ is given as

\begin{equation}\label{eq:theta}
\theta = \arccos(1 - 2\xi),
\end{equation}
with $\xi$ being a uniformly distribute random variable. The two random variables were chosen with reference to Weissten's research on generating random points on the surface of a unit sphere \cite{weisstein2002sphere}. The radius $r$ is distributed according to

\begin{equation}\label{eq:exp E3}
p_r(r)=4\pi r^2 p(r).
\end{equation}

The CDF can be calculated as follows,

\begin{equation}\label{eq:exp cdf3d}
\begin{split}
F(r)&=\int_{-\infty}^{r} p_r(r) dr\\
&=\int_{-\infty}^{r} \frac{4 r^2}{\pi(1+r^2)^2} dr\\
&=\frac{2}{\pi}(-\frac{r}{1+r^2}\bigg|_{-\infty}^{r}-\int_{-\infty}^{r} -\frac{1}{1+r^2} dr)\\
&=\frac{2}{\pi}(\arctan(r)-\frac{r}{1+r^2})
\end{split}
\end{equation}

Choose a uniformly distributed random variable $\gamma_1$, $r$ can be obtained by solving the next equation~\eqref{eq:r}:

\begin{equation}\label{eq:r}
\frac{2}{\pi}(\arctan(r)-\frac{r}{1+r^2})=\gamma_1
\end{equation}

\section{Manufactured solutions}\label{appendix b}
For the 1D case, we have selected the following manufactured solution :
\begin{equation}\label{eq:u1}
\hat{h}_{MMS}(x)=3+sin(x), with \; x \in [0,25].
\end{equation}

This leads to the following Dirichlet boundary conditions :
\begin{equation}\label{eq:b1}
\left\{
\begin{array}{lr}
\hat{h}(0)=3, &  \\
\hat{h}(25)=3+sin(25). &  
\end{array}
\right.
\end{equation}

The function $K$ is now given by:
\begin{equation}\label{eq:k1}
K(x)=C_1 exp\bigg (C_2\sum_{i=1}^{N}cos\big (\xi_1 +2\pi(k_{i,1}x+k_{i,2})\big )\bigg ),
\end{equation}
where we use the shorthand notations $C_1=\langle K\rangle exp(-\frac{\sigma^2}{2})$ and $C_2=\sigma\sqrt{\frac{2}{N}}$. And the source term $f$ has the following form:
\begin{equation}\label{eq:f11}
\begin{split}
f(x)=&C_1 exp\bigg (C_2\sum_{i=1}^{N}cos\big (\xi_1 +2\pi (k_{i,1}x+k_{i,2})\big )\bigg )\\
&\cdot \bigg ((-2\pi)C_2k_{i,1}\sum_{i=1}^{N}sin\big (\xi_1 +2\pi (k_{i,1}x+k_{i,2})\big )cos(x)-sin(x)\bigg ).
\end{split}
\end{equation}

For the 2D case, we consider the following smooth manufactured solution :
\begin{equation}\label{eq:u2}
\hat{h}_{MMS}(x,y)=1+sin(2x+y), \quad with\quad x \in [0,20] \quad and\quad  y \in [0,20],
\end{equation}
along with the Dirichlet and Neumann boundary conditions:
\begin{equation}\label{eq:b2}
\left\{
\begin{array}{lr}
\hat{h}(0,y)=1+sin(y), & \forall  y \in [0,20],\\
\hat{h}(20,y)=1+sin(2\times20+y), &  \forall  y \in [0,20],\\
\frac{\partial \hat{h}}{\partial y}(x,0)=cos(2x), & \forall  x \in [0,20],\\
\frac{\partial \hat{h}}{\partial y}(x,20)=cos(2x+20), & \forall  x \in [0,20].\\
\end{array}
\right.
\end{equation}

The function $K$ is now given by:
\begin{equation}\label{eq:k2}
K(x,y)=C_1 exp\bigg (C_2\sum_{i=1}^{N}cos\big (\xi_1 +2\pi(k_{i,1}x+k_{i,2}y)\big )\bigg ),
\end{equation}
where we use the shorthand notations $C_1$ and $C_2$ same as in 1 dimensional case. And the source term $f$ has the following form:
\begin{equation}\label{eq:f22}
\begin{split}
f(x,y)=&2C_1C_2 exp\bigg (C_2\sum_{i=1}^{N}cos\big (\xi_1 +2\pi (k_{i,1}x+k_{i,2}y)\big )\bigg )\\
&\cdot \sum_{i=1}^{N}\bigg (-2\pi k_{i,1}sin\big (\xi_1 +2\pi (k_{i,1}x+k_{i,2})\big )\bigg )cos(2x+y)\\
&-5C_1exp\bigg (C_2\sum_{i=1}^{N}cos\big (\xi_1 +2\pi (k_{i,1}x+k_{i,2}y)\big )\bigg )sin(2x+y)\\
&+ C_1C_2exp\bigg (C_2\sum_{i=1}^{N}cos\big (\xi_1 +2\pi (k_{i,1}x+k_{i,2}y)\big )\bigg )\\
&\cdot \sum_{i=1}^{N}\bigg (-2\pi k_{i,2}sin\big (\xi_1 +2\pi (k_{i,1}x+k_{i,2})\big )\bigg )cos(2x+y).
\end{split}
\end{equation}

An alternative manufactured solution is:
\begin{equation}\label{eq:u21}
\hat{h}_{MMS}(x,y)=1+sin(2x)+sin(y), \quad with\quad x \in [0,20] \quad and\quad  y \in [0,20].
\end{equation}
along with the Dirichlet and Neumann boundary conditions:
\begin{equation}\label{eq:b21}
\left\{
\begin{array}{lr}
\hat{h}(0,y)=1+sin(y), & \forall  y \in [0,20],\\
\hat{h}(20,y)=1+sin(2\times20)+sin(y), &  \forall  y \in [0,20],\\
\frac{\partial \hat{h}}{\partial y}(x,0)=cos(0), & \forall  x \in [0,20],\\
\frac{\partial \hat{h}}{\partial y}(x,20)=cos(20), & \forall  x \in [0,20].\\
\end{array}
\right.
\end{equation}

The source term $f$ has the following form:
\begin{equation}\label{eq:f21}
\begin{split}
f(x,y)=&2C_1C_2 exp\bigg (C_2\sum_{i=1}^{N}cos\big (\xi_1 +2\pi (k_{i,1}x+k_{i,2}y)\big )\bigg )\\
&\cdot \sum_{i=1}^{N}\bigg (-2\pi k_{i,1}sin\big (\xi_1 +2\pi (k_{i,1}x+k_{i,2})\big )\bigg )cos(2x)\\
&-C_1exp\bigg (C_2\sum_{i=1}^{N}cos\big (\xi_1 +2\pi (k_{i,1}x+k_{i,2}y)\big )\bigg )\big (4sin(2x)+sin(y)\big )\\
&+ C_1C_2exp\bigg (C_2\sum_{i=1}^{N}cos\big (\xi_1 +2\pi (k_{i,1}x+k_{i,2}y)\big )\bigg )\\
&\cdot \sum_{i=1}^{N}\bigg (-2\pi k_{i,2}sin\big (\xi_1 +2\pi (k_{i,1}x+k_{i,2})\big )\bigg )cos(y).
\end{split}
\end{equation}

For the 3D case, we consider the following smooth manufactured solution :
\begin{equation}\label{eq:u3}
\hat{h}_{MMS}(x,y,z)=1+sin(3x+2y+z), \quad with\quad x \in [0,5],\quad  y \in [0,2]\quad and\quad x \in [0,1],
\end{equation}
along with the Dirichlet and Neumann boundary conditions:
\begin{equation}\label{eq:b3}
\left\{
\begin{array}{lr}
\hat{h}(0,y,z)=1+sin(2y+z), & \forall  y \in [0,2],\forall  z \in [0,1],\\
\hat{h}(5,y,z)=1+sin(3\times5+2y+z), &  \forall  y \in [0,2],\forall  z \in [0,1],\\
\frac{\partial \hat{h}}{\partial y}(x,0,z)=2cos(3x+z), & \forall  x \in [0,5],\forall  z \in [0,1],\\
\frac{\partial \hat{h}}{\partial y}(x,2,z)=2cos(3x+2\times2+z), & \forall  x \in [0,5],\forall  z \in [0,1],\\
\frac{\partial \hat{h}}{\partial z}(x,y,0)=cos(3x+2y), & \forall  x \in [0,5],\forall  y \in [0,2],\\
\frac{\partial \hat{h}}{\partial z}(x,y,1)=cos(3x+2y+1), & \forall  x \in [0,5],\forall  y \in [0,2].
\end{array}
\right.
\end{equation}

The function $K$ is now given by:
\begin{equation}\label{eq:k3}
K(x,y,z)=C_1 exp\bigg (C_2\sum_{i=1}^{N}cos\big (\xi_1 +2\pi(k_{i,1}x+k_{i,2}y+k_{i,3}z)\big )\bigg ),
\end{equation}
where we use the shorthand notations $C_2$ same as in 1 dimensional case, but $C_1=\langle K\rangle exp(-\frac{\sigma^2}{6})$. And the source term $f$ has the following form:
\begin{equation}\label{eq:f31}
\begin{split}
f(x,y,z)=&3C_1C_2 exp\bigg (C_2\sum_{i=1}^{N}cos\big (\xi_1 +2\pi (k_{i,1}x+k_{i,2}y+k_{i,3}z)\big )\bigg )\\
&\cdot \sum_{i=1}^{N}\big (-2\pi k_{i,1}sin\big (\xi_1 +2\pi (k_{i,1}x+k_{i,2}+k_{i,3}z)\bigg )\bigg )cos(3x+2y+z)\\
&+2C_1C_2 exp\bigg (C_2\sum_{i=1}^{N}cos\big (\xi_1 +2\pi (k_{i,1}x+k_{i,2}y+k_{i,3}z)\big )\bigg )\\
&\cdot \sum_{i=1}^{N}\bigg (-2\pi k_{i,2}sin\big (\xi_1 +2\pi (k_{i,1}x+k_{i,2}+k_{i,3}z)\big )\bigg )cos(3x+2y+z)\\
&+C_1C_2 exp\bigg (C_2\sum_{i=1}^{N}cos\big (\xi_1 +2\pi (k_{i,1}x+k_{i,2}y+k_{i,3}z)\big )\bigg )\\
&\cdot \sum_{i=1}^{N}\bigg (-2\pi k_{i,3}sin\big (\xi_1 +2\pi (k_{i,1}x+k_{i,2}+k_{i,3}z)\big )\bigg )cos(3x+2y+z)\\
&-14C_1 exp\bigg (C_2\sum_{i=1}^{N}cos\big (\xi_1 +2\pi(k_{i,1}x+k_{i,2}y+k_{i,3}z)\big )\bigg )\\
&\cdot sin(3x+2y+z).
\end{split}
\end{equation}

An alternative manufactured solution is:
\begin{equation}\label{eq:u31}
\hat{h}_{MMS}(x,y,z)=5+sin(3x)+sin(2y)+sin(z), \quad with\quad x \in [0,5],\quad  y \in [0,2]\quad and\quad x \in [0,1],
\end{equation}
along with the Dirichlet and Neumann boundary conditions:
\begin{equation}\label{eq:b31}
\left\{
\begin{array}{lr}
\hat{h}(0,y,z)=5+sin(2y)+sin(z), & \forall  y \in [0,2],\forall  z \in [0,1],\\
\hat{h}(5,y,z)=5+sin(3\times5)+sin(2y)+sin(z), &  \forall  y \in [0,2],\forall  z \in [0,1],\\
\frac{\partial \hat{h}}{\partial y}(x,0,z)=2cos(0), & \forall  x \in [0,5],\forall  z \in [0,1],\\
\frac{\partial \hat{h}}{\partial y}(x,2,z)=2cos(2\times2), & \forall  x \in [0,5],\forall  z \in [0,1],\\
\frac{\partial \hat{h}}{\partial z}(x,y,0)=cos(0), & \forall  x \in [0,5],\forall  y \in [0,2],\\
\frac{\partial \hat{h}}{\partial z}(x,y,1)=cos(1), & \forall  x \in [0,5],\forall  y \in [0,2].
\end{array}
\right.
\end{equation}

And the source term $f$ has the following form:
\begin{equation}\label{eq:f32}
\begin{split}
f(x,y,z)=&3C_1C_2 exp\bigg (C_2\sum_{i=1}^{N}cos\big (\xi_1 +2\pi (k_{i,1}x+k_{i,2}y+k_{i,3}z)\big )\bigg )\\
&\cdot \sum_{i=1}^{N}\bigg (-2\pi k_{i,1}sin\big (\xi_1 +2\pi (k_{i,1}x+k_{i,2}+k_{i,3}z)\big )\bigg )cos(3x)\\
&+2C_1C_2 exp\bigg (C_2\sum_{i=1}^{N}cos\big (\xi_1 +2\pi (k_{i,1}x+k_{i,2}y+k_{i,3}z)\big )\bigg )\\
&\cdot \sum_{i=1}^{N}\bigg (-2\pi k_{i,2}sin\big (\xi_1 +2\pi (k_{i,1}x+k_{i,2}+k_{i,3}z)\big )\bigg )cos(2y)\\
&+C_1C_2 exp\bigg (C_2\sum_{i=1}^{N}cos\big (\xi_1 +2\pi (k_{i,1}x+k_{i,2}y+k_{i,3}z)\big )\bigg )\\
&\cdot \sum_{i=1}^{N}\bigg (-2\pi k_{i,3}sin\big (\xi_1 +2\pi (k_{i,1}x+k_{i,2}+k_{i,3}z)\big )\bigg )cos(z)\\
&-C_1 exp\bigg (C_2\sum_{i=1}^{N}cos\big (\xi_1 +2\pi(k_{i,1}x+k_{i,2}y+k_{i,3}z)\big )\bigg )\\
&\cdot \big (9sin(3x)+4sin(2y)+sin(z)\big ).
\end{split}
\end{equation}

\newpage
\bibliography{refgw.bib}
\end{document}